\documentclass{article} 
\usepackage{iclr2027_conference,times}

\usepackage{hyperref}
\usepackage{url}

\usepackage{booktabs}       
\usepackage{amsfonts}       
\usepackage{nicefrac}       
\usepackage{microtype}      
\usepackage{xcolor}         

\usepackage{amsmath}
\usepackage{amssymb}
\usepackage{mathtools}
\usepackage{amsthm}
\usepackage{graphicx}
\usepackage{subcaption}
\usepackage{placeins}
\usepackage{algorithm}
\usepackage{algpseudocode}

\theoremstyle{plain}
\newtheorem{theorem}{Theorem}[section]
\newtheorem{proposition}[theorem]{Proposition}
\newtheorem{lemma}[theorem]{Lemma}
\newtheorem{corollary}[theorem]{Corollary}
\theoremstyle{definition}

\newtheorem{assumption}[theorem]{Assumption}
\theoremstyle{remark}
\newtheorem{remark}[theorem]{Remark}

\title{Robust Direct Preference Optimization under Preference Label-Flip Attacks}

\author{
   Chenye Yang \\
   Department of Electrical and Computer Engineering \\
   University of California, Davis \\
   Davis, CA 95616, USA \\
   \texttt{cyyyang@ucdavis.edu} \\
   \And
   Weiyu Xu \\
   Department of Electrical and Computer Engineering \\
   University of Iowa \\
   Iowa City, IA 52242, USA \\
   \texttt{weiyu-xu@uiowa.edu} \\
   \And
   Lifeng Lai \\
   Department of Electrical and Computer Engineering \\
   University of California, Davis \\
   Davis, CA 95616, USA \\
   \texttt{lflai@ucdavis.edu}
}

\iclrfinalcopy 

\begin{document}

\maketitle
\fancyhead{} 

\begin{abstract}
Direct Preference Optimization (DPO) is widely used for offline Reinforcement Learning from Human Feedback. 
However, recent works have shown that DPO can be easily manipulated by adversarial preference label flip attacks,
which are specifically designed to mislead the learner toward an attacker-chosen target policy. 
To address this issue, we design robust DPO methods that can defend against such attacks.
Building on the property that lable flips induce parameter-independent gradient shifts in log-linear DPO, 
we propose Trim-$K$ Robust-DPO, which removes the $K$ highest-impact pairs, and Soft-Trim-$K$ Robust-DPO, which smoothly downweights them. 
We establish deterministic gradient perturbation bounds uniformly over all attack sets containing at most $K$ flips, 
explicitly separating defense-induced bias from surviving attack effects. 
Under bounded-region assumptions, these bounds yield time-uniform parameter and policy deviation guarantees. 
We further quantify how quadratic regularization improves curvature while introducing additional bias. 
Experiments on Stanford Human Preferences, HH-RLHF, and UltraFeedback-Binarized validate our analysis.
\end{abstract}

\section{Introduction}

Reinforcement Learning from Human Feedback (RLHF) is a central approach for aligning the output of modern large models with human preferences when an explicit reward function is difficult to specify~\citep{christiano2017deep,ziegler2019fine,ouyang2022training,bai2022training,casper2023open}.
In the offline RLHF setting, the learner trains on a pre-collected dataset of pairwise preference comparisons~\citep{xiong2023iterative,zhu2023principled}, and Direct Preference Optimization (DPO)~\citep{rafailov2023direct} has become especially popular among these approaches.
DPO avoids separate reward modeling and RL-based policy optimization, and directly optimizes a preference-based objective over the policy under the Bradley-Terry model~\citep{bradley1952rank}. 

However, this dependence on preference pairs of RLHF and DPO also introduces a security risk: 
corrupted comparisons can systematically mislead the learned policy~\citep{casper2023open}.
In particular, several recent papers~\citep{wu2024preference,rando2023universal,baumgartner2024best,pathmanathan2024poisoning,nika2025policy,fu2025poisonbench,pathmanathan2025poisoning,wu2025preference} have introduced adversarial preference attacks on offline RLHF.
Unlike random preference noise caused by annotation errors and limited expert knowledge~\citep{casper2023open,yan2024reward},
the targeted adversarial attacker can strategically flip preference labels or inject poisoned comparisons,
to mislead the learner toward an attacker-chosen target output policy.
The label flip attack is practically plausible because it does not require injecting a large number of new preference pairs~\citep{nika2025policy}.
\citet{yang2026efficient}~formalizes this threat on DPO and develops efficient adversarial targeted attacks.
The above security issue is serious since RLHF is widely used in real-world applications like Large Language Models (LLMs)~\citep{bai2022training},
and a misled output might be biased, toxic, or even harmful~\citep{casper2023open,rando2023universal,pathmanathan2025poisoning}.

To address this serious issue, in this paper, we design DPO algorithms robust to adversarial preference flip attacks.
Unlike attacks on bandits or standard RL that manipulate numerical rewards, actions, contexts, or transition dynamics~\citep{jun2018adversarial,ma2018data,garcelon2020adversarial,rakhsha2020policy,rakhsha2021reward,liu2021provably,xu2022efficient},
preference corruption changes the preference order and then the optimized DPO policy.
Existing robust RLHF/DPO studies mainly address random noise or distribution shift in preferences~\citep{yan2024reward,cheng2024rime,mandal2024corruption,chowdhury2024provably,wu2024towards,zhou2025unified,kim2025preference,xu2025robust}.
Our challenge is instead to analyze and mitigate strategically chosen label flips and their effect on the learned policy.
To the best of our knowledge, this is the first work to design DPO algorithms robust against such adversaries.

Our analysis focuses on the log-linear DPO policy class, which enables us to analyze the problem in gradient space by a key observation~\citet{yang2026efficient} that, 
flipping one preference pair shifts its per-sample DPO gradient by a parameter-independent vector.
This property makes attack-impact magnitudes pre-computable before training, without knowing the actual attack strategy.
Based on this structure, we design robust DPO methods that reduce the influence of samples whose flips would have large impact in the gradient space.

We propose \emph{Trim-$K$ Robust-DPO}, which removes the $K$ highest-impact preference pairs and optimizes DPO on the retained set, and \emph{Soft-Trim-$K$ Robust-DPO}, 
which retains all pairs while smoothly downweighting the high-impact region.
Both methods use interpretable pre-computed impact-based rules, and control potential flip-induced gradient effects without attempting to identify attacked samples.
Soft-Trim-$K$ provides a continuous family between standard DPO and Trim-$K$.
For both methods, we establish gradient-bias bounds relative to full-data clean DPO uniformly over all fixed attack sets of at most $K$ flips, 
explicitly separating the bias introduced by trimming or reweighting from surviving attack effects.
Under the practical assumption that training stays in a bounded region around the reference parameter, 
these bounds yield time-uniform parameter and policy deviation guarantees for full-batch gradient descent.
We also analyze quadratically regularized variants, quantifying the trade-off between increased curvature and additional bias relative to unregularized clean DPO.

We evaluate the methods on Stanford Human Preferences (SHP)~\citep{ethayarajh2022understanding}, HH-RLHF~\citep{bai2022training,ganguli2022red}, and UltraFeedback-Binarized~\citep{cui2023ultrafeedback}.
The SHP experiments compare our methods and their regularized variants with standard and robust DPO baselines under different attacks (Appendices~\ref{app:noatk-random} and~\ref{app:baselines}).
Further experiments examine datasets and budgets, matched regularization and score controls, budget misspecification (Appendices~\ref{app:matched-regularization}--\ref{app:k-selection}), and small nonlinear policies (Appendix~\ref{app:nonlinear-extension}).
The results illustrate the trade-offs in our analysis: hard trimming is most helpful when flips concentrate on high-impact pairs, soft trimming offers a smoother robustness--retention trade-off, and regularization can limit attack-induced deviation while introducing bias relative to clean DPO.

\textbf{Main contributions:}
(1) We propose Trim-$K$ and Soft-Trim-$K$ Robust-DPO without attempting to identify the attacked samples.
(2) We establish deterministic gradient-bias bounds relative to full-data clean DPO, uniformly over all fixed attack sets of at most $K$ flips.
(3) Under the stated assumptions, we derive time-uniform parameter and policy deviation guarantees and quantify the curvature gain and bias introduced by quadratic regularization.
(4) We evaluate the methods on three preference datasets, with matched controls, score-assignment ablations, and attack-budget and misspecification studies. We also conduct preliminary experiments with small nonlinear policies.

\textbf{Additional related work:}
Appendix~\ref{app:related-work} summarizes prior works on adversarial attacks and robust learning in bandits, RL, and RLHF.

\section{Problem Formulation}\label{sec:problem-formulation}

\subsection{Offline RLHF and Log-linear DPO}\label{sec:log-linear-dpo}

Offline RLHF usually combines feedback collection, reward modeling, and policy optimization~\citep{christiano2017deep,ouyang2022training,bai2022training}.
Given a prompt $x$, annotators compare two candidate responses $y_w$ and $y_l$ and return a preference $y_w \succ y_l \mid x$, where $y_w$ is preferred to $y_l$. 
Under Bradley-Terry model~\citep{bradley1952rank}, the preference probability:
\begin{equation*}
p^*(y_w \succ y_l \mid x)
=
\frac{\exp(r^*(x,y_w))}
{\exp(r^*(x,y_w))+\exp(r^*(x,y_l))},
\end{equation*}
where $r^*(x,y)$ is an unknown latent reward. 
Standard RLHF first learns a reward model from such comparisons and then optimizes a policy against that reward, 
usually with a KL penalty relative to a reference policy~\citep{ziegler2019fine,schulman2017proximal}.
DPO removes these two stages, 
and optimizes the policy directly from the same preference pairs~\citep{rafailov2023direct}.
Let $\pi_{\mathrm{ref}}$ be the reference policy and let $\beta>0$ control the strength of KL regularization. 
For a dataset $\mathcal D=\{(x^i,y_w^i,y_l^i)\}_{i=1}^N$,
let $\sigma(z)=1/(1+e^{-z})$,
the DPO objective is
\begin{equation}\label{eq:dpo-loss}
    \mathcal{L}_{\mathrm{DPO}}(\pi_\theta, \pi_{\mathrm{ref}}, \mathcal{D}) = - \mathbb{E}_{(x, y_w, y_l) \sim \mathcal{D}} \left[ \log \sigma \left( \beta \log \frac{\pi_\theta(y_w \mid x)}{\pi_{\mathrm{ref}}(y_w \mid x)} - \beta \log \frac{\pi_\theta(y_l \mid x)}{\pi_{\mathrm{ref}}(y_l \mid x)} \right) \right].
\end{equation}

DPO is efficient and stable in practice, 
and avoids the mismatch that can arise when the reward model is fit in a separate stage~\citep{rafailov2023direct}.

In this paper, we mainly consider a log-linear policy class, following the theoretical setup in the literature~\citep{nika2025policy,chowdhury2024provably,xu2025robust}.
In Appendix~\ref{app:nonlinear-extension}, we discuss the potential to extend our methods and analysis beyond the log-linear setting,
and the associated preliminary nonlinear experiments.

\begin{assumption}[Log-linear Policy~\citealp{nika2025policy}]\label{ass:log-linear-policy}
    Let $\psi$ be a $d$-dimensional feature mapping $\psi: \mathcal{X} \times \mathcal{Y} \to \mathbb{R}^{d}$ with $\max_{x,y} \|\psi(x,y)\|_2 \leq 1$.
    We consider the following class of log-linear policies parameterized by $\theta \in \mathbb{R}^{d}$ for all $(x,y) \in \mathcal{X} \times \mathcal{Y}$:
    \begin{equation*}
        \Pi^{\mathrm{log}} = \left\{ \pi_\theta : \pi_\theta(y \mid x) = \frac{\exp\left(\psi(x,y)^\top \theta\right)}{\sum_{y'} \exp\left(\psi(x,y')^\top \theta\right)} \right\}.
    \end{equation*}
\end{assumption}

For each preference pair $(x^i,y_w^i,y_l^i)$ in $\mathcal D$, the per-sample DPO loss is
\begin{equation}\label{eq:dpo-per-sample}
    \ell_i(\theta)
    :=
    \ell_{\mathrm{DPO}}(\theta;x^i,y_w^i,y_l^i)
    =
    - \log \sigma\!\left(
        \beta \Big[
        \log \tfrac{\pi_\theta(y_w^i\mid x^i)}{\pi_{\mathrm{ref}}(y_w^i\mid x^i)}
        - \log \tfrac{\pi_\theta(y_l^i\mid x^i)}{\pi_{\mathrm{ref}}(y_l^i\mid x^i)}
        \Big]
    \right).
\end{equation}
Under Assumption~\ref{ass:log-linear-policy}, the normalization constants cancel:
$\log \pi_\theta(y_w^i\mid x^i) - \log \pi_\theta(y_l^i\mid x^i) = \big\langle \theta,\ \psi(x^i,y_w^i)-\psi(x^i,y_l^i) \big\rangle$.
Define the feature difference $v_i$ and reference-policy offset $b_i$ for each sample, then the DPO logit $z_i(\theta)$ is affine in $\theta$, and $\ell_i(\theta)=-\log\sigma(z_i(\theta))$:
\begin{equation}\label{eq:dpo-logit-affine}
    v_i := \psi(x^i,y_w^i)-\psi(x^i,y_l^i),
    \quad
    b_i := \log \tfrac{\pi_{\mathrm{ref}}(y_w^i\mid x^i)}{\pi_{\mathrm{ref}}(y_l^i\mid x^i)},
    \quad
    z_i(\theta):=\beta\big(\langle \theta,v_i\rangle-b_i\big).
\end{equation}
The corresponding per-sample gradient $g_i(\theta)$ and its Hessian are
\begin{align}
    g_i(\theta)
    :=\nabla \ell_i(\theta)
    &= \beta\big(\sigma(z_i(\theta))-1\big)v_i,
    \label{eq:dpo-per-sample-grad}\\
    \nabla^2 \ell_i(\theta)
    &=\beta^2\sigma(z_i(\theta))(1-\sigma(z_i(\theta)))v_i v_i^\top \succeq 0.
    \label{eq:dpo-per-sample-hess}
\end{align}

The following identity comes from the gradient-space analysis of attacks on log-linear DPO in~\citet{yang2026efficient}.
We restate it here as an immediate structural property of the log-linear policy class.

\begin{lemma}[Constant gradient shift under label swap]\label{lem:constant-shift}
    Let $g_i(\theta)$ be the clean per-sample gradient in~\eqref{eq:dpo-per-sample-grad}.
    Let $g_i^{\mathrm{flip}}(\theta)$ denote the gradient after swapping $(y_w^i,y_l^i)$ to $(y_l^i,y_w^i)$.
    Then for all $\theta\in\mathbb R^d$,
    \begin{equation*}\label{eq:constant-shift}
        g_i^{\mathrm{flip}}(\theta)-g_i(\theta)=\beta v_i,
    \end{equation*}
    and hence the flip impact magnitude is
    $\|g_i^{\mathrm{flip}}(\theta)-g_i(\theta)\|_2=\beta\|v_i\|_2$.
\end{lemma}

We assume training stays in a bounded region, motivated by initialization at the reference, finite training horizons, and norm control. Appendix~\ref{app:bounded-region} explains the resulting local curvature, its role in time-uniform stability, and its distinction from quadratic regularization~\citep{nika2025policy}.

\begin{assumption}[Bounded training region]\label{ass:bounded-region}
There exists a radius $R<\infty$ such that all clean and attacked gradient-descent iterates considered in our theory remain in
$\mathbb B_R(\theta_{\mathrm{ref}}):=\{\theta:\|\theta-\theta_{\mathrm{ref}}\|_2\le R\}$,
where $\theta_{\mathrm{ref}}$ is the parameter of log-linear reference policy $\pi_{\mathrm{ref}}$.
Moreover, the feature covariance of each trimmed or weighted objective is positive definite on a common subspace containing all update directions.
\end{assumption}

\subsection{Budgeted Adversarial Preference Attack and Robust Defense}\label{sec:threat-model}

Given a clean preference dataset $\mathcal{D}=\{(x^i,y_w^i,y_l^i)\}_{i=1}^N$, a label-flip attacker selects a set
$A\subseteq[N]$ with $|A|\le K$ and forms the attacked dataset, denoted by tildes:
\begin{equation*}
\tilde{\mathcal D}(A)
=
\{(x^i,\tilde y_w^i,\tilde y_l^i)\}_{i=1}^N,
\qquad
(\tilde y_w^i,\tilde y_l^i)
=
\begin{cases}
(y_l^i,y_w^i), & i\in A,\\
(y_w^i,y_l^i), & i\notin A.
\end{cases}
\end{equation*}
A general budgeted label-flip attacker may choose any fixed support $A$.
A targeted attacker is the special case that selects $A$ to move training toward an attacker-chosen and potentially misaligned target policy $\pi_{\theta^\dagger}$. Our defense algorithms do not rely on the attack methods used by the attacker.

The robust defender observes only: (1) the possibly attacked dataset $\tilde{\mathcal D}(A)$; (2) an upper-bound $K$ on the attack budget.
It does not know the attacked index set $A$, has no clean preference orientations, 
and no clean validation data is assumed.
The ground-truth comparison target is the standard DPO objective on the full clean dataset:
\begin{equation}\label{eq:true-clean-F}
F(\theta)\equiv F_{\mathrm{DPO}}(\theta):=\frac{1}{N}\sum_{i=1}^N \ell_i(\theta).
\end{equation}
Let $\mathrm{R\text{-}DPO}$ denote a robust DPO rule, and let $\tilde F_{\mathrm{R\text{-}DPO},A}$ denote its differentiable training objective constructed from the observed dataset $\tilde{\mathcal D}(A)$ and budget upper bound $K$.

Our goal is to design $\mathrm{R\text{-}DPO}$ so that robust DPO on the attacked data remains close to clean DPO on the full clean data, uniformly over all admissible attack sets.
More precisely, for the clean and attacked full-batch iterates
\begin{equation*}
\theta_{t+1}^{\mathrm{cln}}
=
\theta_t^{\mathrm{cln}}-\eta\nabla F(\theta_t^{\mathrm{cln}}),
\qquad
\theta_{t+1}^{\mathrm{atk}}
=
\theta_t^{\mathrm{atk}}-\eta\nabla \tilde F_{\mathrm{R\text{-}DPO},A}(\theta_t^{\mathrm{atk}}),
\end{equation*}
initialized from the same $\theta_0$, we seek deterministic bounds of the form
\begin{align*}
\sup_{A \subseteq[N], |A|\le K}
\sup_{\theta\in\mathbb B_R(\theta_{\mathrm{ref}})}
\big\|\nabla \tilde F_{\mathrm{R\text{-}DPO},A}(\theta)-\nabla F(\theta)\big\|_2
&\le B_g(\mathrm{R\text{-}DPO},K),\\
\sup_{A\subseteq[N], |A|\le K}
\sup_{T\ge 1}
\big\|\theta_T^{\mathrm{atk}}-\theta_T^{\mathrm{cln}}\big\|_2
&\le B_\theta(\mathrm{R\text{-}DPO},K),\\
\sup_{A\subseteq[N], |A|\le K}
\sup_{T\ge 1}
\mathbb E_x\!\left[
D_{\mathrm{KL}}\!\big(\pi_{\theta_T^{\mathrm{atk}}}(\cdot\mid x)\,\|\,\pi_{\theta_T^{\mathrm{cln}}}(\cdot\mid x)\big)
\right]
&\le B_\pi(\mathrm{R\text{-}DPO},K).
\end{align*}
The bounds should depend explicitly on the attack budget $K$, 
the feature geometry of the preference pairs, 
the clean bias introduced by the defense rule, 
and the local curvature of the optimized robust objective on $\mathbb B_R(\theta_{\mathrm{ref}})$.
When there is no attack ($K=0$), the $\mathrm{R\text{-}DPO}$ should reduce to standard DPO.

\section{Methods}\label{sec:methods}

The identity in Lemma~\ref{lem:constant-shift} makes the magnitude of each comparison's potential label-flip effect pre-computable from the observed pair geometry.
For each observed pair, define the flip-impact score
\begin{equation}\label{eq:flip-impact-score}
s_i
:=
\beta\big\|\psi(x^i,\tilde y_w^i)-\psi(x^i,\tilde y_l^i)\big\|_2
=
\beta\|v_i\|_2.
\end{equation}
Because a label flip only changes the sign of the feature difference, $s_i$ is invariant to the observed winner and loser.
We first develop Trim-$K$ Robust-DPO, which removes the $K$ largest-score pairs, then introduce Soft-Trim-$K$ Robust-DPO as a smooth extension to preserve more clean information.
Both methods are static and require only $\tilde{\mathcal D}(A)$ and $K$.
\citet{yang2026efficient} uses the same identity and a known target $\theta^\dagger$ to construct attackers.
But we ask a different question: with limited knowledge of the attack, can a defense control the effect of the attack?
The analysis below address this question.

\subsection{Error Decomposition}\label{sec:error-decomposition}

We distinguish three objectives in the analysis. The comparison target $F$ is the full-data clean DPO objective in~\eqref{eq:true-clean-F}.
For a given robust rule, $\tilde F_{\mathrm{R\text{-}DPO},A}$ is the objective actually optimized on the attacked dataset, 
while $F_{\mathrm{R\text{-}DPO}}$ is its auxiliary clean counterpart, obtained by using the clean preference labels while keeping the defense methods fixed.
This auxiliary objective is used only for analysis:
\begin{equation}\label{eq:gradient-bias-decomposition}
\nabla \tilde F_{\mathrm{R\text{-}DPO},A}(\theta)-\nabla F(\theta)
=
\underbrace{\big(\nabla F_{\mathrm{R\text{-}DPO}}(\theta)-\nabla F(\theta)\big)}_{\text{clean defense bias}}
+
\underbrace{\big(\nabla \tilde F_{\mathrm{R\text{-}DPO},A}(\theta)-\nabla F_{\mathrm{R\text{-}DPO}}(\theta)\big)}_{\text{surviving attack bias}}.
\end{equation}
The first term measures the error introduced by the defense even on clean labels, 
and the second measures the residual effect of the label flips under that same defense.
The following subsections define these objectives for each robust DPO method and bound the two terms.

\subsection{Trim-$K$ Robust-DPO}\label{sec:trim-k-robust-dpo}

Rank the samples by decreasing impact score, and let $r(i)\in[N]$ be the resulting rank, with $s_{(j)}$ denoting the score at rank $j$.
For a upper bound $K<N$, define $\delta(K):=s_{(K+1)}$.
Trim-$K$ Robust-DPO removes the indices with $r(i)\le K$ and keeps
$G:=\{i\in[N]:r(i)>K\}$.
Thus $|G|=N-K$ and $\max_{i\in G}s_i=\delta(K)$.
The rule is a worst-case safeguard: it removes the pairs whose flips could have the largest attacked gradient effect, without trying to identify the attacked set exactly.

Write $\tilde \ell_i(\theta):=\ell_{\mathrm{DPO}}(\theta;x^i,\tilde y_w^i,\tilde y_l^i)$ for the per-sample loss on the observed pair produced by attack set $A$.
For the kept set $G$, define the clean trimmed auxiliary objective and the attacked objective actually optimized by the defender:
\begin{equation*}
F_G(\theta)
:=
\frac{1}{|G|}\sum_{i\in G}\ell_i(\theta),
\qquad
\tilde F_{G,A}(\theta)
:=
\frac{1}{|G|}\sum_{i\in G}\tilde \ell_i(\theta).
\end{equation*}
In the notation of Section~\ref{sec:error-decomposition}, $F_{\mathrm{R\text{-}DPO}}=F_G$ and $\tilde F_{\mathrm{R\text{-}DPO},A}=\tilde F_{G,A}$.
Both objectives use the same kept set $G$, which is fixed during training and invariant to label flips.
Without attack ($K=0$), this objective reduces exactly to standard DPO on the full dataset.
If $A=\varnothing$ but the supplied upper bound is $K>0$, trimming can introduce clean-data bias.
Algorithm~\ref{alg:trimk} in Appendix~\ref{app:algorithms} gives the full implementation.

By Lemma~\ref{lem:constant-shift}, only indices in $A\cap G$ contribute a surviving gradient shift, 
so the gradient difference decomposes as
$
\nabla \tilde F_{G,A}(\theta)-\nabla F(\theta)
=
\big(\nabla F_G(\theta)-\nabla F(\theta)\big)
+
\frac{1}{|G|}\sum_{i\in A\cap G}\beta v_i
$.
Define the clean trimming bias on the bounded training region by
\begin{equation}\label{eq:clean-trim-bias}
\mathsf B_G^{\mathrm{cln}}(R)
:=
\sup_{\theta\in\mathbb B_R(\theta_{\mathrm{ref}})}
\|\nabla F_G(\theta)-\nabla F(\theta)\|_2 .
\end{equation}
This term is zero only when the trimmed subset has the same average clean gradient as the full dataset,
and in particular it is zero in the no-attack control $K=0$, where $G=[N]$.
A closed-form upper bound on $\mathsf B_G^{\mathrm{cln}}(R)$ that depends only on the score statistics is proved as Lemma~\ref{lem:trim-bias-bound} in Appendix~\ref{app:supporting-lemmas}:
\begin{equation*}
\mathsf B_G^{\mathrm{cln}}(R)\;\le\;\frac{K\,\delta(K)+\sum_{j=1}^K s_{(j)}}{N}\;\le\;\frac{4\beta K}{N}.
\end{equation*}
For compact notation, define
\begin{equation}\label{eq:Delta-G-def}
\Delta_G
:=
\mathsf B_G^{\mathrm{cln}}(R)+\frac{K\,\delta(K)}{N-K},
\qquad
\bar\Delta_G
:=
\frac{K\,\delta(K)+\sum_{j=1}^K s_{(j)}}{N}
+
\frac{K\,\delta(K)}{N-K}.
\end{equation}

\begin{theorem}[Worst-case gradient robustness of Trim-$K$ Robust-DPO]\label{thm:K-robust-grad}
Let $G$ be the kept set after removing the top-$K$ largest impact scores, so that $|G|=N-K$ and
$\max_{i\in G}s_i\le \delta(K)$.
Then for every attack set $A\subseteq[N]$ with $|A|\le K$ and every $\theta\in\mathbb B_R(\theta_{\mathrm{ref}})$,
\begin{equation*}
\|\nabla \tilde F_{G,A}(\theta)-\nabla F(\theta)\|_2
\le
\Delta_G
\le
\bar\Delta_G .
\end{equation*}
\end{theorem}

This theorem separates the two effects that determine distance to clean DPO: 
the deterministic clean trimming bias and the surviving attack bias. 
The clean-trimming term scales as $K/N$ and is governed by the top-$K$ score sum together with the cutoff $\delta(K)$, 
while the attack-dependent term is controlled by the $(K+1)$-st largest score rather than the largest score in the full dataset.

To propagate this gradient-bias bound through training, define the empirical feature covariance and the bounded-region logistic curvature constant
\begin{equation}\label{eq:trim-cov-curvature}
\Sigma_G:=\frac{1}{|G|}\sum_{i\in G} v_i v_i^\top,
\qquad
\underline c_G(R):=
\min_{\theta\in \mathbb B_R(\theta_{\mathrm{ref}})}\min_{i\in G}
\sigma(z_i(\theta))\bigl(1-\sigma(z_i(\theta))\bigr).
\end{equation}
Because $G$ is finite and $\mathbb B_R(\theta_{\mathrm{ref}})$ is compact, Assumption~\ref{ass:bounded-region} implies $\underline c_G(R)>0$.
Let
\begin{equation}\label{eq:trim-mu-G}
\mu_G:=\beta^2\,\underline c_G(R)\,\lambda_{\min}(\Sigma_G)>0.
\end{equation}
Here and below, $\lambda_{\min}$ denotes the smallest eigenvalue on the common trained subspace of Assumption~\ref{ass:bounded-region}.
The supporting smoothness, restricted strong-convexity, and gradient-descent stability lemmas are stated in Appendix~\ref{app:supporting-lemmas}.
Proofs of the theoretical results are in Appendix~\ref{app:proofs}.

\begin{theorem}[Time-uniform stability for Trim-$K$ Robust-DPO]\label{thm:trim-K-iter-stability}
Under Assumptions~\ref{ass:log-linear-policy} and~\ref{ass:bounded-region}, suppose the clean DPO iterates and the attacked Trim-$K$ iterates remain in $\mathbb B_R(\theta_{\mathrm{ref}})$.
Run full-batch gradient descent from the same initialization with step size $\eta\le 1/\beta^2$:
$\theta_{t+1}^{\mathrm{cln}}=\theta_t^{\mathrm{cln}}-\eta\nabla F(\theta_t^{\mathrm{cln}})$,
$\theta_{t+1}^{\mathrm{atk}}=\theta_t^{\mathrm{atk}}-\eta\nabla \tilde F_{G,A}(\theta_t^{\mathrm{atk}})$.
Then for every $T\ge 1$,
\begin{equation*}
\|\theta_T^{\mathrm{atk}}-\theta_T^{\mathrm{cln}}\|_2
\le
\frac{\Delta_G\bigl(1-(1-\eta\mu_G)^T\bigr)}{\mu_G}
\le
\frac{\Delta_G}{\mu_G}.
\end{equation*}
In particular, Lemma~\ref{lem:trim-bias-bound} gives
$
\|\theta_T^{\mathrm{atk}}-\theta_T^{\mathrm{cln}}\|_2
\le
\bar\Delta_G/\mu_G.
$
\end{theorem}

Theorem~\ref{thm:trim-K-iter-stability} shows that the Trim-$K$ objective has a time-uniform robustness guarantee once we restrict attention to the practical training region.
This is a deterministic deviation bound for every fixed admissible attack set and every training horizon under the stated assumptions. 
It does not rank robust DPO methods, or predict their empirical dominance for a particular attack. 
We note that the constant $\mu_G$ can still be small. 
In Section~\ref{sec:regularized-robust-dpo}, we address this issue and propose regularized variants that can produce tighter behavior.

Corollary~\ref{cor:trim-K-policy-kl} converts the parameter deviation into policy deviation.

\begin{corollary}[Trim-$K$ policy deviation bound]\label{cor:trim-K-policy-kl}
Under the conditions of Theorem~\ref{thm:trim-K-iter-stability}, for any fixed context $x$,
$
D_{\mathrm{KL}}\!\big(\pi_{\theta_T^{\mathrm{atk}}}(\cdot|x)\,\|\,\pi_{\theta_T^{\mathrm{cln}}}(\cdot|x)\big)
\le \frac{1}{2} \|\theta_T^{\mathrm{atk}}-\theta_T^{\mathrm{cln}}\|_2^2
$.
Consequently,
\begin{equation*}
\mathbb E_x\!\left[
D_{\mathrm{KL}}\!\big(\pi_{\theta_T^{\mathrm{atk}}}(\cdot|x)\,\|\,\pi_{\theta_T^{\mathrm{cln}}}(\cdot|x)\big)
\right]
\;\le\;
\frac{1}{2}\left(\frac{\bar\Delta_G}{\mu_G}\right)^2,
\end{equation*}
where the final bound follows from Theorem~\ref{thm:trim-K-iter-stability} and Lemma~\ref{lem:trim-bias-bound}.
\end{corollary}

\subsection{Soft-Trim-$K$ Robust-DPO}\label{sec:soft-trim-k-robust-dpo}

The hard trimming rule of Section~\ref{sec:trim-k-robust-dpo} gives the clearest worst-case guarantee, but it can also remove clean high-impact samples.
This is desirable when the attacker is concentrated on those samples, but it may discard too much useful information when the true attack pattern is widespread or when the budget upper bound $K$ is conservative.
To handle this trade-off, we introduce a soft version that downweights the top-impact samples instead of deleting them.

Using the deterministic ranking $r(i)$ introduced in Section~\ref{sec:trim-k-robust-dpo}, fix a temperature parameter $\tau>0$ and assign the exponential weights
\begin{equation}\label{eq:exp-soft-trim-weight}
w_i^{(\tau)}
=
\begin{cases}
\exp\!\big(-(K+1-r(i))/\tau\big), & r(i)\le K, \\
1, & r(i)\ge K+1.
\end{cases}
\end{equation}
The highest-impact samples receive small positive weights, while samples after the trimming cutoff keep weight one.
When $K=0$, Soft-Trim is standard DPO. 
When $K>0$, $\tau\downarrow 0$ approaches hard trimming and $\tau\to\infty$ approaches DPO.
The weighting rule is illustrated in Appendix~\ref{app:soft-trim-weights}, Figure~\ref{fig:exp-soft-trim-weights}.

Let $W_\tau := \sum_{i=1}^N w_i^{(\tau)}$ and $\varepsilon_\tau:=e^{-1/\tau}$.
Define the clean weighted auxiliary objective and the attacked objective actually optimized by the defender:
\begin{equation*}
F_\tau(\theta)
:=
\frac{1}{W_\tau}\sum_{i=1}^N w_i^{(\tau)}\ell_i(\theta),
\qquad
\tilde F_{\tau,A}(\theta)
:=
\frac{1}{W_\tau}\sum_{i=1}^N w_i^{(\tau)}\tilde \ell_i(\theta).
\end{equation*}
In the notation of Section~\ref{sec:error-decomposition}, $F_{\mathrm{R\text{-}DPO}}=F_\tau$ and $\tilde F_{\mathrm{R\text{-}DPO},A}=\tilde F_{\tau,A}$.
Both objectives use the same weights $w_i^{(\tau)}$, which are fixed during training and invariant to label flips.
Algorithm~\ref{alg:softtrim} in Appendix~\ref{app:algorithms} gives the implementation.

For each sample, define the weighted impact
$\rho_i^{(\tau)}:=w_i^{(\tau)}s_i$.
Let $\rho_{[1]}^{(\tau)}\ge \rho_{[2]}^{(\tau)}\ge \cdots\ge \rho_{[N]}^{(\tau)}$ be these weighted impacts sorted in descending order, and define
\begin{equation*}
\Gamma_\tau(K)
:=
\max_{A\subseteq[N],\,|A|\le K}\sum_{i\in A}\rho_i^{(\tau)}
=
\sum_{j=1}^K\rho_{[j]}^{(\tau)}.
\end{equation*}
Thus $\Gamma_\tau(K)/W_\tau$ is the largest normalized weighted flip-impact mass that any attack set of at most $K$ samples can have.
Define the clean weighting bias on the bounded training region by
\begin{equation*}
\mathsf B_\tau^{\mathrm{cln}}(R)
:=
\sup_{\theta\in\mathbb B_R(\theta_{\mathrm{ref}})}
\|\nabla F_\tau(\theta)-\nabla F(\theta)\|_2 .
\end{equation*}
Let
$a_j^{(\tau)}:=\exp \bigl(-(K+1-j)/\tau\bigr)$,
$D_\tau:=K-\sum_{j=1}^K a_j^{(\tau)}$,
$\bar s=\tfrac{1}{N}\sum_i s_i$.
A closed-form score-only upper bound is proved as Lemma~\ref{lem:softtrim-bias-bound} in Appendix~\ref{app:supporting-lemmas}: 
\begin{equation*}
\mathsf B_\tau^{\mathrm{cln}}(R)
\le
\frac{D_\tau\bar s+\sum_{j=1}^K(1-a_j^{(\tau)})s_{(j)}}{W_\tau}.
\end{equation*}
For compact notation, define
\begin{equation*}
\Delta_\tau
:=
\mathsf B_\tau^{\mathrm{cln}}(R)+\frac{\Gamma_\tau(K)}{W_\tau},
\qquad
\bar\Delta_\tau
:=
\frac{D_\tau\bar s+\sum_{j=1}^K(1-a_j^{(\tau)})s_{(j)}}{W_\tau}
+
\frac{\varepsilon_\tau\sum_{j=1}^K s_{(j)}+K\,\delta(K)}{W_\tau}.
\end{equation*}
For $K=0$, all sums above are empty, $D_\tau=\Gamma_\tau(0)=0$, and $\Delta_\tau=\bar\Delta_\tau=0$.

\begin{theorem}[Worst-case gradient robustness of Soft-Trim-$K$ Robust-DPO]\label{thm:softtrim-grad}
For any attack set $A\subseteq[N]$ with $|A|\le K$ and every $\theta\in\mathbb B_R(\theta_{\mathrm{ref}})$,
\begin{equation}\label{eq:softtrim-exact-bound}
\|\nabla \tilde F_{\tau,A}(\theta)-\nabla F(\theta)\|_2
\le
\Delta_\tau.
\end{equation}
Moreover, for the exponential weights in~\eqref{eq:exp-soft-trim-weight},
$
\Gamma_\tau(K)
\;\le\;
\varepsilon_\tau\sum_{j=1}^K s_{(j)}+K\,\delta(K)
$.
Combining with Lemma~\ref{lem:softtrim-bias-bound},
\begin{equation}\label{eq:softtrim-coarse-bound}
\|\nabla \tilde F_{\tau,A}(\theta)-\nabla F(\theta)\|_2
\le
\bar\Delta_\tau.
\end{equation}
\end{theorem}

As $\tau\downarrow 0$, we have $W_\tau\to N-K$ and the exact attack term converges to the largest total score among at most $K$ samples that survive hard Trim-$K$:
\begin{equation*}
\frac{\Gamma_\tau(K)}{W_\tau}
\to
\frac{1}{N-K}
\max_{A\subseteq G,\,|A|\le K}\sum_{i\in A}s_i
\le
\frac{K}{N-K}\,\delta(K).
\end{equation*}
Thus the Soft-Trim objective approaches hard Trim-$K$, and its exact surviving-attack bound approaches a quantity no larger than the hard-trim term $K\delta(K)/(N-K)$, while $\mathsf B_\tau^{\mathrm{cln}}(R)$ approaches $\mathsf B_G^{\mathrm{cln}}(R)$.

Define the weighted feature covariance and bounded-region curvature constant
\begin{equation}\label{eq:softtrim-cov-curvature}
\Sigma_\tau:=\frac{1}{W_\tau}\sum_{i=1}^N w_i^{(\tau)}v_i v_i^\top,
\qquad
\underline c_\tau(R):=
\min_{\theta\in \mathbb B_R(\theta_{\mathrm{ref}})}\min_{i:w_i^{(\tau)}>0}
\sigma(z_i(\theta))\bigl(1-\sigma(z_i(\theta))\bigr),
\end{equation}
and let $\mu_\tau:=\beta^2\underline c_\tau(R)\lambda_{\min}(\Sigma_\tau)>0$.
By the same argument as Lemma~\ref{lem:trimk-restricted-sc}, the clean weighted objective $F_\tau$ is $\mu_\tau$-strongly convex on $\mathbb B_R(\theta_{\mathrm{ref}})$ and $\beta^2$-smooth.

Then we can propagate the gradient-bias bound to time-uniform parameter and policy deviation bounds in Theorem~\ref{thm:softtrim-iter-stability} and Corollary~\ref{cor:soft-trim-K-policy-kl}.

\begin{theorem}[Time-uniform stability for Soft-Trim-$K$ Robust-DPO]\label{thm:softtrim-iter-stability}
Under Assumptions~\ref{ass:log-linear-policy} and~\ref{ass:bounded-region}, suppose the clean DPO iterates and the attacked Soft-Trim-$K$ iterates remain in $\mathbb B_R(\theta_{\mathrm{ref}})$ and gradient descent is run with step size $\eta\le 1/\beta^2$ from the same initialization:
$\theta_{t+1}^{\mathrm{cln}}=\theta_t^{\mathrm{cln}}-\eta\nabla F(\theta_t^{\mathrm{cln}})$,
$\theta_{t+1}^{\mathrm{atk}}=\theta_t^{\mathrm{atk}}-\eta\nabla \tilde F_{\tau,A}(\theta_t^{\mathrm{atk}})$.
Then for every $T\ge 1$,
\begin{equation*}
\|\theta_T^{\mathrm{atk}}-\theta_T^{\mathrm{cln}}\|_2
\le
\frac{\Delta_\tau\bigl(1-(1-\eta\mu_\tau)^T\bigr)}{\mu_\tau}
\le
\frac{\Delta_\tau}{\mu_\tau}.
\end{equation*}
In particular, Theorem~\ref{thm:softtrim-grad} and Lemma~\ref{lem:softtrim-bias-bound} give
$
\|\theta_T^{\mathrm{atk}}-\theta_T^{\mathrm{cln}}\|_2
\le
\bar\Delta_\tau/\mu_\tau.
$
\end{theorem}

\begin{corollary}[Soft-Trim-$K$ policy deviation bound]\label{cor:soft-trim-K-policy-kl}
Under the conditions of Theorem~\ref{thm:softtrim-iter-stability},
\begin{equation*}
\mathbb E_x\!\left[
D_{\mathrm{KL}}\!\big(\pi_{\theta_T^{\mathrm{atk}}}(\cdot|x)\,\|\,\pi_{\theta_T^{\mathrm{cln}}}(\cdot|x)\big)
\right]
\;\le\;
\frac{1}{2} \left(\frac{\bar\Delta_\tau}{\mu_\tau}\right)^2.
\end{equation*}
\end{corollary}

Soft-Trim preserves more clean information than hard trimming, offering a smoother robustness--retention trade-off without uniformly improving its worst-case bias.

The guarantee requires the defender's $K$ to upper-bound the true number of flips. 
Underestimating $K$ invalidates the guarantee, while overestimating it removes or downweights more clean data.
Appendix~\ref{app:k-selection} studies budget selection and misspecified $K$.

\subsection{Regularized Trim-$K$ and Soft-Trim-$K$ Robust-DPO}\label{sec:regularized-robust-dpo}

Write $\mathfrak R\in\{G,\tau\}$ for the above two instances of the generic $\mathrm{R\text{-}DPO}$: $G$ for Trim-$K$ and $\tau$ for Soft-Trim-$K$.
We use the same convention for $F_{\mathfrak R}$, $\tilde F_{\mathfrak R,A}$, $\mathsf B_{\mathfrak R}^{\mathrm{cln}}(R)$, $\mu_{\mathfrak R}$, $\Delta_{\mathfrak R}$, and $\bar\Delta_{\mathfrak R}$.
The bounded-region curvature constants $\mu_G$ and $\mu_\tau$ are positive, but they can be small. 
A simple way to increase the curvature of the optimized robust objective $\tilde F_{\mathfrak R,A}$ is to keep the same trimming or soft-weighting rule and add a quadratic regularization $\lambda$ around the reference parameter:
\begin{equation*}
F_{\mathfrak R,\lambda}(\theta):=F_{\mathfrak R}(\theta)+\frac{\lambda}{2}\|\theta-\theta_{\mathrm{ref}}\|_2^2,
\qquad
\tilde F_{\mathfrak R,\lambda, A}(\theta):=\tilde F_{\mathfrak R,A}(\theta)+\frac{\lambda}{2}\|\theta-\theta_{\mathrm{ref}}\|_2^2.
\end{equation*}
Since the reference objective remains the unregularized clean DPO objective $F$, regularization also creates a deterministic bias. Let
\begin{equation*}
\mathsf B_{\mathfrak R,\lambda}^{\mathrm{cln}}(R)
:=
\sup_{\theta\in\mathbb B_R(\theta_{\mathrm{ref}})}
\left\|
\nabla F_{\mathfrak R}(\theta)+\lambda(\theta-\theta_{\mathrm{ref}})-\nabla F(\theta)
\right\|_2 
\;\le\;
\mathsf B_{\mathfrak R}^{\mathrm{cln}}(R)+\lambda R .
\end{equation*}
Thus the surviving-attack term is unchanged, the clean-bias term gains at most $\lambda R$, and the curvature denominator becomes $\mu_{\mathfrak R}+\lambda$.
\begin{proposition}[Quadratic regularization improves curvature]\label{prop:reg-dpo-curvature}
For $\mathfrak R\in\{G,\tau\}$, suppose $F_{\mathfrak R}$ is $\mu_{\mathfrak R}$-strongly convex on $\mathbb B_R(\theta_{\mathrm{ref}})$ and $\beta^2$-smooth. Then
$F_{\mathfrak R,\lambda}$ is $(\mu_{\mathfrak R}+\lambda)$-strongly convex on $\mathbb B_R(\theta_{\mathrm{ref}})$ and $(\beta^2+\lambda)$-smooth.
\end{proposition}

Thus $K$ controls trimming/downweighting and $\lambda$ controls explicit curvature; at $K=0$, the regularized variants reduce to regularized DPO and retain only the regularization bias relative to clean DPO.

\begin{corollary}[Time-uniform stability for regularized Robust-DPO]\label{cor:reg-robust-dpo}
Under Assumptions~\ref{ass:log-linear-policy} and~\ref{ass:bounded-region}, fix $\mathfrak R\in\{G,\tau\}$, $\lambda>0$, and an attack set $A\subseteq[N]$ with $|A|\le K$.
Run full-batch gradient descent on $\tilde F_{\mathfrak R,\lambda,A}$ and clean DPO on $F$ with the same step size $0<\eta\le 1/(\beta^2+\lambda)$ and initialization $\theta_0$.
Then for every $T\ge 1$,
\begin{equation*}
\|\theta_T^{\mathrm{atk}}-\theta_T^{\mathrm{cln}}\|_2
\;\le\;
\frac{\mathsf B_{\mathfrak R,\lambda}^{\mathrm{cln}}(R)+\Delta_{\mathfrak R}-\mathsf B_{\mathfrak R}^{\mathrm{cln}}(R)}{\mu_{\mathfrak R}+\lambda}
\;\le\;
\frac{\bar\Delta_{\mathfrak R}+\lambda R}{\mu_{\mathfrak R}+\lambda}.
\end{equation*}
Here $\Delta_{\mathfrak R}-\mathsf B_{\mathfrak R}^{\mathrm{cln}}(R)$ is the previously derived surviving-attack bound: $K\delta(K)/(N-K)$ for $\mathfrak R=G$ and $\Gamma_\tau(K)/W_\tau$ for $\mathfrak R=\tau$.
\end{corollary}

\begin{corollary}[Policy deviation for regularized Robust-DPO]\label{cor:reg-policy-kl}
Under the conditions of Corollary~\ref{cor:reg-robust-dpo}, for either $\mathfrak R\in\{G,\tau\}$ and every $T\ge 1$,
\begin{equation}\label{eq:reg-policy-kl}
\mathbb E_x\!\left[
D_{\mathrm{KL}}\!\big(\pi_{\theta_T^{\mathrm{atk}}}(\cdot|x)\,\|\,\pi_{\theta_T^{\mathrm{cln}}}(\cdot|x)\big)
\right]
\;\le\;
\frac{1}{2}\left(\frac{\bar\Delta_{\mathfrak R}+\lambda R}{\mu_{\mathfrak R}+\lambda}\right)^2.
\end{equation}
\end{corollary}

Corollary~\ref{cor:reg-robust-dpo} makes the regularization trade-off explicit: the surviving-attack term is divided by $\mu_{\mathfrak R}+\lambda$, but the comparison to unregularized clean DPO pays an additional bias of at most $\lambda R$.
Thus regularization is most useful when this curvature gain dominates the clean bias, 
for example when targeted flips are not concentrated among the highest-impact samples,
which is observed in the targeted attack experiment below.
When $K=0$, the bound measures only the regularization offset.

\section{Experiments}\label{sec:experiments}

\textbf{Setup.} 
We evaluate on SHP~\citep{ethayarajh2022understanding} with $160{,}000$ training and $40{,}000$ test pairs,
using $\ell_2$-normalized \texttt{all-MiniLM-L6-v2} pairwise features ($d=384$), $\beta=1$, 
and SGD initialized at an independently trained $\theta_{\mathrm{ref}}$.
The four scenarios are no attack, random corruption, top-impact attack, and targeted BMP-A attack~\citep{yang2026efficient}.
Corrupted scenarios flip $30\%$ of the training data ($K=48{,}000$), while the no-attack control uses $K=0$.
The $30\%$ setting is a severe stress test, not an estimate of real-world attack prevalence.
We compare \texttt{DPO}, Trim-$K$ (\texttt{TRIM}), Soft-Trim-$K$ (\texttt{SOFT\_TRIM}), and their regularized variants (\texttt{\_REG}), with $\tau=20{,}000$ and $\lambda=10^{-3}$.
The matched full-data \texttt{DPO\_REG} control uses the same attacked data and quadratic penalty $\lambda\|\theta-\theta_{\mathrm{ref}}\|_2^2/2$, without trimming or downweighting.
Full constructions are in Appendices~\ref{app:shp_simulation_details}--\ref{app:bmp}.

\textbf{Metric.} 
Let $\theta^\star$ be the final standard DPO parameter trained on the full clean set, and write $\pi^\star=\pi_{\theta^\star}$.
The primary metric is candidate-set KL ($\mathrm{candKL}$), the KL from $\pi^\star$ to the evaluated policy on held-out prompt-specific candidate sets of size $100$, so $\mathrm{candKL}(\theta^\star)=0$.
Appendix~\ref{app:shp_simulation_details} discusses the interpretation of this metric.

\textbf{Top-impact attack.} 
Figure~\ref{fig:allrdpo-topimpact} shows that \texttt{TRIM} and \texttt{SOFT\_TRIM} reduce candKL from $0.155$ for DPO to $0.018$ and $0.064$, respectively. 
Both unregularized defenses outperform their REG counterparts in this setting: directly controlling the high-impact pairs is effective, 
while regularization introduces additional bias.

\textbf{BMP-A attack.} 
Figure~\ref{fig:allrdpo-targeted} shows that \texttt{TRIM\_REG} and \texttt{SOFT\_TRIM\_REG} attain candKL $0.1176$ and $0.1232$, compared with $0.1296$ for matched full-data \texttt{DPO\_REG}. 
These reductions of $9.3\%$ and $4.9\%$ demonstrate an additional benefit from impact-based trimming beyond the shared quadratic regularizer in this targeted-attack setting.

\textbf{Additional experiments.} 
We include comprehensive additional experiments in Appendices~\ref{app:noatk-random}--\ref{app:nonlinear-extension}. In particular, appendices~\ref{app:noatk-random}--\ref{app:baselines} report no-attack/random-corruption controls, clean-data retention costs, parameter-distance curves, and comparisons with existing robust-DPO baselines. 
Matched-regularization controls and score-assignment ablations further separate the effects of trimming and regularization in Appendices~\ref{app:matched-regularization}--\ref{app:score-ablation}. 
Sweeps on SHP, HH-RLHF, and UltraFeedback-Binarized in Appendix~\ref{app:multidataset-budget} show lower KL than DPO for both REG variants at $20\%$--$30\%$ attack budgets, while at smaller budgets retention and regularization bias can outweigh the realized attack effect.
A fixed-attack study examines sensitivity to the defender's assumed $K$ and budget selection in Appendix~\ref{app:k-selection}. 
A separate preliminary nonlinear-policy extension study is reported in Appendix~\ref{app:nonlinear-extension}.

\begin{figure}[tbp]
\centering
\begin{subfigure}[t]{0.49\textwidth}
  \centering
  \includegraphics[width=\linewidth]{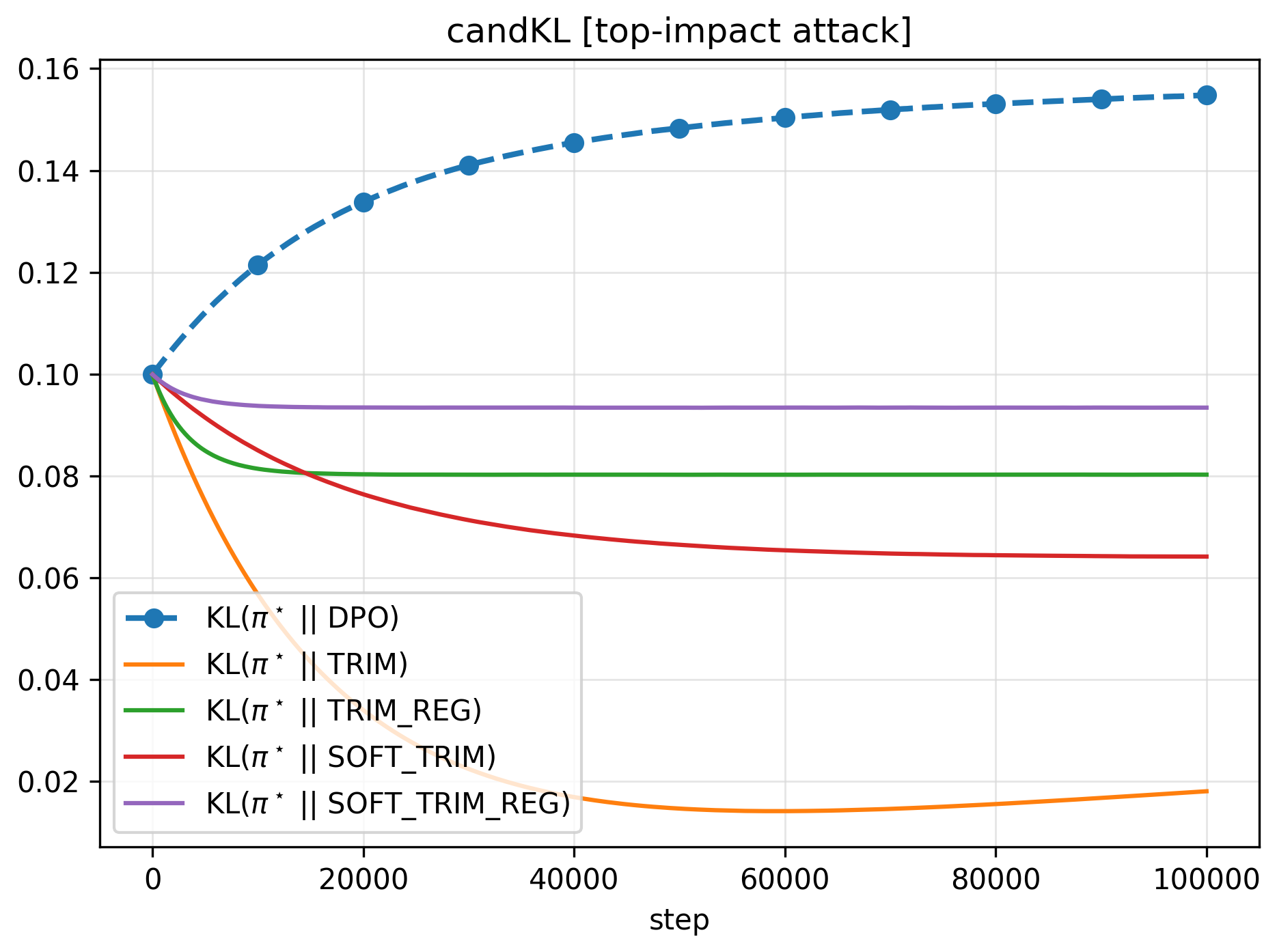}
  \caption{Top-impact attack.}
  \label{fig:allrdpo-topimpact}
\end{subfigure}
\hfill
\begin{subfigure}[t]{0.49\textwidth}
  \centering
  \includegraphics[width=\linewidth]{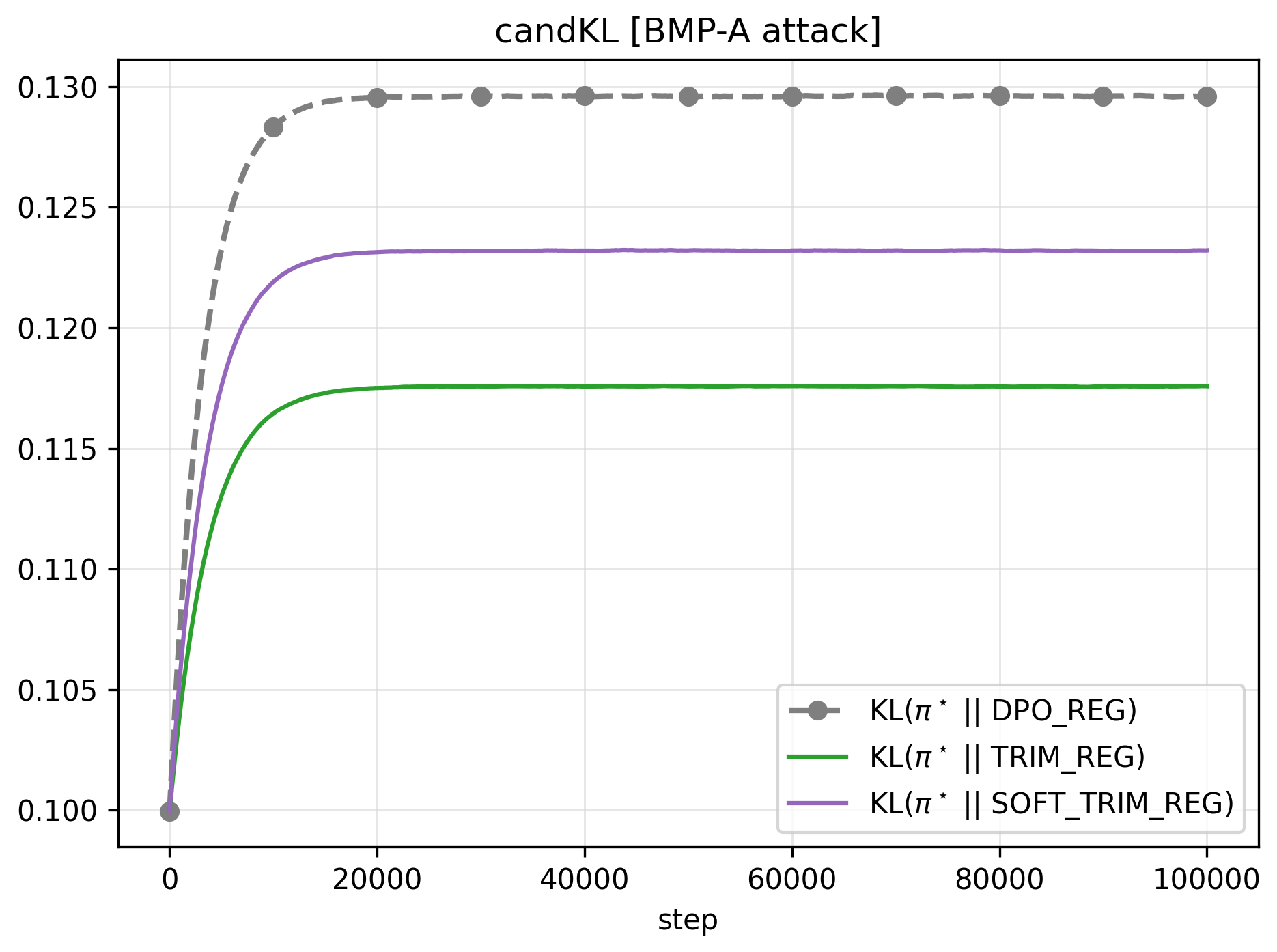}
  \caption{BMP-A attack.}
  \label{fig:allrdpo-targeted}
\end{subfigure}
\caption{Candidate-set $\mathrm{KL}(\pi^\star\|\cdot)$ under two fixed $30\%$ label-flip attacks. Solid curves: our proposed methods. Dashed curves: DPO (left) and matched \texttt{DPO\_REG} (right). $\lambda=10^{-3}$. Lower is better.}
\label{fig:allrdpo-attacks}
\end{figure}

\section{Conclusion}\label{sec:conclusion}

We have designed DPO methods robust to budgeted label-flip attacks in the log-linear setting.
Building on the parameter-independent gradient shift property, 
we have proposed Trim-$K$ Robust-DPO and Soft-Trim-$K$ Robust-DPO, as well as their regularized variants.
For all these methods, we have derived explicit worst-case bounds relative to the true clean DPO objective, 
separating the clean bias of the defense rule from the surviving adversarial bias, 
and translated these bounds into guarantees on parameter and policy deviation.
Under a bounded training region, the optimized robust objective admits a positive curvature constant, 
which leads to time-uniform robustness guarantees. 
Regularization increases this curvature term, while its deterministic bias relative to unregularized clean DPO must also be accounted for.
These results provide a principled robustness characterization of trimmed DPO objectives uniformly over all fixed budgeted label-flip attacks.
Experiments on SHP, HH-RLHF, and UltraFeedback-Binarized are consistent with our analysis and support the behavior of the proposed methods:
trimming is strong for top-impact attacks, while regularization helps control deviation in the targeted BMP-A experiment.
Establishing robustness guarantees beyond the log-linear model remains a future direction, and the preliminary results are discussed in Appendix~\ref{app:nonlinear-extension}.


\newpage

\subsection*{AI use statement}

In this work, we used generative AI tools for aid and polish the writing of this manuscript, 
including improving clarity, grammar, and consistency of notation.
We have not used generative AI tools for other tasks with required disclosure.
Additionally, we used generative AI tools for cross-checking of experimental descriptions against the implementation and saved results. 
We have reviewed all AI-assisted work. 
We checked LLM-polished grammar and verified the experimental descriptions against our code.
We take responsibility for the final content of this work,
including text, claims or artifacts produced with the aid of generative AI.

\subsection*{Reproducibility statement}

The policy class, attack model, and training assumptions are specified in Section~\ref{sec:problem-formulation}, with further discussion of the bounded training region in Appendix~\ref{app:bounded-region}.
Appendix~\ref{app:algorithms} gives algorithm pseudocode, and Appendices~\ref{app:supporting-lemmas}--\ref{app:proofs} provide supporting lemmas and complete proofs for the log-linear analysis.
Appendices~\ref{app:shp_simulation_details}--\ref{app:bmp} describe the main SHP experiment's data processing, feature construction, reference-policy training, SGD hyperparameters, compute resources, evaluation metrics, and attack construction.
Appendices~\ref{app:noatk-random}--\ref{app:score-ablation} document the additional attack scenarios, baseline methods, matched regularization controls, and score-assignment ablation.
Appendices~\ref{app:multidataset-budget}--\ref{app:k-selection} specify data splits, seed protocols, hyperparameters, and evaluation procedures for the three-dataset budget sweep and budget-misspecification study.
Appendix~\ref{app:nonlinear-extension} presents the nonlinear flip-gradient identity and its proof, together with the preliminary nonlinear-policy experiments' architecture, training and evaluation settings, and results.


\bibliography{Robust_DPO_Bib}
\bibliographystyle{iclr2027_conference}


\appendix

\section{Introduction: Additional Related Work}\label{app:related-work}

\subsection{Bandit Problems: Adversarial Attacks and Robust Defenses}

Adversarial attacks have been studied in several bandit settings, 
including multi-armed bandits (MAB)~\citep{jun2018adversarial,liu2019data,liu2020action,ma2023adversarial,yang2024stochastic},
linear contextual bandits~\citep{ma2018data,garcelon2020adversarial,liu2021efficient},
and dueling bandits~\citep{gajane2015relative,agarwal2021stochastic,saha2021adversarial}.
These works show that bandit algorithms can be sensitive to adversarially corrupted feedback or interaction signals.

In MAB, the learner chooses one arm in each round and observes the corresponding reward~\citep{robbins1952some,auer1995gambling,auer2002nonstochastic}.
An adversary may corrupt rewards~\citep{jun2018adversarial,ma2023adversarial,liu2019data},
or manipulate the selected actions~\citep{liu2020action} to force the learner to pull a target arm more often.
Most existing studies focus on stochastic rewards, where each arm has a fixed reward distribution. 
Other works consider adversarial rewards, where rewards are chosen arbitrarily~\citep{auer2002nonstochastic,zimmert2019an},
or non-stationary rewards, where reward distributions change over time within some extent~\citep{yang2024stochastic,garivier2011upper,besbes2014stochastic,auer2019adaptively,besbes2019optimal,cheung2022hedging}.

In linear contextual bandits, the expected reward is assumed to be linear in the context features~\citep{auer2002using,abe2003reinforcement,li2010contextual}.
Existing attacks include reward poisoning~\citep{ma2018data,garcelon2020adversarial}, context poisoning~\citep{garcelon2020adversarial}, and action poisoning~\citep{liu2021efficient}.

Dueling bandits are closer to RLHF, in which the learner selects a pair of arms and observes only which one is preferred~\citep{yue2012k,zoghi2014relative,sui2018advancements}.
Prior work studies stochastic dueling bandits with corrupted pairs~\citep{agarwal2021stochastic},
as well as adversarial preference models~\citep{gajane2015relative,saha2021adversarial,lyu2025adversarial}.

Adversarial attacks naturally motivate robust bandit algorithms. 
Existing robust defenses have been studied for stochastic bandits~\citep{lykouris2018stochastic, gupta2019better, guan2020robust, bogunovic2021stochastic},
adversarial bandits~\citep{yang2020adversarial},
stochastic linear contextual bandits~\citep{ding2022robust},
strategic manipulation in stochastic bandits~\citep{feng2020intrinsic},
stochastic Lipschitz bandits~\citep{kang2023robust},
and non-stationary bandits~\citep{yang2024stochastic}.
These defenses are important in real applications, where corrupted feedback can lead to harmful decisions such as unfair recommendations or unsafe medical treatments. 

\subsection{Standard Reinforcement Learning: Adversarial Attacks and Robust Defenses}

Adversarial attacks on standard RL have also been widely studied, 
in which the attacker corrupts multiple aspects of the learning process so that the learned policy is misled. 
Common poisoning models include reward/cost poisoning~\citep{huang2019deceptive,zhang2020adaptive,rakhsha2020policy,rakhsha2021reward,banihashem2022admissible,xu2022efficient,li2025online};
action poisoning~\citep{liu2021provably};
and more general environment or transition dynamics poisoning~\citep{rakhsha2020policy,xu2021transferable}.
Poisoning attacks are also studied in batch settings~\citep{ma2019policy},
and in dynamics-agnostic settings~\citep{sun2020vulnerabilityaware,liu2021provably,rakhsha2021reward,xu2022efficient,li2025online}.

Beyond single-agent RL, adversarial attacks have also been considered in multi-agent RL~\citep{liu2023efficient,wu2023reward},
and in federated RL~\citep{ma2024reward}.
Other works characterize the limits of reward-only or action-only poisoning in episodic RL~\citep{rangi2022understanding},
or formulate the optimal attack through the attacker's own Markov decision process~\citep{mcmahan2024optimal}.
In deep RL, the policies are neural networks and the attackers modify the high-dimensional observations~\citep{huang2017adversarial,kos2017delving,lin2017tactics,pattanaik2017robust,mandlekar2017adversarially,pan2019characterizing,gleave2019adversarial,sun2020stealthy}.

Existing works have proposed defense strategies in RL algorithms~\citep{tessler2019action, zhang2022corruption, ye2023corruption}.
Related robustness ideas have also been developed in broader machine learning settings~\citep{dong2024byzantine,li2021adversarial,pillutla2022robust}.
These works show that defense mechanisms are necessary when corrupted interactions can push the learner toward undesirable behaviors.

\subsection{Reinforcement Learning from Human Feedback: Adversarial Attacks and Robust Defenses}

RLHF is different from bandit and standard RL because it learns policies from pairwise discrete preference data rather than a numerical reward from the environment.
Depending on whether the training is on a pre-collected dataset or through real-time interaction with human annotators~\citep{xiong2023iterative},
RLHF can be categorized into offline RLHF~\citep{christiano2017deep,ouyang2022training,bai2022training,zhu2023principled,rafailov2023direct} and online RLHF~\citep{chen2022human,saha2023dueling,wang2023rlhf,zhan2023provablereward,wu2024making}.
It has been shown that offline RLHF is vulnerable to preference poisoning, such as data injection attack and label flipping attack,
through experiments~\citep{rando2023universal,wu2024preference,baumgartner2024best,pathmanathan2024poisoning,fu2025poisonbench,pathmanathan2025poisoning,wu2025preference},
and theoretical analysis~\citep{nika2025policy}.
\citet{nika2025policy} proves that a large poisoned preference set may be needed to mislead offline RLHF for data injection attack.
Meanwhile, \citet{yang2026efficient} proposes efficient label flipping attacks that strategically select a small subset of preference pairs to flip and mislead offline RLHF toward a target policy.
In online RLHF, preferences can also be attacked to manipulate the learning process~\citep{yang2025human}.

Other than the adversarial attack described above, RLHF is also sensitive to the random noise in preference data, 
which is caused by annotation errors, limited expert involvement, and other non-adversarial sources~\citep{casper2023open,yan2024reward}.
This noise may lead to unpredictable or biased aligned model behavior. 

Motivated by these reliability issues of preference data, 
existing works propose robust online RLHF against adversarial attacks~\citep{yang2025human}
and robust offline RLHF against random noise~\citep{yan2024reward,cheng2024rime,mandal2024corruption,chowdhury2024provably,wu2024towards} and distribution shift~\citep{xu2025robust}.
In the offline setting, there is a gap in the literature on robust RLHF against adversarial attacks, especially for specifically designed attacks leading to a target policy.
This paper aims to fill this gap by proposing robust DPO methods against budgeted preference label flips, including adversarial flips.

\section{Problem Formulation: Why the Bounded Training Region Matters}\label{app:bounded-region}

\begin{remark}[Why the bounded-region Assumption~\ref{ass:bounded-region} matters]\label{rem:dpo-not-strongly-convex}
Without Assumption~\ref{ass:bounded-region}, the standard DPO objective need not be globally strongly convex on all of $\mathbb R^d$. In the log-linear setting, the per-sample Hessian is
\begin{equation*}
\nabla^2 \ell_i(\theta)
=
\beta^2\,\sigma(z_i(\theta))\big(1-\sigma(z_i(\theta))\big)\,v_i v_i^\top,
\end{equation*}
and $\inf_\theta \sigma(z_i(\theta))\big(1-\sigma(z_i(\theta))\big) = 0$ as $\|\theta-\theta_{\mathrm{ref}}\|_2\to\infty$ along separable directions.
Thus, global strong convexity is not guaranteed even though the bounded-region curvature constant $\mu_G$ in~\eqref{eq:trim-mu-G} is positive.
\end{remark}

The strong-convexity statement in~\cite{nika2025policy} does not contradict the above discussion. 
The key distinction is whether the curvature comes from the logistic part on a bounded training region or from an explicit quadratic regularizer. 
In the notation of~\cite{nika2025policy}, their parameter displacement is
\begin{equation*}
    v := \theta-\theta_\mu,
\end{equation*}
and their regularized DPO/logistic objective includes the quadratic term
\begin{equation*}
    \frac{\lambda}{2}\|\theta-\theta_\mu\|_2^2 .
\end{equation*}
For the unregularized logistic part, their Hessian lower bound contains a data covariance term multiplied by a logistic-curvature constant:
\begin{equation*}
    \nabla^2 \ell_v(D) \succeq n\beta C_v \Sigma_D,
\end{equation*}
where
\begin{equation*}
    C_v := \min_i
    \frac{\exp(\beta y_i v^\top x_i+b)}{
    \bigl(1+\exp(\beta y_i v^\top x_i+b)\bigr)^2}.
\end{equation*}
Here, their $C_v$ is exactly the minimum logistic curvature factor over the dataset. 
In our notation, their center $\theta_\mu$ corresponds to the reference parameter $\theta_{\mathrm{ref}}$, their displacement $v$ corresponds to $\theta-\theta_{\mathrm{ref}}$, and their signed feature $y_i x_i$ corresponds to our pairwise feature direction
\begin{equation*}
    v_i=\psi(x^i,y_w^i)-\psi(x^i,y_l^i),
\end{equation*}
up to the sign convention used to write the logistic loss. 
Thus their scalar argument $\beta y_i v^\top x_i+b$ corresponds to our DPO logit $z_i(\theta)$.

At any fixed finite parameter displacement, $C_v>0$. 
More importantly for our setting, if training is restricted to a bounded ball around the reference parameter, then compactness of this ball and finiteness of the dataset imply a positive lower bound
\begin{equation}\label{eq:appendix-cbar-R}
    \underline C_R
    :=
    \inf_{\|\theta-\theta_{\mathrm{ref}}\|_2\le R}
    \min_i
    \frac{\exp(z_i(\theta))}{(1+\exp(z_i(\theta)))^2}
    >0 .
\end{equation}
Consequently, on the bounded training region, the unregularized DPO objective is strongly convex whenever the corresponding empirical feature covariance is positive definite on the trained parameter subspace. 
This is the local curvature used in our unregularized Trim-$K$ and Soft-Trim-$K$ bounds.

This local statement does not imply global strong convexity on all of $\mathbb R^d$. 
Indeed, for
\begin{equation*}
    q(z):=\frac{e^z}{(1+e^z)^2},
\end{equation*}
we have $q(z)>0$ for every finite $z$, but $q(z)\to0$ as $z\to\pm\infty$. 
Therefore, along a direction where some sample logit diverges, the corresponding logistic curvature factor vanishes, and the global infimum of $C_v$ is zero. 
This is why the unregularized objective need not admit a uniform global strong-convexity constant.

The quadratic regularizer in~\cite{nika2025policy} gives a different source of curvature. 
It adds $\lambda I$ to the Hessian, so the regularized objective can obtain an additional global curvature term even when the logistic curvature factor is small. 
This is the same mechanism used by our regularized variants in Section~\ref{sec:regularized-robust-dpo}: the denominator improves by $\lambda$, while comparison to the unregularized clean DPO target also introduces the deterministic regularization bias term $\lambda R$.

\begin{proposition}[Linear-in-$T$ deviation under convex smoothness alone]\label{prop:linearT-necessary}
For any $\bar L\in(0,\beta^2]$, there exist a convex $\bar L$-smooth function $F$, a constant bias vector $b\neq 0$, and a step size $0<\eta\le 2/\bar L$ such that, when clean and attacked gradient descent are run from the same initialization,
\begin{equation*}
\theta_{t+1}^{\mathrm{cln}}=\theta_t^{\mathrm{cln}}-\eta\nabla F(\theta_t^{\mathrm{cln}}),
\qquad
\theta_{t+1}^{\mathrm{atk}}=\theta_t^{\mathrm{atk}}-\eta(\nabla F(\theta_t^{\mathrm{atk}})+b),
\end{equation*}
the iterate deviation satisfies
\begin{equation*}
\|\theta_T^{\mathrm{atk}}-\theta_T^{\mathrm{cln}}\|_2=\eta T\|b\|_2
\qquad \text{for all } T\ge 1.
\end{equation*}
Thus, convex smoothness plus a uniform gradient-bias bound is not enough to imply a time-uniform deviation bound. 
\end{proposition}

\begin{proof}[Proof of Proposition~\ref{prop:linearT-necessary}]
Choose any nonzero vector $b$ and any $0<\eta\le2/\bar L$.
Take $F(\theta)\equiv 0$ for all $\theta$. Then $F$ is convex and $\bar L$-smooth for every $\bar L>0$,
and $\nabla F(\theta)=0$ for all $\theta$. The clean GD iterates are constant:
\begin{equation*}
\theta_{t+1}^{\mathrm{cln}}=\theta_t^{\mathrm{cln}}.
\end{equation*}
The attacked GD iterates satisfy
\begin{equation*}
\theta_{t+1}^{\mathrm{atk}}=\theta_t^{\mathrm{atk}}-\eta b.
\end{equation*}
Starting from the same initialization $\theta_0^{\mathrm{atk}}=\theta_0^{\mathrm{cln}}$, we obtain by induction
\begin{equation*}
\theta_T^{\mathrm{atk}}-\theta_T^{\mathrm{cln}} = -\eta T b,
\end{equation*}
and hence $\|\theta_T^{\mathrm{atk}}-\theta_T^{\mathrm{cln}}\|_2=\eta T\|b\|_2$.
\end{proof}

Proposition~\ref{prop:linearT-necessary} explains why Assumption~\ref{ass:bounded-region} is important for time-uniform robustness. 
Without a bounded training region, the DPO objective may not have a uniform curvature lower bound, 
so the stability argument would fall back to the accumulated-bias recursion over gradient steps, leading to a deviation that grows linearly in $T$, illustrated by the proposition. 
The bounded-region assumption prevents this failure mode by keeping the logistic curvature uniformly positive along training. 
For the regularized variants, the quadratic regularizer provides an additional curvature term, but its additional bias means that it need not improve the complete deviation bound.

\section{Methods: Algorithm Details}\label{app:algorithms}

Algorithms~\ref{alg:trimk} and~\ref{alg:softtrim} give the implementation details for applying the objective-level defenses described in the main text to the observed dataset.
In the log-linear setting, the score $s_i=\beta\|v_i\|_2$ is the exact magnitude of a pair's potential flip-induced gradient change.
It does not estimate whether the pair is corrupted.
Because $\|\psi(x,y_w)-\psi(x,y_l)\|_2$ is invariant to swapping the two responses, this score can be computed directly from the observed pair even if it has been flipped.
Given precomputed $d$-dimensional features, computing all $N$ scores costs $O(Nd)$ and sorting them costs $O(N\log N)$, once before training.
Training then uses standard DPO updates on the fixed retained set or with fixed scalar weights on the per-example losses.
The REG variants additionally evaluate the quadratic penalty and its gradient, requiring $O(d)$ arithmetic per update.

\begin{algorithm}[H]
\caption{Trim-$K$ Robust-DPO}
\label{alg:trimk}
\begin{algorithmic}[1]
\Require Observed preference dataset $\tilde{\mathcal D}(A)=\{(x^i,\tilde y_w^i,\tilde y_l^i)\}_{i=1}^N$, budget upper bound $K$, reference policy $\pi_{\mathrm{ref}}$, inverse-temperature $\beta$, optimizer, training horizon $T$
\For{$i=1,\dots,N$}
    \State Compute observed direction $\tilde v_i \gets \psi(x^i,\tilde y_w^i)-\psi(x^i,\tilde y_l^i)$
    \State Compute score $s_i \gets \beta \|\tilde v_i\|_2$
\EndFor
\State Sort $\{s_i\}_{i=1}^N$ in descending order and keep the $N-K$ smallest-score indices $G$
\State Initialize parameter $\theta_0$
\For{$t=0,1,\dots,T-1$}
    \State Compute the observed trimmed DPO objective
    \begin{equation*}
    \tilde F_{G,A}(\theta_t)=\frac{1}{N-K}\sum_{i\in G}\ell_{\mathrm{DPO}}(\theta_t;x^i,\tilde y_w^i,\tilde y_l^i)
    \end{equation*}
    \State Update $\theta_{t+1}\gets \theta_t-\eta \nabla \tilde F_{G,A}(\theta_t)$
\EndFor
\State \Return $\theta_T$
\end{algorithmic}
\end{algorithm}

\begin{algorithm}[H]
\caption{Soft-Trim-$K$ Robust-DPO}
\label{alg:softtrim}
\begin{algorithmic}[1]
\Require Observed preference dataset $\tilde{\mathcal D}(A)=\{(x^i,\tilde y_w^i,\tilde y_l^i)\}_{i=1}^N$, budget upper bound $K$, temperature $\tau$, reference policy $\pi_{\mathrm{ref}}$, inverse-temperature $\beta$, optimizer, training horizon $T$
\For{$i=1,\dots,N$}
    \State Compute observed direction $\tilde v_i \gets \psi(x^i,\tilde y_w^i)-\psi(x^i,\tilde y_l^i)$
    \State Compute score $s_i \gets \beta \|\tilde v_i\|_2$
\EndFor
\State Rank the scores in descending order and assign weights $w_i^{(\tau)}$ by~\eqref{eq:exp-soft-trim-weight}
\State Initialize parameter $\theta_0$
\For{$t=0,1,\dots,T-1$}
    \State Compute the observed weighted DPO objective
    \begin{equation*}
    \tilde F_{\tau,A}(\theta_t)=\frac{1}{W_\tau}\sum_{i=1}^N
    w_i^{(\tau)}\ell_{\mathrm{DPO}}(\theta_t;x^i,\tilde y_w^i,\tilde y_l^i)
    \end{equation*}
    \State Update $\theta_{t+1}\gets \theta_t-\eta \nabla \tilde F_{\tau,A}(\theta_t)$
\EndFor
\State \Return $\theta_T$
\end{algorithmic}
\end{algorithm}

\section{Methods: Soft-Trim Weight Visualization}\label{app:soft-trim-weights}

Figure~\ref{fig:exp-soft-trim-weights} visualizes the rank-based weights used by Soft-Trim-$K$ Robust-DPO. 
Smaller values of $\tau$ more aggressively downweight the highest-score samples and approach Trim-$K$ Robust-DPO, while larger values of $\tau$ keep the weights closer to one and approach standard DPO.

\begin{figure}[ht]
\centering
\includegraphics[width=0.75\textwidth]{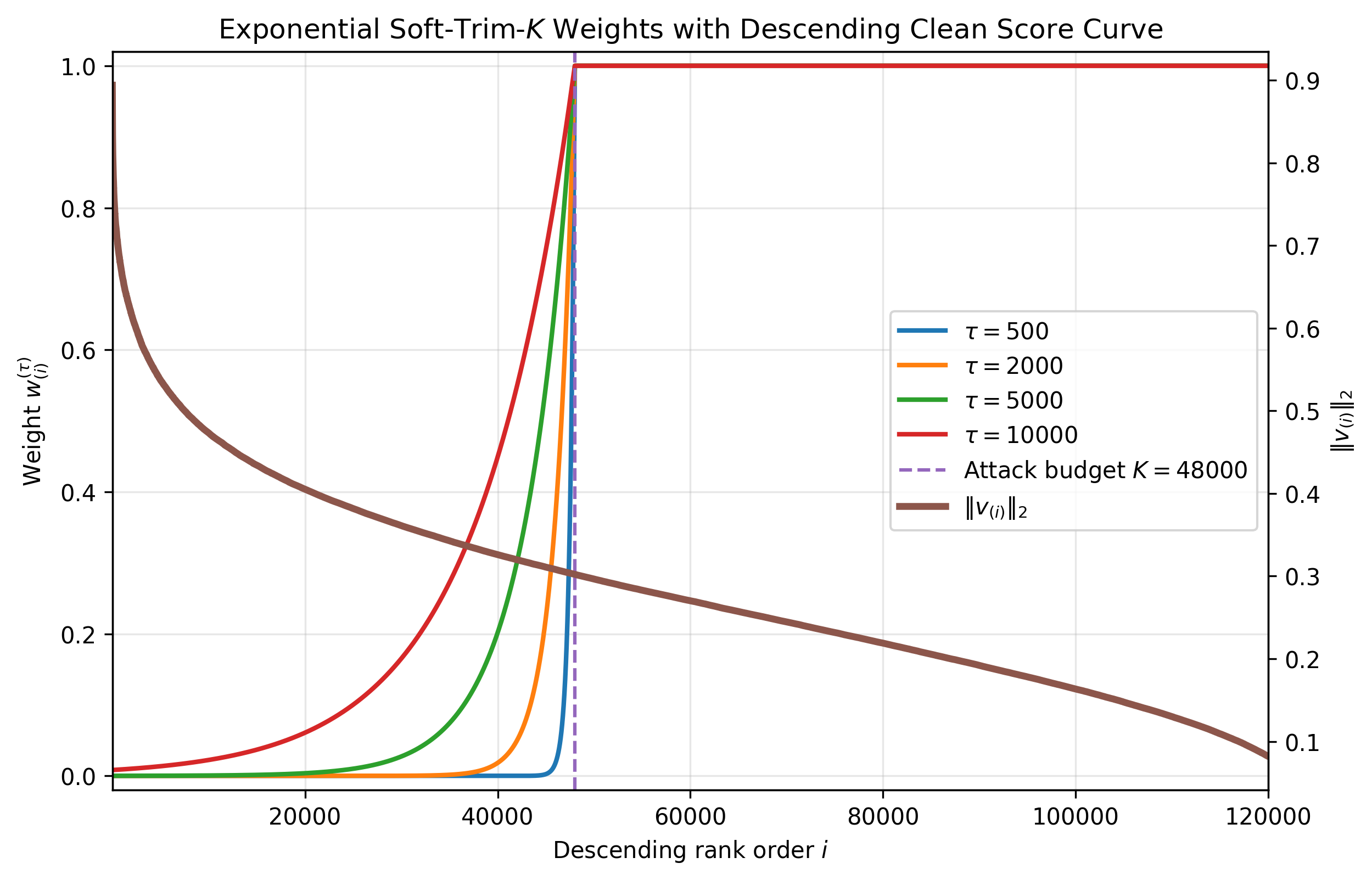}
\caption{Left vertical axis: Soft-Trim-$K$ weights $w_i^{(\tau)}$ as a function of the rank $r(i)$, for different temperatures $\tau$.
Right vertical axis: the impact scores $s_i=\beta\|v_i\|_2$ as a function of rank, with $\beta=1$. The trimming cutoff at $K=48{,}000$ is indicated by the vertical dashed line.}
\label{fig:exp-soft-trim-weights}
\end{figure}

\section{Methods: Supporting Lemmas for the Main Analysis}\label{app:supporting-lemmas}

This appendix collects the standard smoothness, curvature, and optimization facts used in Section~\ref{sec:methods}.

\begin{lemma}[Smoothness of per-sample and normalized log-linear DPO objectives]\label{lem:dpo-smooth-affine}
Under Assumption~\ref{ass:log-linear-policy}, for every sample $i$, the clean per-sample loss $\ell_i$ defined in~\eqref{eq:dpo-per-sample} is convex and $L_i$-smooth, where
\begin{equation*}
L_i\le \frac{\beta^2}{4}\|v_i\|_2^2\le \beta^2.
\end{equation*}
In particular, for all $\theta,\theta'$,
\begin{equation*}
\|\nabla \ell_i(\theta)-\nabla \ell_i(\theta')\|_2
\le
\beta^2\|\theta-\theta'\|_2.
\end{equation*}
The same conclusion holds for the observed loss $\tilde\ell_i$, because swapping the preference orientation replaces $v_i$ by $-v_i$ and preserves its norm. Consequently, for any fixed weights $a_i\ge 0$ with $\sum_i a_i=1$, both normalized objectives
\begin{equation*}
F_a(\theta):=\sum_i a_i\ell_i(\theta),
\qquad
\tilde F_a(\theta):=\sum_i a_i\tilde\ell_i(\theta)
\end{equation*}
are convex and $\beta^2$-smooth. In particular, this applies to $F$, $F_G$, $\tilde F_{G,A}$, $F_\tau$, and $\tilde F_{\tau,A}$.
\end{lemma}

\begin{lemma}[Restricted strong convexity of the trimmed DPO objective]\label{lem:trimk-restricted-sc}
Under Assumptions~\ref{ass:log-linear-policy} and~\ref{ass:bounded-region}, the clean trimmed objective $F_G$ introduced in Section~\ref{sec:trim-k-robust-dpo} is $\mu_G$-strongly convex on $\mathbb B_R(\theta_{\mathrm{ref}})$ and $\beta^2$-smooth, where $\mu_G$ is defined in~\eqref{eq:trim-mu-G}.
\end{lemma}

\begin{lemma}[Closed-form bound on the clean Trim-$K$ bias]\label{lem:trim-bias-bound}
Under Assumption~\ref{ass:log-linear-policy}, for every $\theta\in\mathbb R^d$,
\begin{equation*}
\|\nabla F_G(\theta)-\nabla F(\theta)\|_2
\;\le\;
\frac{1}{N}\!\left(K\,\delta(K)+\sum_{j=1}^K s_{(j)}\right).
\end{equation*}
Consequently,
\begin{equation*}
\mathsf B_G^{\mathrm{cln}}(R)
\;\le\;
\frac{K\,\delta(K)+\sum_{j=1}^K s_{(j)}}{N}
\;\le\;
\frac{2}{N}\sum_{j=1}^K s_{(j)}
\;\le\;
\frac{4\beta K}{N}.
\end{equation*}
Moreover, for the regularized clean bias $\mathsf B_{G,\lambda}^{\mathrm{cln}}(R)$, 
\begin{equation*}
\mathsf B_{G,\lambda}^{\mathrm{cln}}(R)
\;\le\;
\frac{K\,\delta(K)+\sum_{j=1}^K s_{(j)}}{N}+\lambda R.
\end{equation*}
\end{lemma}

\begin{lemma}[Closed-form bound on the clean Soft-Trim-$K$ bias]\label{lem:softtrim-bias-bound}
Define
\begin{equation}\label{eq:Ctau-clean-bound}
\mathsf C_\tau
:=
\frac{1}{N W_\tau}
\left[
\sum_{j=1}^K
\left|D_\tau-N\bigl(1-a_j^{(\tau)}\bigr)\right|s_{(j)}
+
D_\tau\sum_{j=K+1}^N s_{(j)}
\right].
\end{equation}
Under Assumption~\ref{ass:log-linear-policy}, for every $\theta\in\mathbb R^d$,
\begin{equation*}
\|\nabla F_\tau(\theta)-\nabla F(\theta)\|_2
\;\le\;
\mathsf C_\tau.
\end{equation*}
Consequently,
\begin{equation*}
\mathsf B_\tau^{\mathrm{cln}}(R)
\;\le\;
\mathsf C_\tau
\;\le\;
\frac{D_\tau\bar s+\sum_{j=1}^K(1-a_j^{(\tau)})s_{(j)}}{W_\tau}
\;\le\;
\frac{4\beta K}{W_\tau}.
\end{equation*}
Moreover, for the regularized clean bias $\mathsf B_{\tau,\lambda}^{\mathrm{cln}}(R)$, 
\begin{equation*}
\mathsf B_{\tau,\lambda}^{\mathrm{cln}}(R)
\;\le\;
\mathsf C_\tau+\lambda R.
\end{equation*}
As $\tau\downarrow 0$, $\mathsf C_\tau$ reduces to the exact hard-trim type bound
\begin{equation*}
\frac{1}{N(N-K)}
\left[
(N-K)\sum_{j=1}^K s_{(j)}
+
K\sum_{j=K+1}^N s_{(j)}
\right],
\end{equation*}
which is at most $\bigl(K\delta(K)+\sum_{j=1}^K s_{(j)}\bigr)/N$.
As $\tau\to\infty$, all weights converge to one, $D_\tau\to 0$, and $\mathsf C_\tau\to 0$.
\end{lemma}

\begin{lemma}[Non-expansiveness of gradient descent for convex smooth objectives]\label{lem:gd-nonexpansive}
Let $F:\mathbb R^d\to\mathbb R$ be convex and $L$-smooth. Define $T(\theta):=\theta-\eta\nabla F(\theta)$. If $0<\eta\le 2/L$, then $T$ is non-expansive:
\begin{equation*}
\|T(x)-T(y)\|_2\le \|x-y\|_2,
\qquad \forall x,y\in\mathbb R^d.
\end{equation*}
\end{lemma}

\begin{lemma}[Time-uniform stability under objective perturbation]\label{lem:sc-time-uniform}
Let $F_0:\mathbb R^d\to\mathbb R$ be the target objective and let $H:\mathbb R^d\to\mathbb R$ be the objective optimized by the attacked run.
Assume $H$ is $\mu$-strongly convex and $L$-smooth on a convex region containing the clean and attacked iterates, and assume
\begin{equation*}
\sup_{\theta\ \text{in the region}}\|\nabla H(\theta)-\nabla F_0(\theta)\|_2\le B.
\end{equation*}
Run full-batch gradient descent from the same initialization with step size $0<\eta\le 1/L$:
\begin{equation*}
\theta_{t+1}^{\mathrm{cln}}=\theta_t^{\mathrm{cln}}-\eta\nabla F_0(\theta_t^{\mathrm{cln}}),
\qquad
\theta_{t+1}^{\mathrm{atk}}=\theta_t^{\mathrm{atk}}-\eta\nabla H(\theta_t^{\mathrm{atk}}).
\end{equation*}
Then for every $T\ge 1$,
\begin{equation*}
\|\theta_T^{\mathrm{atk}}-\theta_T^{\mathrm{cln}}\|_2
\le
\frac{B}{\mu}\big(1-(1-\eta\mu)^T\big)
\le
\frac{B}{\mu}.
\end{equation*}
\end{lemma}

\section{Methods: Proofs of Theoretical Results}\label{app:proofs}

\subsection{Proof of Lemma~\ref{lem:dpo-smooth-affine}}
\begin{proof}
Fix a sample $i$. Since $b_i$ does not depend on $\theta$, $\nabla_\theta z_i(\theta)=\beta v_i$. Differentiating,
\begin{equation*}
\nabla \ell_i(\theta)= -\big(1-\sigma(z_i(\theta))\big)\,\beta v_i,
\qquad
\nabla^2 \ell_i(\theta)= \beta^2 \sigma(z_i(\theta))\big(1-\sigma(z_i(\theta))\big)\, v_i v_i^\top.
\end{equation*}
Thus, using $\sigma(z_i(\theta))(1-\sigma(z_i(\theta)))\le 1/4$ and $\|v_i v_i^\top\|_{\mathrm{op}}=\|v_i\|_2^2$,
\begin{equation*}
\|\nabla^2 \ell_i(\theta)\|_{\mathrm{op}}
\le \frac{\beta^2}{4}\|v_i\|_2^2.
\end{equation*}
Under Assumption~\ref{ass:log-linear-policy},
$\|v_i\|_2\le \|\psi(x^i,y_w^i)\|_2+\|\psi(x^i,y_l^i)\|_2\le 2$, hence
$\|\nabla^2 \ell_i(\theta)\|_{\mathrm{op}}\le \beta^2$ for all $\theta$.
A uniform Hessian bound implies gradient Lipschitzness:
\begin{equation*}
\|\nabla \ell_i(\theta)-\nabla \ell_i(\theta')\|_2
\le \sup_{\xi}\|\nabla^2 \ell_i(\xi)\|_{\mathrm{op}}\,\|\theta-\theta'\|_2
\le \beta^2\|\theta-\theta'\|_2.
\end{equation*}
The Hessian is positive semidefinite, so $\ell_i$ is convex. For the observed orientation, either the pair is unchanged or $(v_i,b_i,z_i)$ becomes $(-v_i,-b_i,-z_i)$. The same Hessian bound and convexity conclusion therefore hold for $\tilde\ell_i$.

Finally, for any fixed normalized nonnegative weighted average $F_a(\theta)=\sum_i a_i\ell_i(\theta)$,
\begin{equation*}
\|\nabla^2 F_a(\theta)\|_{\mathrm{op}}
=\Big\|\sum_i a_i\nabla^2 \ell_i(\theta)\Big\|_{\mathrm{op}}
\le \sum_i a_i\|\nabla^2 \ell_i(\theta)\|_{\mathrm{op}}
\le \beta^2,
\end{equation*}
and $\nabla^2F_a(\theta)\succeq 0$. The identical argument applies to $\tilde F_a$ using the observed losses $\tilde\ell_i$. This proves the stated claims for the fixed normalized clean and attacked empirical objectives.
\end{proof}

\subsection{Proof of Lemma~\ref{lem:trimk-restricted-sc}}
\begin{proof}
For $\theta\in\mathbb B_R(\theta_{\mathrm{ref}})$, the Hessian of the clean trimmed objective is
\begin{equation*}
\nabla^2 F_G(\theta)
=
\frac{1}{|G|}\sum_{i\in G}
\beta^2\sigma(z_i(\theta))\bigl(1-\sigma(z_i(\theta))\bigr)v_i v_i^\top.
\end{equation*}
By the definition of $\underline c_G(R)$,
\begin{equation*}
\nabla^2 F_G(\theta)
\succeq
\beta^2\underline c_G(R)\frac{1}{|G|}\sum_{i\in G}v_i v_i^\top
=
\beta^2\underline c_G(R)\Sigma_G
\succeq
\mu_G I.
\end{equation*}
Therefore $F_G$ is $\mu_G$-strongly convex on the bounded region. Its $\beta^2$-smoothness follows from Lemma~\ref{lem:dpo-smooth-affine}.
\end{proof}

\subsection{Proof of Lemma~\ref{lem:trim-bias-bound}}
\begin{proof}
Let $T:=[N]\setminus G$ denote the trimmed set, so $|T|=K$ and $|G|=N-K$. Splitting $\sum_{i=1}^N g_i(\theta)=\sum_{i\in G}g_i(\theta)+\sum_{i\in T}g_i(\theta)$, we obtain
\begin{align*}
\nabla F_G(\theta)-\nabla F(\theta)
&=
\frac{1}{N-K}\sum_{i\in G}g_i(\theta)
-\frac{1}{N}\sum_{i=1}^N g_i(\theta)\\
&=
\left(\frac{1}{N-K}-\frac{1}{N}\right)\sum_{i\in G}g_i(\theta)
-\frac{1}{N}\sum_{i\in T}g_i(\theta)\\
&=
\frac{K}{N(N-K)}\sum_{i\in G}g_i(\theta)
-\frac{1}{N}\sum_{i\in T}g_i(\theta).
\end{align*}
By the closed-form gradient~\eqref{eq:dpo-per-sample-grad} and $|\sigma(z)-1|\le 1$, we have $\|g_i(\theta)\|_2\le \beta\|v_i\|_2=s_i$ for every $\theta$. Triangle inequality gives
\begin{align*}
\|\nabla F_G(\theta)-\nabla F(\theta)\|_2
&\le
\frac{K}{N(N-K)}\sum_{i\in G}s_i+\frac{1}{N}\sum_{i\in T}s_i.
\end{align*}
By construction of $G$, every $i\in G$ satisfies $s_i\le \delta(K)$, hence $\sum_{i\in G}s_i\le (N-K)\delta(K)$. Also, $T$ contains exactly the $K$ largest scores, so $\sum_{i\in T}s_i=\sum_{j=1}^K s_{(j)}$. Substituting,
\begin{equation*}
\|\nabla F_G(\theta)-\nabla F(\theta)\|_2
\le
\frac{K\,\delta(K)}{N}+\frac{1}{N}\sum_{j=1}^K s_{(j)}
=
\frac{K\,\delta(K)+\sum_{j=1}^K s_{(j)}}{N}.
\end{equation*}
Taking the supremum over $\theta\in\mathbb B_R(\theta_{\mathrm{ref}})$ gives the bound on $\mathsf B_G^{\mathrm{cln}}(R)$.
Since $\delta(K)=s_{(K+1)}\le s_{(j)}$ for $j\le K$, $K\delta(K)\le \sum_{j=1}^K s_{(j)}$, yielding the simplified $\frac{2}{N}\sum_{j=1}^K s_{(j)}$. Each $s_i=\beta\|v_i\|_2\le 2\beta$ under Assumption~\ref{ass:log-linear-policy}, hence $\sum_{j=1}^K s_{(j)}\le 2\beta K$, giving the universal bound $\frac{4\beta K}{N}$.
For the regularized bias, note
\begin{equation*}
\nabla F_G(\theta)+\lambda(\theta-\theta_{\mathrm{ref}})-\nabla F(\theta)
=
(\nabla F_G(\theta)-\nabla F(\theta))+\lambda(\theta-\theta_{\mathrm{ref}}),
\end{equation*}
and $\|\theta-\theta_{\mathrm{ref}}\|_2\le R$ on $\mathbb B_R(\theta_{\mathrm{ref}})$, so the regularized bound follows by the triangle inequality.
\end{proof}

\subsection{Proof of Lemma~\ref{lem:softtrim-bias-bound}}
\begin{proof}
Write the samples in descending score order, so that the top-$K$ weights are $a_j^{(\tau)}$ for $j\le K$ and the remaining weights are one.
Recall from Section~\ref{sec:soft-trim-k-robust-dpo} that $D_\tau=K-\sum_{j=1}^K a_j^{(\tau)}=N-W_\tau$. Since
\begin{equation*}
\nabla F_\tau(\theta)-\nabla F(\theta)
=
\sum_{i=1}^N\left(\frac{w_i^{(\tau)}}{W_\tau}-\frac{1}{N}\right)g_i(\theta),
\end{equation*}
and $\|g_i(\theta)\|_2\le s_i$, the triangle inequality gives
\begin{equation*}
\|\nabla F_\tau(\theta)-\nabla F(\theta)\|_2
\le
\sum_{i=1}^N
\left|\frac{w_i^{(\tau)}}{W_\tau}-\frac{1}{N}\right|s_i .
\end{equation*}
For a top-$K$ sample with rank $j$, the absolute coefficient is
\begin{equation*}
\left|\frac{a_j^{(\tau)}}{W_\tau}-\frac{1}{N}\right|
=
\frac{\left|N a_j^{(\tau)}-W_\tau\right|}{N W_\tau}
=
\frac{\left|D_\tau-N(1-a_j^{(\tau)})\right|}{N W_\tau}.
\end{equation*}
For a sample with rank $j>K$, the coefficient is
\begin{equation*}
\left|\frac{1}{W_\tau}-\frac{1}{N}\right|
=
\frac{D_\tau}{N W_\tau}.
\end{equation*}
Combining these two displays proves
\begin{equation*}
\|\nabla F_\tau(\theta)-\nabla F(\theta)\|_2
\le
\mathsf C_\tau .
\end{equation*}
Taking the supremum over $\theta\in\mathbb B_R(\theta_{\mathrm{ref}})$ gives $\mathsf B_\tau^{\mathrm{cln}}(R)\le \mathsf C_\tau$.
Also, for every $j\le K$,
\begin{equation*}
\left|D_\tau-N(1-a_j^{(\tau)})\right|
\le
D_\tau+N(1-a_j^{(\tau)}),
\end{equation*}
which implies
\begin{equation*}
\mathsf C_\tau
\le
\frac{D_\tau\bar s+\sum_{j=1}^K(1-a_j^{(\tau)})s_{(j)}}{W_\tau}.
\end{equation*}

The same simpler upper bound also follows directly from the alternative decomposition
\begin{equation*}
\nabla F_\tau(\theta)-\nabla F(\theta)
=
\frac{1}{W_\tau}
\left[
D_\tau\nabla F(\theta)
-
\sum_{j=1}^K(1-a_j^{(\tau)})g_{(j)}(\theta)
\right],
\end{equation*}
where $g_{(j)}$ denotes the gradient of the sample with the $j$-th largest score.
Using $\|\nabla F(\theta)\|_2\le \bar s$ and $\|g_{(j)}(\theta)\|_2\le s_{(j)}$ gives
\begin{equation*}
\mathsf B_\tau^{\mathrm{cln}}(R)
\le
\frac{D_\tau\bar s+\sum_{j=1}^K(1-a_j^{(\tau)})s_{(j)}}{W_\tau}.
\end{equation*}
Finally, $\bar s\le 2\beta$, $D_\tau=\sum_{j=1}^K(1-a_j^{(\tau)})\le K$, and $s_{(j)}\le 2\beta$ imply the universal bound $4\beta K/W_\tau$.
The regularized bound follows by adding $\lambda(\theta-\theta_{\mathrm{ref}})$ and using $\|\theta-\theta_{\mathrm{ref}}\|_2\le R$.
The two limiting statements follow by taking $a_j^{(\tau)}\to 0$ as $\tau\downarrow0$ and $a_j^{(\tau)}\to1$ as $\tau\to\infty$ in the definition of $\mathsf C_\tau$.
\end{proof}

\subsection{Proof of Lemma~\ref{lem:gd-nonexpansive}}

\begin{proof}
Let $d:=x-y$ and $\Delta g:=\nabla F(x)-\nabla F(y)$. 
We first recall that for convex $L$-smooth $F$,
$\nabla F$ is $1/L$-cocoercive~\citep{bauschke2009baillon}:
\begin{equation}\label{eq:cocoercive}
\langle d,\Delta g\rangle \;\ge\; \frac{1}{L}\|\Delta g\|_2^2.
\end{equation}

Now expand the distance after one GD step:
\begin{equation*}
\|T(x)-T(y)\|_2^2
=
\|d-\eta \Delta g\|_2^2
=
\|d\|_2^2 - 2\eta\langle d,\Delta g\rangle + \eta^2\|\Delta g\|_2^2.
\end{equation*}
Applying~\eqref{eq:cocoercive} yields
\begin{equation*}
\|T(x)-T(y)\|_2^2
\le
\|d\|_2^2 + \Big(\eta^2-\frac{2\eta}{L}\Big)\|\Delta g\|_2^2.
\end{equation*}
If $0<\eta\le 2/L$, then $\eta^2-\frac{2\eta}{L}\le 0$, hence $\|T(x)-T(y)\|_2^2\le \|d\|_2^2$,
which implies $\|T(x)-T(y)\|_2\le \|x-y\|_2$.
\end{proof}

\subsection{Proof of Lemma~\ref{lem:sc-time-uniform}}
\begin{proof}
Let $T_H(\theta):=\theta-\eta\nabla H(\theta)$ and $T_0(\theta):=\theta-\eta\nabla F_0(\theta)$.
Since $H$ is $\mu$-strongly convex and $L$-smooth, and $0<\eta\le 1/L$, the gradient-descent map $T_H$ is a contraction~\citep{ryu2016primer}:
\begin{equation*}
\|T_H(x)-T_H(y)\|_2\le (1-\eta\mu)\|x-y\|_2 .
\end{equation*}
Therefore, with $\Theta_t:=\theta_t^{\mathrm{atk}}-\theta_t^{\mathrm{cln}}$,
\begin{align*}
\|\Theta_{t+1}\|_2
&=
\|T_H(\theta_t^{\mathrm{atk}})-T_0(\theta_t^{\mathrm{cln}})\|_2\\
&\le
\|T_H(\theta_t^{\mathrm{atk}})-T_H(\theta_t^{\mathrm{cln}})\|_2
+
\eta\|\nabla H(\theta_t^{\mathrm{cln}})-\nabla F_0(\theta_t^{\mathrm{cln}})\|_2\\
&\le
(1-\eta\mu)\|\Theta_t\|_2+\eta B.
\end{align*}
Unrolling this recursion from $\Theta_0=0$ gives
\begin{equation*}
\|\Theta_T\|_2
\le
\eta B\sum_{t=0}^{T-1}(1-\eta\mu)^{T-1-t}
=
\frac{B}{\mu}\big(1-(1-\eta\mu)^T\big)
\le
\frac{B}{\mu}.
\end{equation*}
\end{proof}

\subsection{Proof of Lemma~\ref{lem:constant-shift}}
We include the proof in our notation, following the gradient-shift analysis in~\citet{yang2026efficient}.
\begin{proof}
For sample $i$, swapping $(y_w^i,y_l^i)$ changes $(v_i,b_i)$ to $(-v_i,-b_i)$ and changes the logit to $z_i^{\mathrm{flip}}(\theta)=-z_i(\theta)$.
Using $\sigma(-z)=1-\sigma(z)$ and $g_i(\theta)=\beta(\sigma(z_i(\theta))-1)v_i$, we obtain
\begin{equation*}
g_i^{\mathrm{flip}}(\theta)
=
\beta(\sigma(-z_i(\theta))-1)(-v_i)
=
\beta\sigma(z_i(\theta))v_i.
\end{equation*}
Therefore $g_i^{\mathrm{flip}}(\theta)-g_i(\theta)=\beta v_i$.
\end{proof}

\subsection{Proof of Theorem~\ref{thm:K-robust-grad}}
\begin{proof}
By the decomposition before the theorem,
\begin{equation*}
\nabla \tilde F_{G,A}(\theta)-\nabla F(\theta)
=
\big(\nabla F_G(\theta)-\nabla F(\theta)\big)
+
\frac{1}{|G|}\sum_{i\in A\cap G}\beta v_i .
\end{equation*}
The first term is bounded by $\mathsf B_G^{\mathrm{cln}}(R)$ by definition.
For the second term,
\begin{equation*}
\left\|\frac{1}{|G|}\sum_{i\in A\cap G}\beta v_i\right\|_2
\le \frac{1}{|G|}\sum_{i\in A\cap G}\|\beta v_i\|_2
= \frac{1}{|G|}\sum_{i\in A\cap G}s_i
\le \frac{|A\cap G|}{|G|}\max_{i\in G}s_i
\le \frac{K}{N-K}\delta(K).
\end{equation*}
Combining the two bounds proves the claim.
\end{proof}

\subsection{Proof of Theorem~\ref{thm:trim-K-iter-stability}}
\begin{proof}
The Hessian of $\tilde F_{G,A}$ is the same as the Hessian of $F_G$, because label swaps replace $(v_i,z_i)$ by $(-v_i,-z_i)$ and leave
$\sigma(z_i)(1-\sigma(z_i))v_i v_i^\top$ unchanged.
Thus, by Lemma~\ref{lem:trimk-restricted-sc}, $\tilde F_{G,A}$ is $\mu_G$-strongly convex and $\beta^2$-smooth on the bounded region.
Theorem~\ref{thm:K-robust-grad} gives the uniform perturbation bound
\begin{equation*}
\sup_{\theta\in\mathbb B_R(\theta_{\mathrm{ref}})}
\|\nabla\tilde F_{G,A}(\theta)-\nabla F(\theta)\|_2
\le
\mathsf B_G^{\mathrm{cln}}(R)+\frac{K}{N-K}\delta(K).
\end{equation*}
Applying Lemma~\ref{lem:sc-time-uniform} with $F_0=F$, $H=\tilde F_{G,A}$, $\mu=\mu_G$, $L=\beta^2$, and
\begin{equation*}
B=\mathsf B_G^{\mathrm{cln}}(R)+\frac{K}{N-K}\delta(K)
\end{equation*}
proves the stated bound.
\end{proof}

\subsection{Proof of Corollary~\ref{cor:trim-K-policy-kl}}
\begin{proof}
Fix $x$ and define the log-partition
\begin{equation*}
A_x(\theta):=\log\sum_{y}\exp\!\big(\langle \theta,\psi(x,y)\rangle\big).
\end{equation*}
Under Assumption~\ref{ass:log-linear-policy} ($\|\psi(x,y)\|_2\le 1$), we have
$\|\nabla^2 A_x(\theta)\|_{\mathrm{op}}\le 1$ for all $\theta$, hence $A_x$ is $1$-smooth.
Moreover, $\pi_\theta(\cdot|x)$ is the exponential family induced by $A_x$, and the conditional KL admits the standard Bregman form
\begin{equation*}
D_{\mathrm{KL}}\!\big(\pi_{\theta}(\cdot|x)\,\|\,\pi_{\theta'}(\cdot|x)\big)
=
A_x(\theta')-A_x(\theta)-\langle \nabla A_x(\theta),\theta'-\theta\rangle.
\end{equation*}
By $1$-smoothness of $A_x$, the descent inequality implies that for all $\theta,\theta'$,
\begin{equation*}
A_x(\theta') \le A_x(\theta) + \langle \nabla A_x(\theta),\,\theta'-\theta\rangle
+ \frac{1}{2}\|\theta'-\theta\|_2^2.
\end{equation*}
Rearranging yields the Bregman bound
\begin{equation*}
A_x(\theta')-A_x(\theta)-\langle \nabla A_x(\theta),\,\theta'-\theta\rangle
\le \frac{1}{2}\|\theta'-\theta\|_2^2.
\end{equation*}
Thus
\begin{equation*}
D_{\mathrm{KL}}\!\big(\pi_{\theta}(\cdot|x)\,\|\,\pi_{\theta'}(\cdot|x)\big)
\le \frac12\|\theta-\theta'\|_2^2.
\end{equation*}
Substituting $(\theta,\theta')=(\theta_T^{\mathrm{atk}},\theta_T^{\mathrm{cln}})$ proves the first claim, and taking expectation over $x$ preserves the inequality.
Finally, apply Theorem~\ref{thm:trim-K-iter-stability} to bound $\|\theta_T^{\mathrm{atk}}-\theta_T^{\mathrm{cln}}\|_2$.
\end{proof}

\subsection{Proof of Theorem~\ref{thm:softtrim-grad}}
\begin{proof}
For any attack set $A$ with $|A|\le K$,
\begin{equation*}
\nabla\tilde F_{\tau,A}(\theta)-\nabla F(\theta)
=
\big(\nabla F_\tau(\theta)-\nabla F(\theta)\big)
+
\frac{1}{W_\tau}\sum_{i\in A}w_i^{(\tau)}\beta v_i .
\end{equation*}
The first term is bounded by $\mathsf B_\tau^{\mathrm{cln}}(R)$ by definition.
For the second term, using $\|\beta v_i\|_2=s_i$ gives
\begin{equation*}
\left\|\frac{1}{W_\tau}\sum_{i\in A}w_i^{(\tau)}\beta v_i\right\|_2
\le
\frac{1}{W_\tau}\sum_{i\in A}w_i^{(\tau)}s_i
=
\frac{1}{W_\tau}\sum_{i\in A}\rho_i^{(\tau)}
\le
\frac{\Gamma_\tau(K)}{W_\tau},
\end{equation*}
where the last inequality follows from the definition of $\Gamma_\tau(K)$ as the largest possible sum of $K$ weighted impacts.
This proves~\eqref{eq:softtrim-exact-bound}.

For the exponential weights, every original top-$K$ sample has weight at most $\varepsilon_\tau=e^{-1/\tau}$, and every sample outside the original top-$K$ has score at most $\delta(K)$.
Therefore, for any $A$ with $|A|\le K$,
\begin{equation*}
\sum_{i\in A}\rho_i^{(\tau)}
\le
\varepsilon_\tau\sum_{j=1}^K s_{(j)}+K\delta(K).
\end{equation*}
Taking the maximum over $A$ proves the upper-bound for $\Gamma_\tau(K)$. 
\end{proof}

\subsection{Proof of Theorem~\ref{thm:softtrim-iter-stability}}
\begin{proof}
The Hessian of $\tilde F_{\tau,A}$ is the same as the Hessian of $F_\tau$, because label swaps leave each matrix
$\sigma(z_i)(1-\sigma(z_i))v_i v_i^\top$ unchanged.
The weighted clean objective satisfies
\begin{equation*}
\nabla^2F_\tau(\theta)
=
\frac{1}{W_\tau}\sum_{i=1}^N
w_i^{(\tau)}\beta^2\sigma(z_i(\theta))\bigl(1-\sigma(z_i(\theta))\bigr)v_i v_i^\top.
\end{equation*}
By the definitions of $\underline c_\tau(R)$ and $\Sigma_\tau$,
\begin{equation*}
\nabla^2F_\tau(\theta)
\succeq
\beta^2\underline c_\tau(R)\Sigma_\tau
\succeq
\mu_\tau I
\end{equation*}
for all $\theta\in\mathbb B_R(\theta_{\mathrm{ref}})$.
Thus $\tilde F_{\tau,A}$ is $\mu_\tau$-strongly convex on the bounded region and is $\beta^2$-smooth by Lemma~\ref{lem:dpo-smooth-affine}.
Theorem~\ref{thm:softtrim-grad} gives
\begin{equation*}
\sup_{\theta\in\mathbb B_R(\theta_{\mathrm{ref}})}
\|\nabla\tilde F_{\tau,A}(\theta)-\nabla F(\theta)\|_2
\le
\mathsf B_\tau^{\mathrm{cln}}(R)+\frac{\Gamma_\tau(K)}{W_\tau}.
\end{equation*}
Applying Lemma~\ref{lem:sc-time-uniform} with $F_0=F$, $H=\tilde F_{\tau,A}$, $\mu=\mu_\tau$, $L=\beta^2$, and
\begin{equation*}
B=\mathsf B_\tau^{\mathrm{cln}}(R)+\frac{\Gamma_\tau(K)}{W_\tau}
\end{equation*}
proves the first displayed bound. The second follows by substituting the coarse bound~\eqref{eq:softtrim-coarse-bound}.
\end{proof}

\subsection{Proof of Corollary~\ref{cor:soft-trim-K-policy-kl}}
\begin{proof}
The KL-to-parameter conversion is the same as in the proof of Corollary~\ref{cor:trim-K-policy-kl}. Applying Theorem~\ref{thm:softtrim-iter-stability} gives the stated bound.
\end{proof}

\subsection{Proof of Proposition~\ref{prop:reg-dpo-curvature}}
\begin{proof}
On the bounded region, $F_{\mathfrak R}$ is $\mu_{\mathfrak R}$-strongly convex and $\beta^2$-smooth by assumption. 
Adding the quadratic regularizer adds Hessian $\lambda I$. Hence
\begin{equation*}
\nabla^2 F_{\mathfrak R,\lambda}(\theta)=\nabla^2 F_{\mathfrak R}(\theta)+\lambda I \succeq (\mu_{\mathfrak R}+\lambda)I,
\end{equation*}
and
\begin{equation*}
\|\nabla^2 F_{\mathfrak R,\lambda}(\theta)\|_{\mathrm{op}}
\le
\|\nabla^2 F_{\mathfrak R}(\theta)\|_{\mathrm{op}}+\lambda
\le
\beta^2+\lambda.
\end{equation*}
This proves $(\mu_{\mathfrak R}+\lambda)$-strong convexity and $(\beta^2+\lambda)$-smoothness.
\end{proof}

\subsection{Proof of Corollary~\ref{cor:reg-robust-dpo}}\label{app:reg-robust-proof}
\begin{proof}
Fix $\mathfrak R\in\{G,\tau\}$ and an admissible attack set $A$.
Assumption~\ref{ass:bounded-region} supplies the bounded training region and curvature on the directions updated by both runs.
By Lemma~\ref{lem:constant-shift}, label flips change the gradient of either fixed robust objective by a constant vector.
Consequently, $\tilde F_{\mathfrak R,\lambda,A}$ has the same Hessian as $F_{\mathfrak R,\lambda}$ and, by Proposition~\ref{prop:reg-dpo-curvature}, is $(\mu_{\mathfrak R}+\lambda)$-strongly convex and $(\beta^2+\lambda)$-smooth on that region, in the trained coordinates of Assumption~\ref{ass:bounded-region}.
The regularizer cancels from the attacked-to-clean robust difference, giving
\begin{equation*}
\nabla\tilde F_{\mathfrak R,\lambda,A}(\theta)-\nabla F(\theta)
=\bigl(\nabla F_{\mathfrak R,\lambda}(\theta)-\nabla F(\theta)\bigr) 
+\bigl(\nabla\tilde F_{\mathfrak R,A}(\theta)-\nabla F_{\mathfrak R}(\theta)\bigr).
\end{equation*}
The first term has norm at most $\mathsf B_{\mathfrak R,\lambda}^{\mathrm{cln}}(R)$.
The second has norm at most $\Delta_{\mathfrak R}-\mathsf B_{\mathfrak R}^{\mathrm{cln}}(R)$: this is $K\delta(K)/(N-K)$ by the proof of Theorem~\ref{thm:K-robust-grad} when $\mathfrak R=G$, and $\Gamma_\tau(K)/W_\tau$ by the proof of Theorem~\ref{thm:softtrim-grad} when $\mathfrak R=\tau$.
The common initialization and Assumption~\ref{ass:bounded-region} place both trajectories in the same affine training space. Restricting to its coordinates preserves the updates and cannot increase the gradient-perturbation bound.
Apply Lemma~\ref{lem:sc-time-uniform} in those coordinates with $F_0=F$, $H=\tilde F_{\mathfrak R,\lambda,A}$, $\mu=\mu_{\mathfrak R}+\lambda$, $L=\beta^2+\lambda$, and
\begin{equation*}
B=\mathsf B_{\mathfrak R,\lambda}^{\mathrm{cln}}(R)+\Delta_{\mathfrak R}-\mathsf B_{\mathfrak R}^{\mathrm{cln}}(R)
\end{equation*}
to obtain the first inequality.
For the second, use $\mathsf B_{\mathfrak R,\lambda}^{\mathrm{cln}}(R)\le\mathsf B_{\mathfrak R}^{\mathrm{cln}}(R)+\lambda R$ and $\Delta_{\mathfrak R}\le\bar\Delta_{\mathfrak R}$.
The latter follows from Lemma~\ref{lem:trim-bias-bound} for $G$, and Theorem~\ref{thm:softtrim-grad} together with Lemma~\ref{lem:softtrim-bias-bound} for $\tau$.
\end{proof}

\subsection{Proof of Corollary~\ref{cor:reg-policy-kl}}\label{app:reg-policy-proof}
\begin{proof}
Fix either $\mathfrak R\in\{G,\tau\}$.
As in the proof of Corollary~\ref{cor:soft-trim-K-policy-kl}, Assumption~\ref{ass:log-linear-policy} implies that the log-partition function is $1$-smooth.
Its Bregman-divergence representation of conditional KL therefore gives, for every prompt $x$,
\begin{equation*}
D_{\mathrm{KL}}\!\big(\pi_{\theta_T^{\mathrm{atk}}}(\cdot|x)\,\|\,\pi_{\theta_T^{\mathrm{cln}}}(\cdot|x)\big)
\le \frac12\|\theta_T^{\mathrm{atk}}-\theta_T^{\mathrm{cln}}\|_2^2.
\end{equation*}
The parameter-distance term is independent of $x$. Taking expectation over prompts and applying Corollary~\ref{cor:reg-robust-dpo} yields
\begin{equation*}
\mathbb E_x\!\left[
D_{\mathrm{KL}}\!\big(\pi_{\theta_T^{\mathrm{atk}}}(\cdot|x)\,\|\,\pi_{\theta_T^{\mathrm{cln}}}(\cdot|x)\big)
\right]
\le \frac12\|\theta_T^{\mathrm{atk}}-\theta_T^{\mathrm{cln}}\|_2^2
\le \frac12\left(\frac{\bar\Delta_{\mathfrak R}+\lambda R}{\mu_{\mathfrak R}+\lambda}\right)^2,
\end{equation*}
which proves~\eqref{eq:reg-policy-kl}.
\end{proof}

\section{Experiments: Additional details for the SHP simulation}\label{app:shp_simulation_details}

\paragraph{Data and feature construction.}
The SHP simulation uses $200{,}000$ decoded preference pairs from the SHP training split (seed $0$), followed by an $80/20$ train/test split that yields $160{,}000$ training pairs and $40{,}000$ held-out test pairs.
For each example, we use the feature difference $v_i$ defined in~\eqref{eq:dpo-logit-affine}. 
We concatenate the prompt and response with two newline characters, tokenize with a maximum length of $256$ tokens, and extract the final-layer CLS representation from \texttt{sentence-transformers/all-MiniLM-L6-v2}. 
We $\ell_2$-normalize this vector to obtain $\psi(x,y)\in\mathbb{R}^{384}$, so $\|v_i\|_2\le 2$. 
This instantiates the log-linear DPO parameterization.
SHP is a public Reddit-derived preference dataset. We cite its associated paper and follow the dataset-card terms.
The embedding model \texttt{sentence-transformers/all-MiniLM-L6-v2} is released under the Apache-2.0 license.

\paragraph{Reference policy.}
The reference parameter $\theta_{\mathrm{ref}}\in\mathbb{R}^{384}$ is trained \emph{independently of the DPO runs} on the same training split using a pairwise logistic ranking loss
\begin{equation*}
    \mathcal L_{\mathrm{ref}}(\theta)
    \;=\;\mathbb E_i\!\left[\log\!\big(1+\exp(-\langle\theta,v_i\rangle)\big)\right]
    \;+\;\tfrac{\lambda_{\mathrm{ref}}}{2}\,\|\theta\|_2^2,
\end{equation*}
with $\lambda_{\mathrm{ref}}=10^{-3}$, $5{,}000$ mini-batch SGD steps, learning rate $0.25$, and batch size $100{,}000$.
The resulting parameter has $\|\theta_{\mathrm{ref}}\|_2\approx 2.26$.
The reference parameter defines the reference-policy log-ratio in the DPO loss, provides the common initialization, and anchors the quadratic regularizer. It is not treated as a near-optimal policy.
All DPO and robust-DPO runs are initialized at $\theta_0=\theta_{\mathrm{ref}}$.

\paragraph{Optimization and attack configuration.}
All models are trained by SGD with learning rate $0.25$, batch size $25{,}600$, and $100{,}000$ optimization steps, with full-batch evaluation every $100$ steps.
We set $\beta=1$ and attack fraction $0.3$, giving a nominal attacked-scenario budget of $K=\lfloor 0.3\cdot 160{,}000\rfloor = 48{,}000$ flipped training pairs.
Four attack settings are considered:
\begin{itemize}
    \item \textbf{No attack}: the training data is clean.
    \item \textbf{Random corruption}: $K$ training pairs are randomly selected and flipped.
    \item \textbf{Top-impact attack}: $K$ training pairs with the largest impact scores $s_i=\beta\|v_i\|_2$ are flipped.
    \item \textbf{Targeted BMP-A attack}~\citep{yang2026efficient}: $K$ training pairs are selected by BMP-A, a gradient-space label-flip attack that drives DPO toward the target policy $\theta^\dagger = 2\theta_{\mathrm{ref}}-\theta^\star$.
    The detailed procedure and construction of this adversarial attack are provided in Appendix~\ref{app:bmp}.
\end{itemize}
For the attacked scenarios, \texttt{TRIM} and \texttt{TRIM\_REG} remove the top-$K$ examples under the observed impact ranking ($K=48{,}000$, leaving $112{,}000$ training pairs). The REG variants use $\lambda=10^{-3}$.
For the no-attack control, we set $K=0$ while keeping $\lambda=10^{-3}$ for the REG controls. The unregularized proposed defenses then reduce to standard DPO, and the REG variants reduce to full-data regularized DPO.
For \texttt{SOFT\_TRIM} and \texttt{SOFT\_TRIM\_REG}, the full training set is retained with exponential weights at $\tau=20{,}000$. When $K=0$, all weights are one.
We also run a clean-data wrong-$K$ stress test in which the realized attack budget is zero but the defender uses the nominal attacked-scenario budget $K=48{,}000$. 
This measures the clean-data cost of overestimating $K$.

\paragraph{Compute resources.}
The reported runs used Apple Silicon M4 Max with MPS for sequential preprocessing, clean-reference training, and targeted-attack construction. The remaining algorithm cases were dispatched to CPU worker processes.
The logged run used 64GB unified memory, 12 performance cores, 4 efficiency cores, 14 parallel workers, and finished the full no-attack, random, top-impact, and targeted-BMP-A sweep in about two hours.

\paragraph{Compared models and legend names.}
We use the shorthand defined in Section~\ref{sec:experiments}: \texttt{DPO}, \texttt{DPO\_REG}, \texttt{TRIM}, \texttt{TRIM\_REG}, \texttt{SOFT\_TRIM}, and \texttt{SOFT\_TRIM\_REG}. The additional baselines are \texttt{DR\_DPO}, \texttt{NR\_DPO}, \texttt{WDPO}, and \texttt{KLDPO} (Appendix~\ref{app:baselines}).
In the conservative wrong-$K$ scenario, ``DPO (trimmed data)'' denotes DPO trained on the retained clean subset. In the ordinary no-attack scenario ($K=0$), no pairs are removed and the legend uses \texttt{DPO}.
The matched full-data \texttt{DPO\_REG} control uses the same observed labels and $\lambda=10^{-3}$ without trimming or downweighting (Appendix~\ref{app:matched-regularization}).
The clean full-set DPO reference is displayed as $\pi^\star$ in policy KL and as $\theta^\star$ in parameter $L_2$ plots.

\paragraph{Evaluation metrics.}
Our primary metric is the \emph{candidate-set policy distance}, defined as the average categorical KL between two induced finite-action policies on held-out prompts.
For each evaluation prompt $x$, we construct an indexed candidate list of size $M=100$:
\begin{equation*}
\mathcal Y(x)=(y_w,y_l,y^{(1)}_{\mathrm{neg}},\ldots,y^{(M-2)}_{\mathrm{neg}}).
\end{equation*}
The response pool contains both responses from each of the $160{,}000$ training pairs, with duplicate and empty strings removed. We sample $M-2=98$ additional responses with replacement using seed $123$, excluding strings equal to $y_w$ or $y_l$. Repeated sampled responses remain separate candidate entries. Writing $\mathcal Y_j(x)$ for entry $j$, the evaluated categorical policy is
\begin{equation*}
\pi_\theta(j\mid x;\mathcal Y(x))
=
\frac{\exp(\langle \theta,\psi(x,\mathcal Y_j(x))\rangle)}
{\sum_{k=1}^{M} \exp(\langle \theta,\psi(x,\mathcal Y_k(x))\rangle)}.
\end{equation*}
The same candidate lists are fixed across all compared methods and checkpoints.
We then report
\begin{equation*}
\mathrm{candKL}(\theta)
\;:=\;
\mathbb E_x\!\left[
\mathrm{KL}\!\left(
\pi^\star(\cdot\mid x;\mathcal Y(x))
\,\middle\|\,
\pi_{\theta}(\cdot\mid x;\mathcal Y(x))
\right)
\right]
\end{equation*}
averaged over $5{,}000$ held-out prompts.
We also record parameter distance $\|\theta-\theta^\star\|_2$,
but $\mathrm{candKL}$ is the primary metric as it measures the downstream policy deviation on a richer action space, 
and it reflects the effect of attack and defense more directly.

Candidate-set KL measures deviation from the clean-DPO policy on the fixed evaluation candidates. It does not directly measure absolute response quality.

\section{Experiments: BMP-A Targeted Attack: Full Construction}\label{app:bmp}

This appendix explains how we instantiate the Binary Matching Pursuit Attack (BMP-A) from~\citet{yang2026efficient} in our robust DPO experiments.
BMP-A is used only to generate the attacked training set. 
All robust defense methods then train on this same attacked dataset or on the corresponding trimmed subset. 
BMP-A selects flips to match an attacker-chosen target gradient, whereas top-impact selects by magnitude and random corruption selects uniformly. 
Their realized policy effects depend on the target, budget, and data geometry.

\paragraph{Gradient-space reduction.}
Recall the feature difference $v_i$ and affine logit $z_i(\theta)$ from~\eqref{eq:dpo-logit-affine}.
In our SHP simulation, the reference policy is parameterized by $\theta_{\mathrm{ref}}$, so $b_i=\langle\theta_{\mathrm{ref}},v_i\rangle$ and the logit specializes to
\begin{equation*}
    z_i(\theta)=\beta\langle \theta-\theta_{\mathrm{ref}},v_i\rangle .
\end{equation*}
By Lemma~\ref{lem:constant-shift}, flipping pair $i$ changes the clean per-sample gradient $g_i(\theta)$ by the fixed vector $\beta v_i$.
Thus a support $A$ changes the unnormalized gradient sum by $\sum_{i\in A}\beta v_i$, or the gradient of the averaged objective $F$ by $N^{-1}\sum_{i\in A}\beta v_i$.
This is the key observation used in~\citet{yang2026efficient}: targeted label flipping can be formulated as a sparse binary approximation problem in gradient space.

\paragraph{Target policy.}
We choose the adversarial target as the reflection of the clean DPO solution across the reference parameter,
\begin{equation*}
    \theta^\dagger = 2\theta_{\mathrm{ref}}-\theta^\star.
\end{equation*}
This target asks the attacked optimizer to move away from $\theta_{\mathrm{ref}}$ by roughly the same distance as the clean optimizer, but in the opposite direction.
Define the unnormalized clean-gradient sum and its cancelling target vector by
\begin{equation*}
    g^\dagger := \sum_{i=1}^N g_i(\theta^\dagger),
    \qquad
    y := -g^\dagger.
\end{equation*}
If the attacker can find a flip set $A$ such that $\sum_{i\in A}\beta v_i\approx y$, then the attacked DPO gradient at $\theta^\dagger$ is close to zero.
Thus $\theta^\dagger$ becomes an approximate stationary point of the attacked objective.

\paragraph{Sparse binary attack problem.}
Let
\begin{equation*}
    V_{\mathrm{atk}} := [\beta v_1,\ldots,\beta v_N]\in\mathbb R^{d\times N}.
\end{equation*}
The attacker seeks a binary selector $x\in\{0,1\}^N$ with $\|x\|_0\le K$ such that
\begin{equation}\label{eq:bmp-attack-objective}
    V_{\mathrm{atk}}x \approx y.
\end{equation}
Here $x_i=1$ means that pair $i$ is flipped.
This is the same gradient-space attack formulation as in~\citet{yang2026efficient}, specialized to the target $\theta^\dagger$ used in our experiments.

\paragraph{BMP-A selection rule.}
BMP-A is a greedy solver for~\eqref{eq:bmp-attack-objective}.
Starting from residual $r=y$ and empty support $A$, it repeatedly selects an unflipped pair whose normalized column direction has the largest correlation with the current residual:
\begin{equation*}
    i_t\in\arg\max_{i\in[N]\setminus A}
    \frac{|\langle V_{\mathrm{atk},i},r\rangle|}{\max\{\|V_{\mathrm{atk},i}\|_2,10^{-12}\}}.
\end{equation*}
The normalization is important because the columns $\beta v_i$ are not unit-norm.
The numerical floor matches the implementation and gives zero score to zero-impact columns.
After choosing $i_t$, BMP-A sets $x_{i_t}=1$ and updates
\begin{equation*}
    r\leftarrow r-V_{\mathrm{atk},i_t},
\end{equation*}
so that the invariant $r=y-V_{\mathrm{atk}}x$ is maintained.
The procedure stops when the budget $K$ is exhausted or when the residual norm is below the tolerance.
Algorithm~\ref{alg:bmp-a} gives the exact procedure used in our experiments.

\begin{algorithm}[ht]
\caption{Binary Matching Pursuit Attack (BMP-A) used for targeted label flipping}\label{alg:bmp-a}
\begin{algorithmic}[1]
    \Require Attack matrix $V_{\mathrm{atk}}=[\beta v_1,\ldots,\beta v_N]$, target $y=-g^\dagger$, budget $K$, tolerance $\varepsilon$
    \Ensure Binary flip vector $x\in\{0,1\}^N$, flip set $A$
    \State $r\gets y$, \quad $x\gets 0$, \quad $A\gets\emptyset$
    \For{$t=1,2,\ldots,K$}
        \State $i_t\in\arg\max\limits_{i\in[N]\setminus A}\dfrac{|\langle V_{\mathrm{atk},i},r\rangle|}{\|V_{\mathrm{atk},i}\|_2}$
        \State $A\gets A\cup\{i_t\}$, \quad $x_{i_t}\gets 1$
        \State $r\gets r-V_{\mathrm{atk},i_t}$
        \If{$\|r\|_2\le\varepsilon$}
            \State \textbf{break}
        \EndIf
    \EndFor
\end{algorithmic}
\end{algorithm}

\paragraph{Forming the attacked dataset.}
After BMP-A returns the support $A$, we form $\tilde{\mathcal D}(A)$ using the label-flip rule in Section~\ref{sec:threat-model}. This attacked dataset is fixed before defense training and shared by all methods.

\paragraph{Target-strength line search.}
A fixed budget may not be sufficient to reach the full opposite target. 
Therefore, our implementation also considers intermediate targets
\begin{equation*}
    \theta^\dagger(\alpha)
    =\theta_{\mathrm{ref}}+\alpha(\theta^\dagger-\theta_{\mathrm{ref}}),
    \qquad \alpha\in(0,1].
\end{equation*}
For each $\alpha$, we compute $g^\dagger(\alpha)=\sum_i g_i(\theta^\dagger(\alpha))$, set $y(\alpha)=-g^\dagger(\alpha)$, and run BMP-A with budget $K$.
We scan $\alpha$ from large to small over $\{1.0,0.9,\ldots,0.1\}$ and use the first value satisfying
\begin{equation*}
    \frac{\|y(\alpha)-V_{\mathrm{atk}}x\|_2}{\|y(\alpha)\|_2}\le 0.15.
\end{equation*}
If no value meets this criterion, we use the $\alpha$ with the smallest relative residual. 
This selects the largest tested $\alpha$ meeting the residual criterion when one is available.

\paragraph{Realized BMP-A instance.}
In our SHP experiment, $N=160{,}000$, $d=384$, $K=48{,}000$, $\beta=1$, and the residual tolerance is $\varepsilon=10^{-3}$. 
The line search accepts $\alpha=1.0$, so the full target $\theta^\dagger=2\theta_{\mathrm{ref}}-\theta^\star$ is used.
The final BMP-A support has size $48{,}000$, the relative gradient-space residual is $0.094$, and
\begin{equation*}
    \|\theta^\dagger-\theta^\star\|_2\approx 40.3.
\end{equation*}
Training standard DPO on this attacked dataset yields a final parameter distance around $30.7$ from $\theta^\star$. The constructed attack therefore causes a substantial deviation from clean DPO, providing a targeted-attack setting for evaluating the defenses.

\section{Experiments: No-Attack and Random-Corruption Results}\label{app:noatk-random}

This appendix reports results for the no-attack, clean-data wrong-$K$, and random-corruption scenarios, and provides parameter-distance ($\|\theta-\theta^\star\|_2$) curves for all four attack scenarios.
The figures in this section compare standard DPO with the four proposed defenses (\texttt{TRIM}, \texttt{TRIM\_REG}, \texttt{SOFT\_TRIM}, and \texttt{SOFT\_TRIM\_REG}). Comparisons with the four additional baselines appear in Appendix~\ref{app:baselines}.

\subsection{No-attack scenario}

When the training data are clean, the realized attack budget is $K=0$.
With no samples to trim and all soft-trim weights equal to one, \texttt{TRIM} and \texttt{SOFT\_TRIM} reduce exactly to standard DPO.
The \texttt{\_REG} controls keep $\lambda=10^{-3}$, so \texttt{TRIM\_REG} and \texttt{SOFT\_TRIM\_REG} reduce to full-data regularized DPO and retain a clean-data offset from the unregularized reference $\pi^\star$.
Figure~\ref{fig:app-noattack-candkl} confirms that both unregularized defenses retain the clean-DPO solution when $K=0$, with final $\mathrm{candKL}=0$ up to numerical tolerance. The REG controls isolate the deviation introduced by the quadratic regularizer.

\begin{figure}[ht]
\centering
\includegraphics[width=0.49\textwidth]{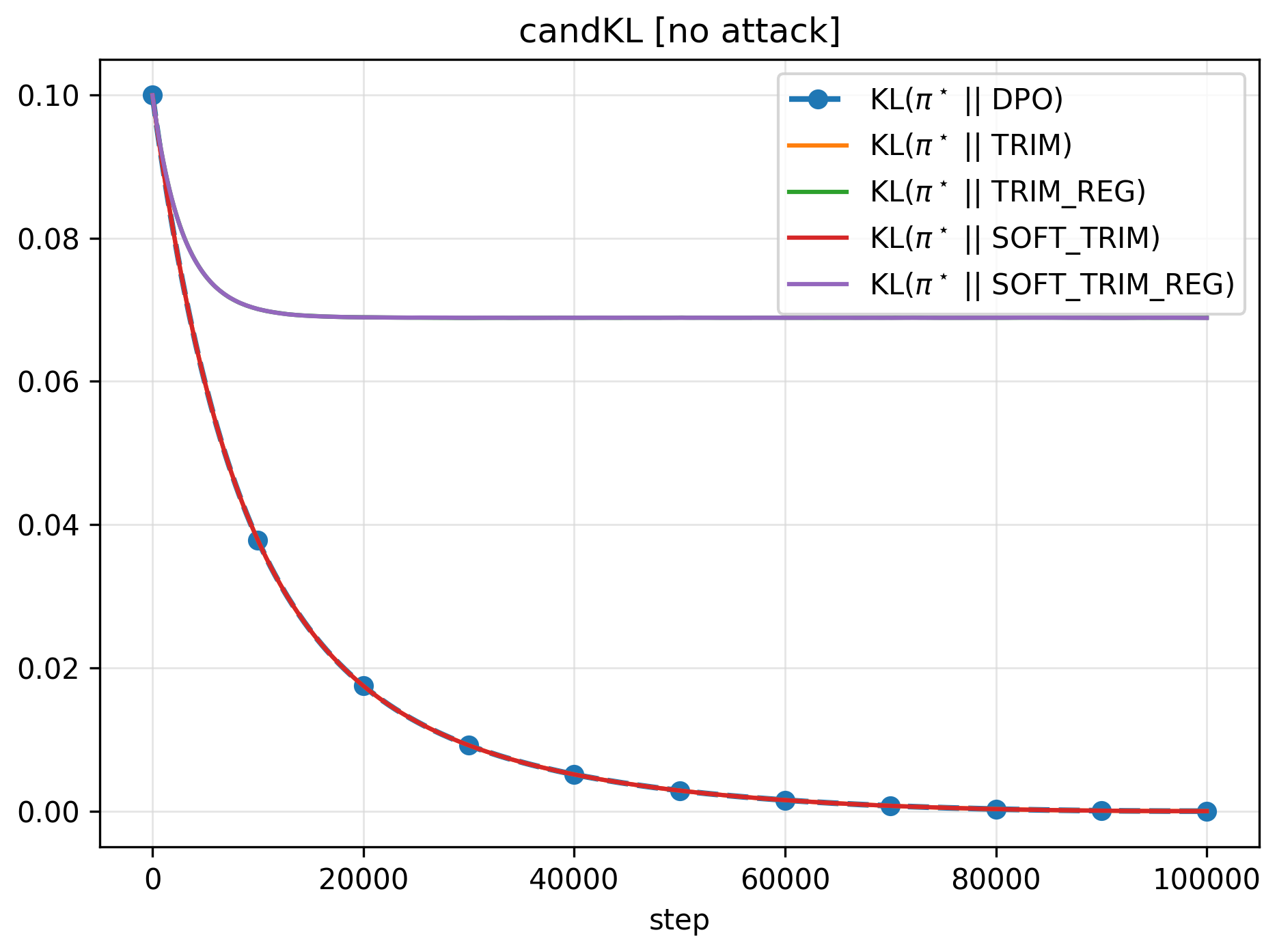}
\includegraphics[width=0.49\textwidth]{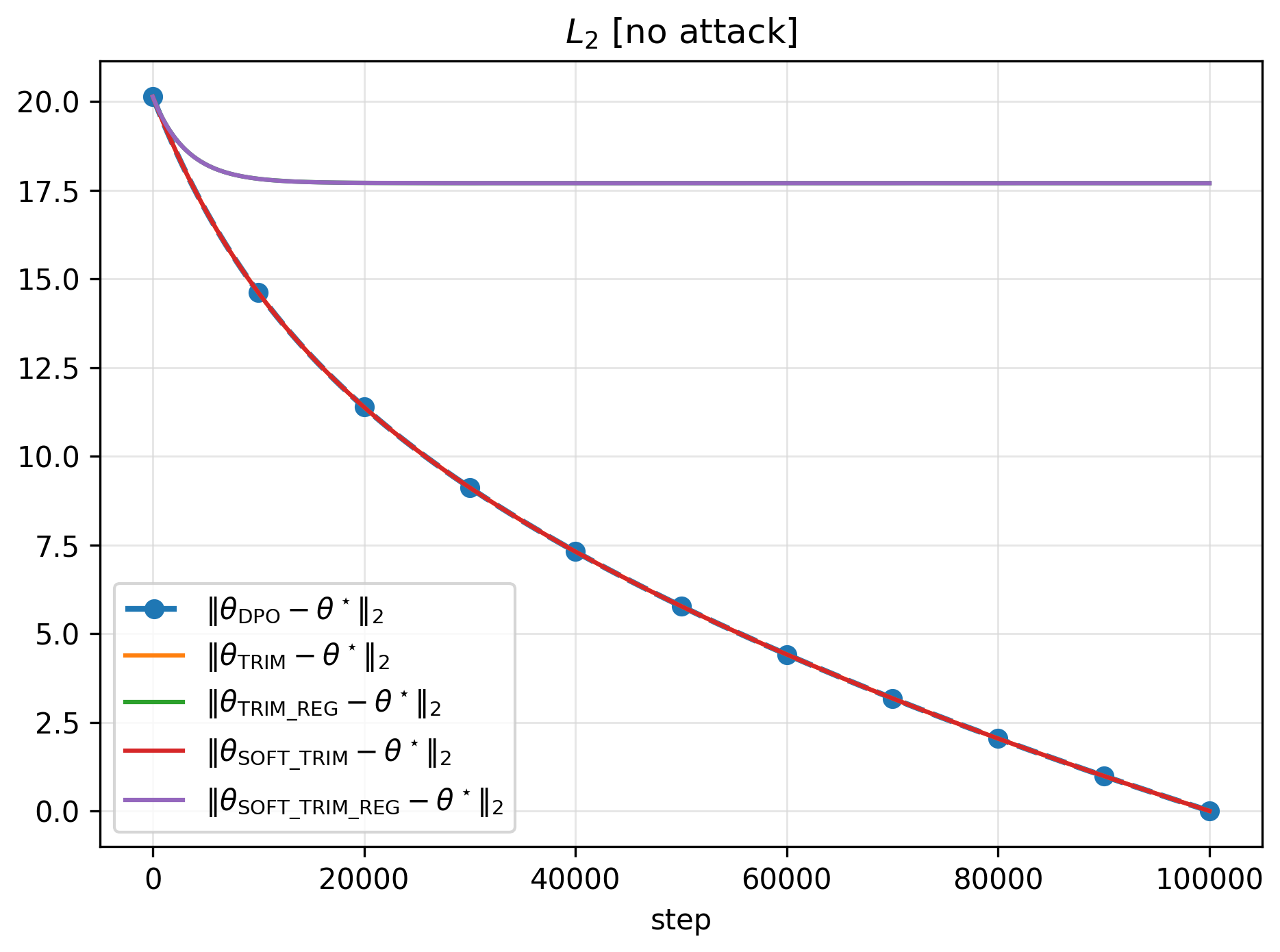}
\caption{No-attack scenario. $\mathrm{candKL}$, plotted as $\mathrm{KL}(\pi^\star\|\cdot)$ (left), and $\|\theta-\theta^\star\|_2$ (right). The unregularized proposed defenses use $K=0$, while the REG variants keep $\lambda=10^{-3}$.}
\label{fig:app-noattack-candkl}
\end{figure}

\subsection{No-attack scenario with conservative wrong $K$}

We also train on clean data while setting the trimming budget to $K=48{,}000$. Hard trimming removes $30\%$ of the clean training set, giving final candKL $0.018$ for \texttt{TRIM}. Soft trimming retains all pairs and reaches candKL $0.007$.

Figure~\ref{fig:app-noattack-wrongk-candkl} illustrates the clean-data cost of overestimating $K$. \texttt{SOFT\_TRIM} incurs a smaller deviation than \texttt{TRIM} in this setting, consistent with its less aggressive downweighting rule.

\begin{figure}[ht]
\centering
\includegraphics[width=0.49\textwidth]{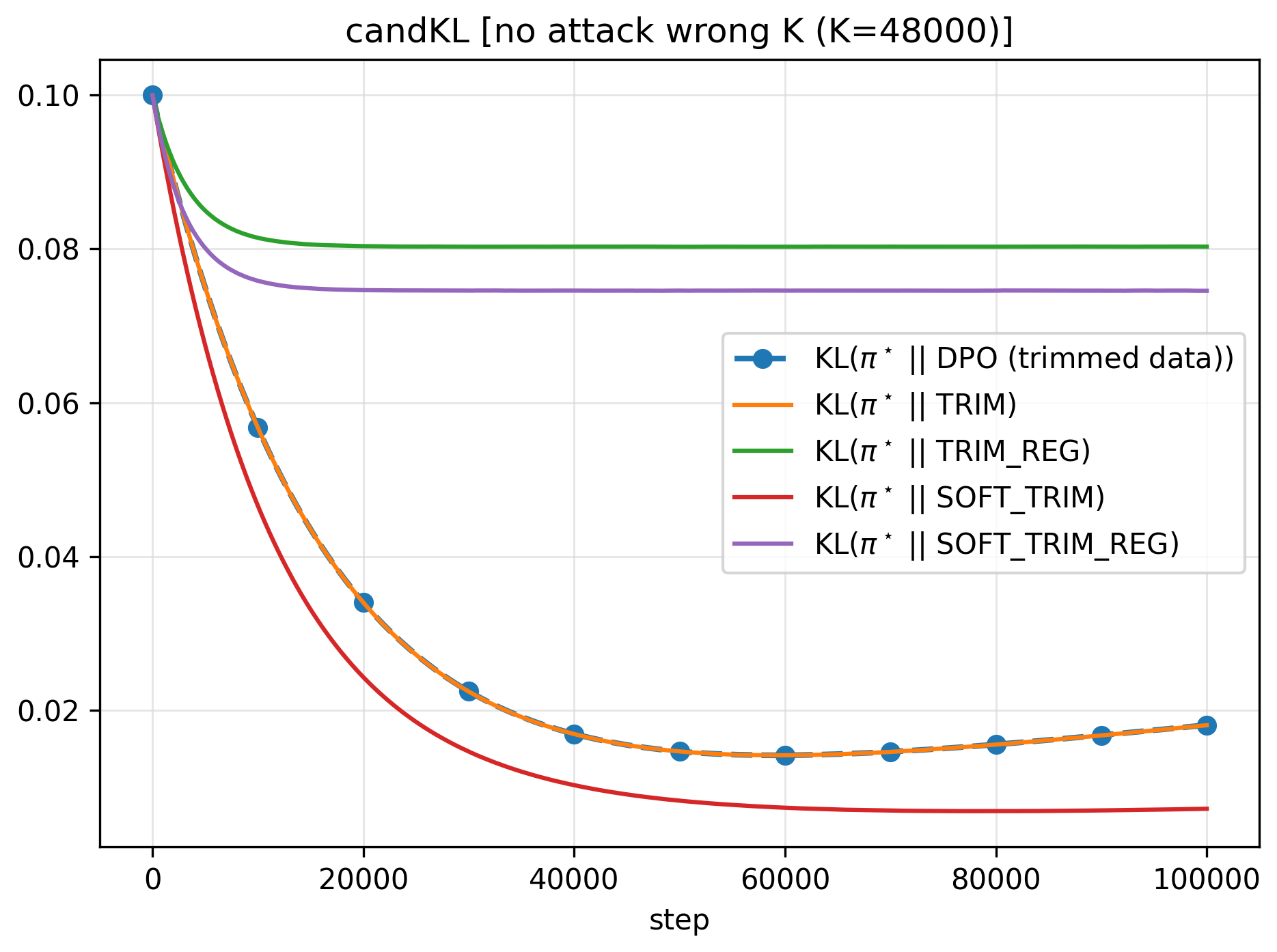}
\includegraphics[width=0.49\textwidth]{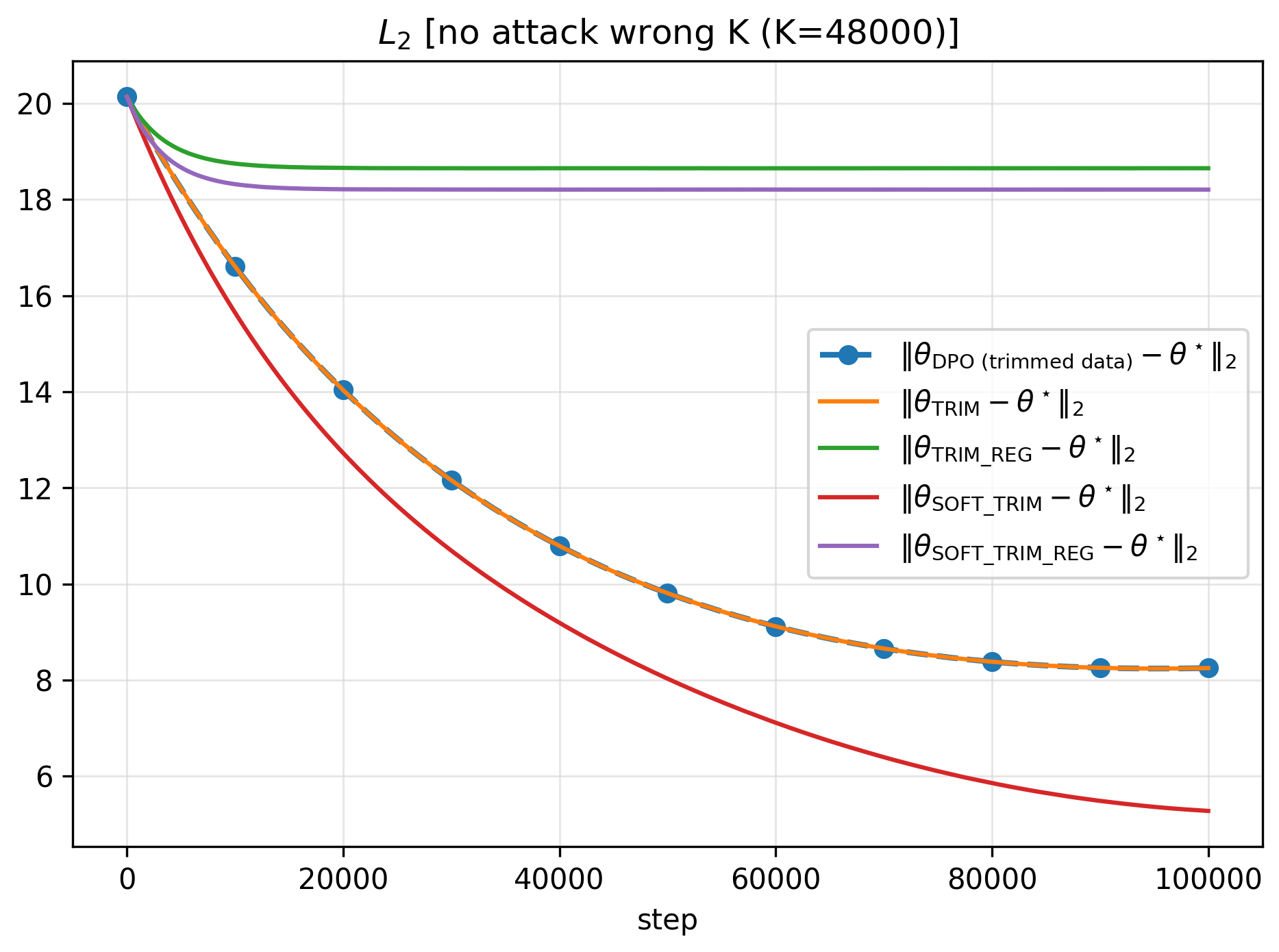}
\caption{No-attack scenario with conservative wrong $K$. $\mathrm{candKL}$, plotted as $\mathrm{KL}(\pi^\star\|\cdot)$ (left), and $\|\theta-\theta^\star\|_2$ (right). The true attack budget is zero, but the defender uses $K=48{,}000$, so hard trimming discards clean high-impact pairs and soft trimming keeps a smaller clean-data bias.}
\label{fig:app-noattack-wrongk-candkl}
\end{figure}

\subsection{Random-corruption scenario}

Under $30\%$ random label flips, only $\approx30\%$ of the top-$K$ impact samples are actually corrupted, so \texttt{TRIM}'s precision and recall for identifying attacked samples are both $\approx0.30$.
Across ten random seeds, standard \texttt{DPO} reaches $\mathrm{candKL}=0.041\pm0.001$, compared with $0.039\pm0.003$ for \texttt{TRIM} and $0.037\pm0.002$ for \texttt{SOFT\_TRIM}. The regularized variants have higher final $\mathrm{candKL}$, with $0.0918\pm0.0004$ for \texttt{TRIM\_REG} and $0.0892\pm0.0003$ for \texttt{SOFT\_TRIM\_REG}.
The $\pm$ values are empirical standard deviations across random attack seeds.
Figure~\ref{fig:app-random-candkl} shows modest reductions in mean candKL for the unregularized defenses, even though the impact ranking does not preferentially select randomly flipped pairs. Regularization does not improve this comparison, consistent with the clean-data offset observed in the no-attack control.

\begin{figure}[ht]
\centering
\includegraphics[width=0.49\textwidth]{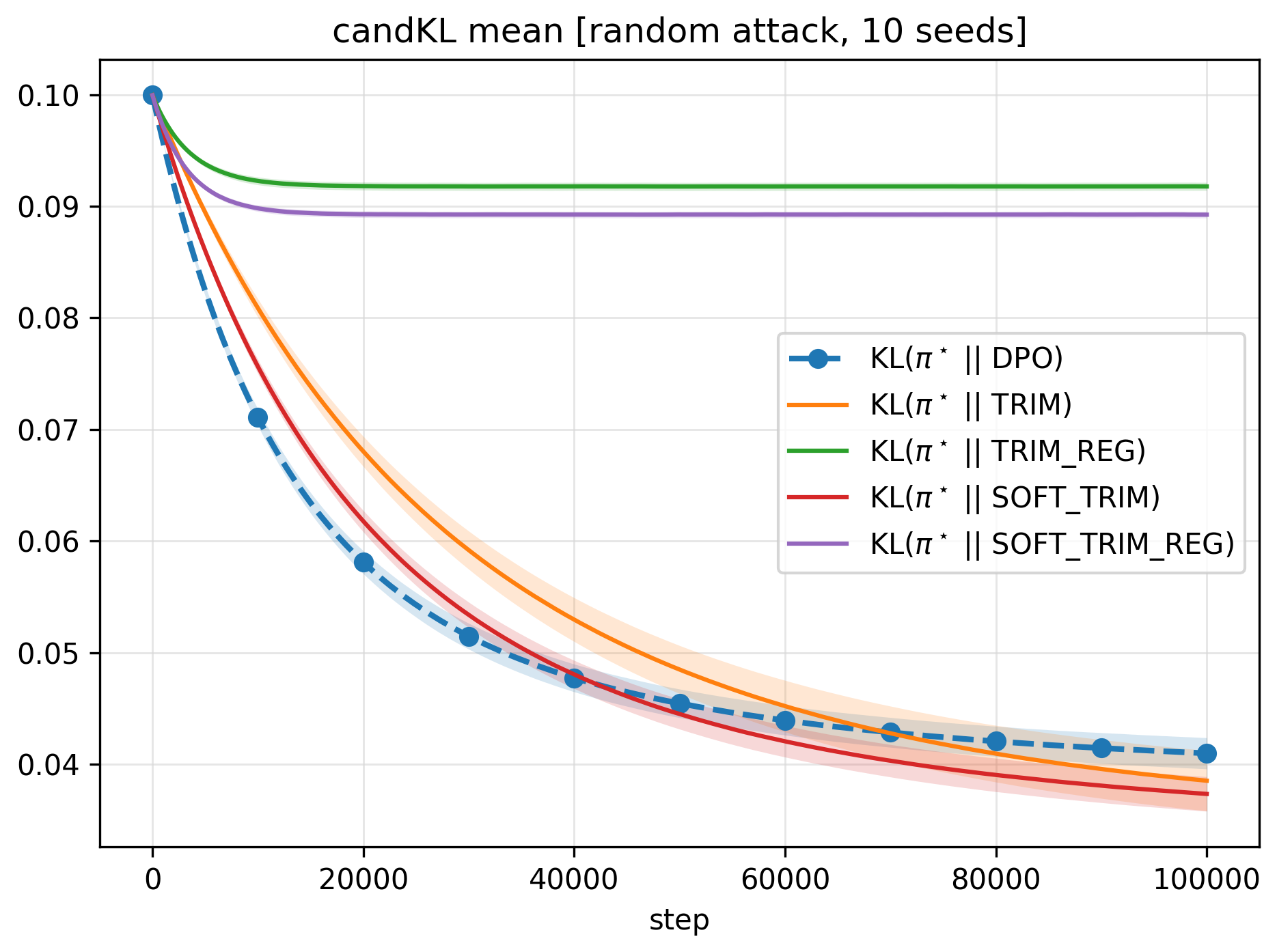}
\includegraphics[width=0.49\textwidth]{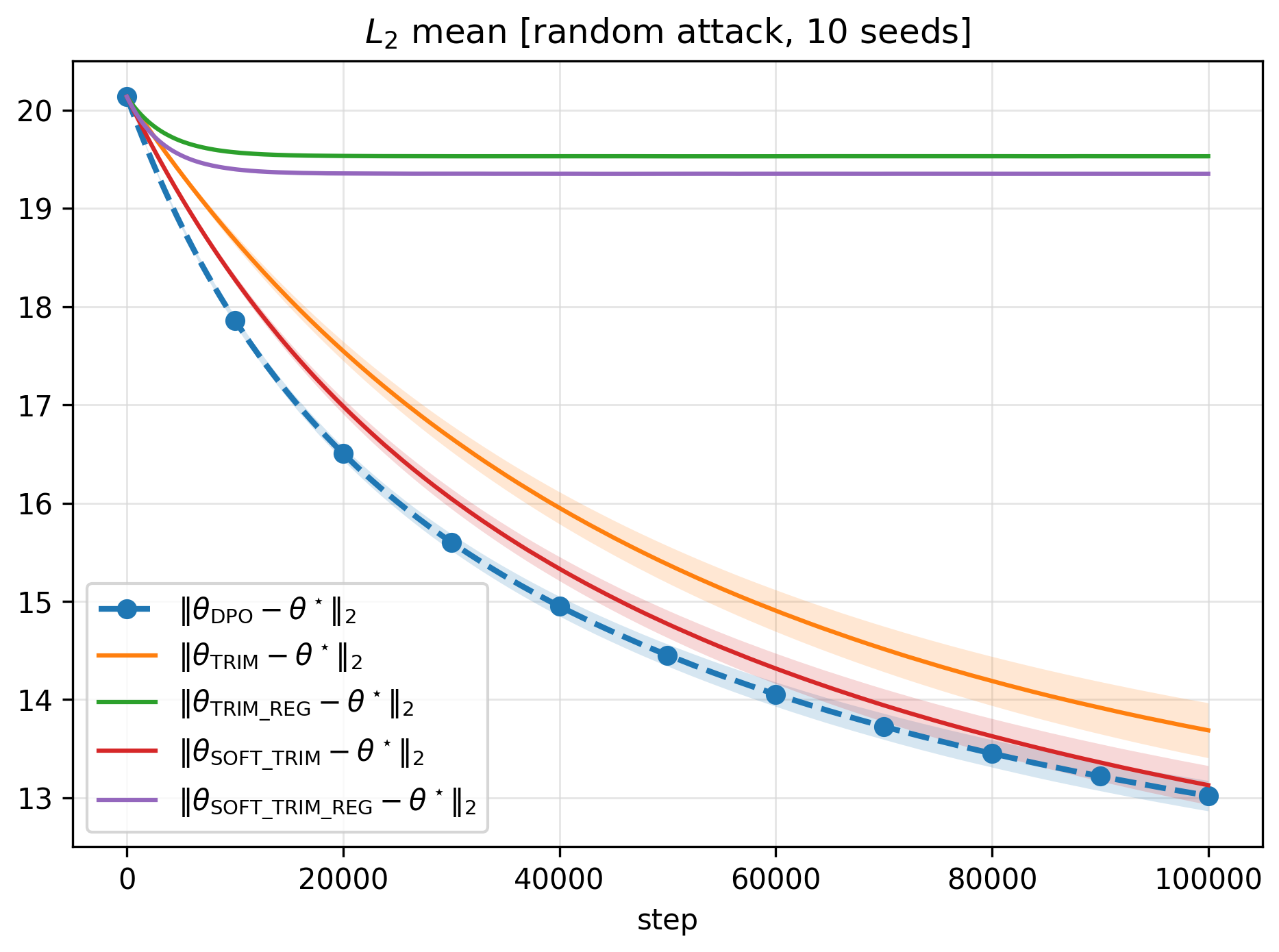}
\caption{Random-corruption scenario ($30\%$ of pairs flipped uniformly at random), averaged over ten seeds. Left: candKL from clean DPO. Right: parameter distance from $\theta^\star$.}
\label{fig:app-random-candkl}
\end{figure}

\subsection{Parameter-distance curves for top-impact and targeted attacks}

For completeness, Figures~\ref{fig:app-topimpact-theta} and~\ref{fig:app-targeted-theta} show the $\|\theta-\theta^\star\|_2$ trajectories for the top-impact and targeted attacks.
Under the top-impact attack, \texttt{TRIM} drives $\theta$-distance to $8.25$ (identical to \texttt{TRIM} on clean data with the same $K$), while attacked \texttt{DPO} reaches $25.2$.
Under the targeted BMP-A attack, \texttt{TRIM\_REG} ($\|\theta-\theta^\star\|_2=21.2$) and \texttt{SOFT\_TRIM\_REG} ($21.5$) are substantially less deflected than unregularized \texttt{TRIM} ($32.3$) and \texttt{DPO} ($30.7$), consistent with the $\mathrm{candKL}$ results.

The parameter-distance comparisons support the same conclusions as candKL: hard trimming is effective against the top-impact attack, while the REG variants reduce deviation from clean DPO under BMP-A.

\begin{figure}[ht]
\centering
\includegraphics[width=0.49\textwidth]{figures/attack_rdpo_shp_figure/topimpact_attack/DPO-TRIM_DPO-TRIM_DPO_REG-SOFT_TRIM_DPO-SOFT_TRIM_DPO_REG/plot_candKL_cases.png}
\includegraphics[width=0.49\textwidth]{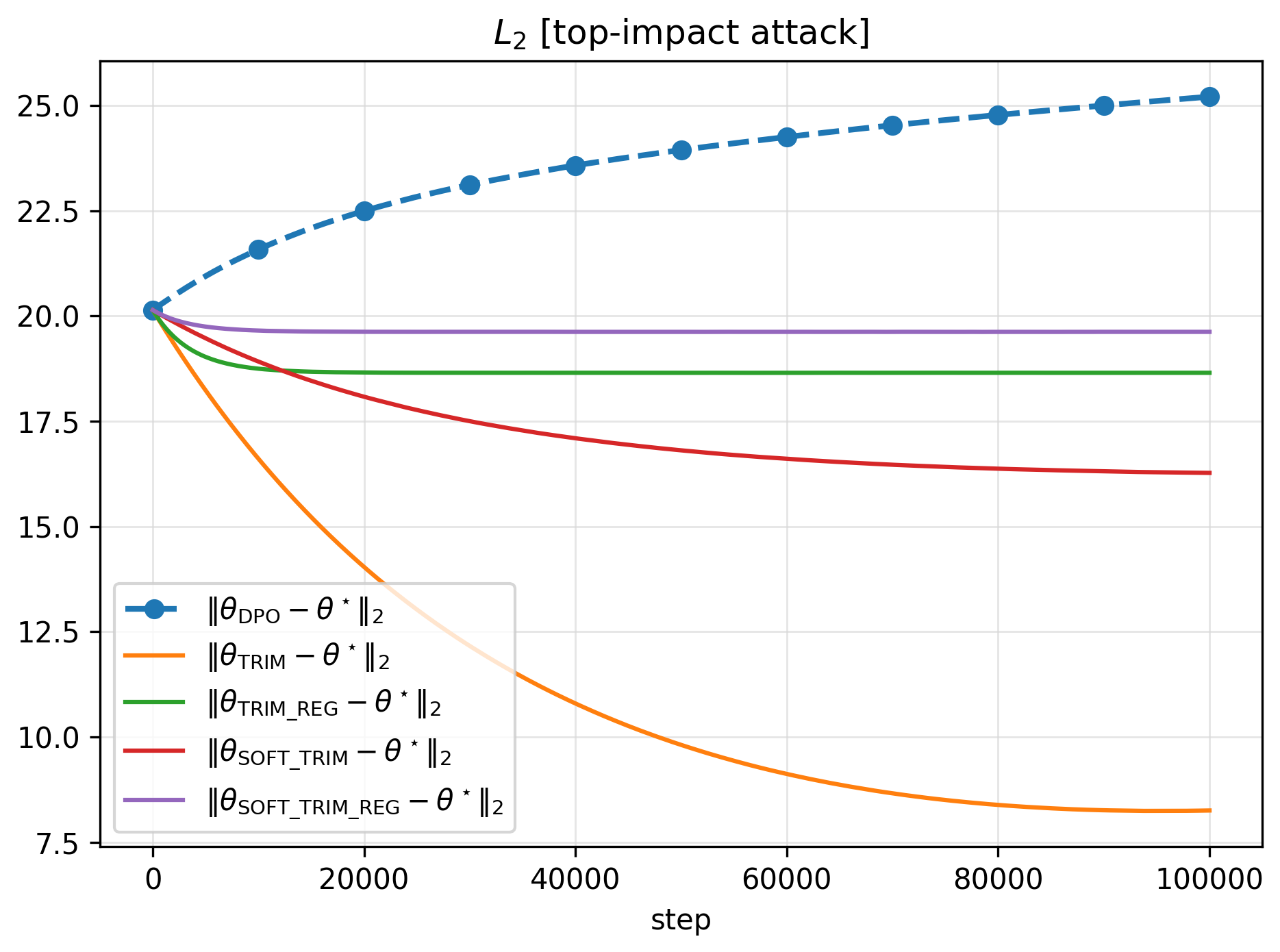}
\caption{Top-impact attack. $\mathrm{candKL}$, plotted as $\mathrm{KL}(\pi^\star\|\cdot)$ (left), and $\|\theta-\theta^\star\|_2$ (right). \texttt{TRIM} matches its clean-data counterpart with the same $K$.}
\label{fig:app-topimpact-theta}
\end{figure}

\begin{figure}[ht]
\centering
\includegraphics[width=0.49\textwidth]{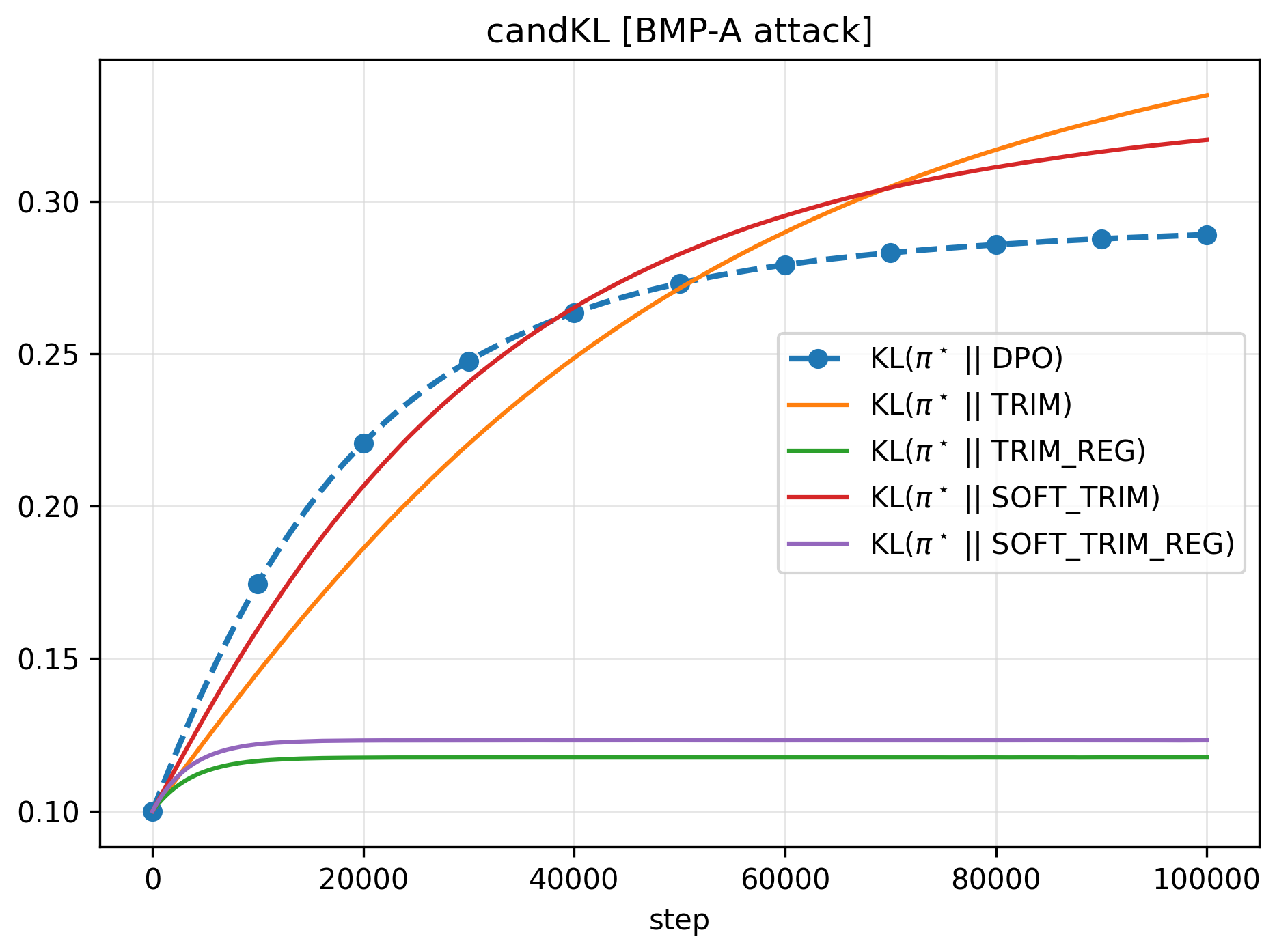}
\includegraphics[width=0.49\textwidth]{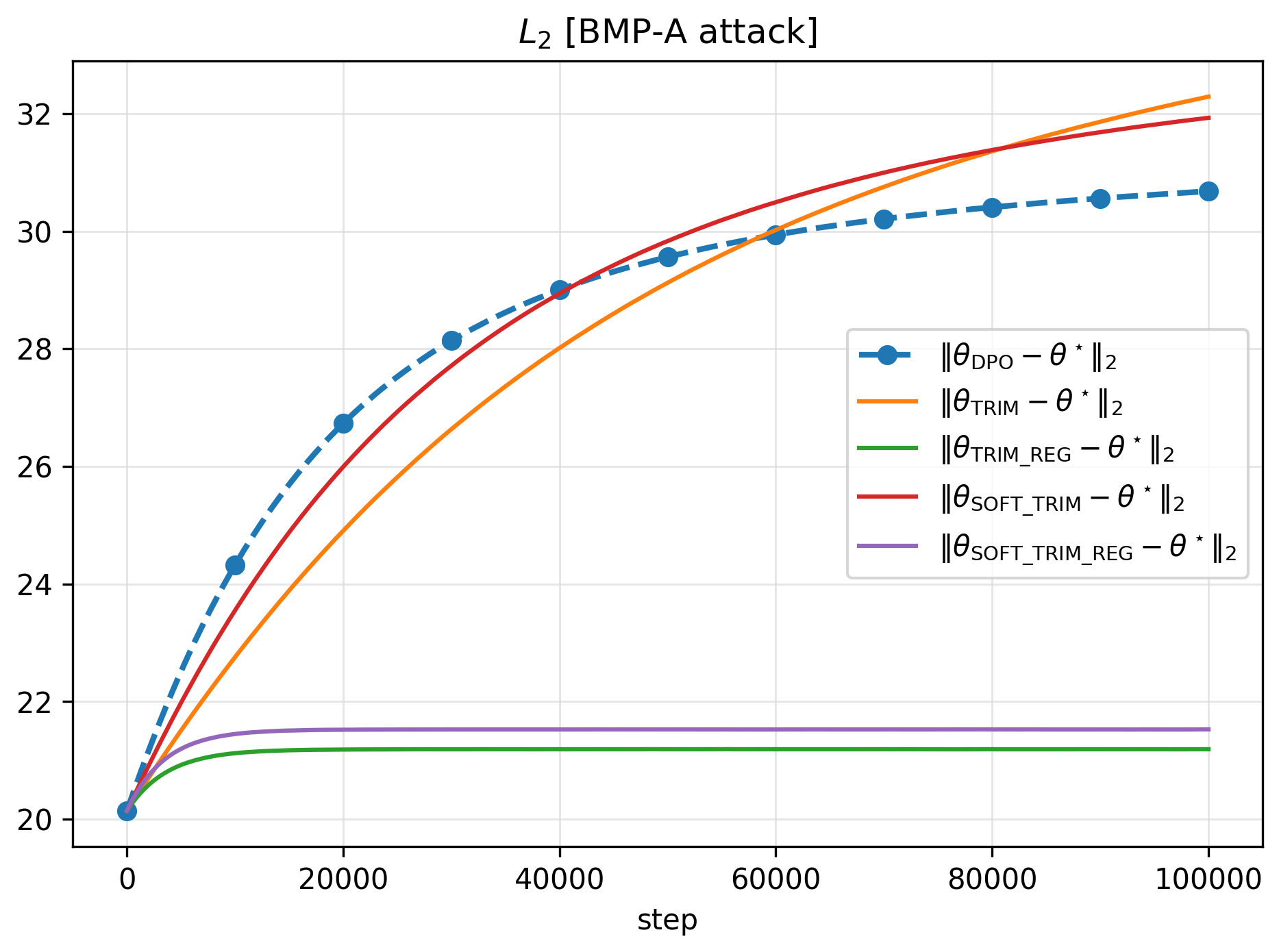}
\caption{Targeted BMP-A attack. $\mathrm{candKL}$, plotted as $\mathrm{KL}(\pi^\star\|\cdot)$ (left), and $\|\theta-\theta^\star\|_2$ (right). Both REG variants reduce deviation from clean DPO.}
\label{fig:app-targeted-theta}
\end{figure}

\section{Experiments: Comparison with \texttt{DR\_DPO}, \texttt{NR\_DPO}, \texttt{WDPO}, and \texttt{KLDPO} Baselines}\label{app:baselines}

We compare our proposed defenses against four existing robust DPO baselines designed for random noise or distributional shift.

\subsection{Baseline methods}\label{app:baselines-methods}

\paragraph{\texttt{DR\_DPO}~\citep{wu2024towards}.}
\texttt{DR\_DPO} applies an entropic aggregation to the DPO log-likelihood over preference pairs. In our positive-loss notation, the implemented objective is
\begin{equation*}
\mathcal{L}_{\mathrm{DR}}(\pi_\theta) \;=\; -\beta'\log\mathbb{E}\!\left[\exp\!\left(-\tilde\ell_i(\theta)/\beta'\right)\right],
\end{equation*}
where $\beta'>0$ is a tunable temperature. We use $\beta'=1.0$.
The original \texttt{DR\_DPO} objective applies the entropic transform to $h_i=-\tilde\ell_i$, the DPO log-likelihood, following Appendix D of~\citet{wu2024towards}.
In positive-loss notation, this gives a soft-minimum rather than the KL-DRO worst-case risk $\beta'\log\mathbb{E}[\exp(\tilde\ell_i/\beta')]$.
The induced per-sample gradient weight is $\exp(-\tilde\ell_i/\beta')/Z$, which downweights pairs with high observed loss.
This weighting can reduce the influence of high-loss noisy pairs, but high loss can also reflect a clean, difficult comparison. 
Under targeted corruption, fitting attacked labels can make opposing clean comparisons harder, so this weighting need not suppress the poisoned signal.

\paragraph{\texttt{NR\_DPO}~\citep{chowdhury2024provably}.}
\texttt{NR\_DPO} (called rDPO in the original paper) applies the symmetric-flip unbiased correction:
\begin{equation*}
\ell_{\mathrm{NR}}(\theta;z) \;=\; \frac{(1-\epsilon)\,\ell(z) - \epsilon\,\ell(-z)}{1-2\epsilon},
\end{equation*}
where $z$ is the chosen-rejected logit and $\epsilon\in[0,1/2)$ is the assumed flip rate.
We set $\epsilon$ equal to the true attack fraction $=0.3$ in attack scenarios, and to $0$ in the no-attack scenario, which recovers standard DPO.
This correction is unbiased for the clean DPO loss under i.i.d.\ symmetric label flips, and would work well when corruption matches that model.

\paragraph{\texttt{WDPO}~\citep{xu2025robust}.}
WDPO targets distributional shift in preference data, e.g.\ across demographic or temporal subpopulations, 
via a Wasserstein-DRO formulation $\min_\theta\sup_{P:\,W_p(P,P_o)\le\rho}\mathbb{E}_{\xi\sim P}[\ell(\xi;\theta)]$.
Following Algorithm~1 of~\citet{xu2025robust}, we tractably approximate this by adding an input-gradient penalty to the empirical DPO loss.
In the SHP simulator, the observed pair has $\tilde v_i=\psi(x^i,\tilde y_w^i)-\psi(x^i,\tilde y_l^i)$ and logit $\tilde z_i(\theta)=\beta\langle\theta-\theta_{\mathrm{ref}},\tilde v_i\rangle$.
The implementation adds $2\rho\cdot\mathbb{E}[\beta^2\sigma(-\tilde z_i(\theta))^2\|\theta-\theta_{\mathrm{ref}}\|_2^2]$ to the empirical loss. The factor $2$ accounts for the sum of squared gradient norms with respect to the chosen and rejected feature vectors, since their difference is $\tilde v_i$.
This is the log-linear analogue of the embedding-gradient penalty in WDPO. We set $\rho=10^{-2}$ for our simulator.

\paragraph{\texttt{KLDPO}~\citep{xu2025robust}.}
KLDPO targets the same distributional-shift setting as WDPO but with a KL-divergence uncertainty set, leading to a worst-case kernel that upweights samples with above-average loss.
Our minibatch implementation uses $w_i=\min\{10,\exp((\tilde\ell_i-\overline{\tilde\ell})/\tau_{\mathrm{kl}})\}$ with temperature $\tau_{\mathrm{kl}}=0.05$, where $\overline{\tilde\ell}$ is the current minibatch mean loss. 
These clipped weights are held fixed when differentiating the minibatch mean of $w_i\tilde\ell_i$. This is the approximate reweighting update used in our comparison with Algorithm~2 of~\citet{xu2025robust}.

\subsection{Results and discussion}\label{app:baselines-results}

We evaluate all four robust-DPO baselines together with our four proposed defenses (\texttt{TRIM}, \texttt{TRIM\_REG}, \texttt{SOFT\_TRIM}, \texttt{SOFT\_TRIM\_REG}) on the no-attack, random-corruption, top-impact, and targeted-BMP-A scenarios introduced in Section~\ref{sec:experiments}.
Per-step $\mathrm{candKL}$ and parameter-distance trajectories for each scenario are shown in Figures~\ref{fig:app-baselines-noattack}--\ref{fig:app-baselines-targeted}.
The final-step $\mathrm{candKL}$ values are summarized in Table~\ref{tab:app-baselines-summary}.
For the random-corruption column, all methods are averaged over the same ten random attack seeds. The proposed-method means and standard deviations are reported separately in Appendix~\ref{app:noatk-random}.

\paragraph{No-attack scenario (Figure~\ref{fig:app-baselines-noattack}).}
When the training data are clean, the proposed \texttt{TRIM} and \texttt{SOFT\_TRIM} defenses are evaluated with $K=0$, so the unregularized variants coincide with standard DPO and have final $\mathrm{candKL}=0$.
The REG variants keep $\lambda=10^{-3}$ and therefore report the clean-data offset of full-data regularized DPO.
\texttt{NR\_DPO} (with $\epsilon=0$, recovering standard DPO) and \texttt{DR\_DPO} also converge to near-zero $\mathrm{candKL}$.
\texttt{WDPO} ($\mathrm{candKL}=0.096$) and \texttt{KLDPO} ($\mathrm{candKL}=0.089$) remain farther from the full-data clean-DPO reference. Their penalty and reweighting rules change the training updates even when no corruption is present.

Thus, the unregularized \texttt{TRIM} and \texttt{SOFT\_TRIM} methods preserve clean-DPO training when the defender sets $K=0$. The REG controls retain the bias introduced by their quadratic penalty.

\begin{figure}[ht]
\centering
\includegraphics[width=0.49\textwidth]{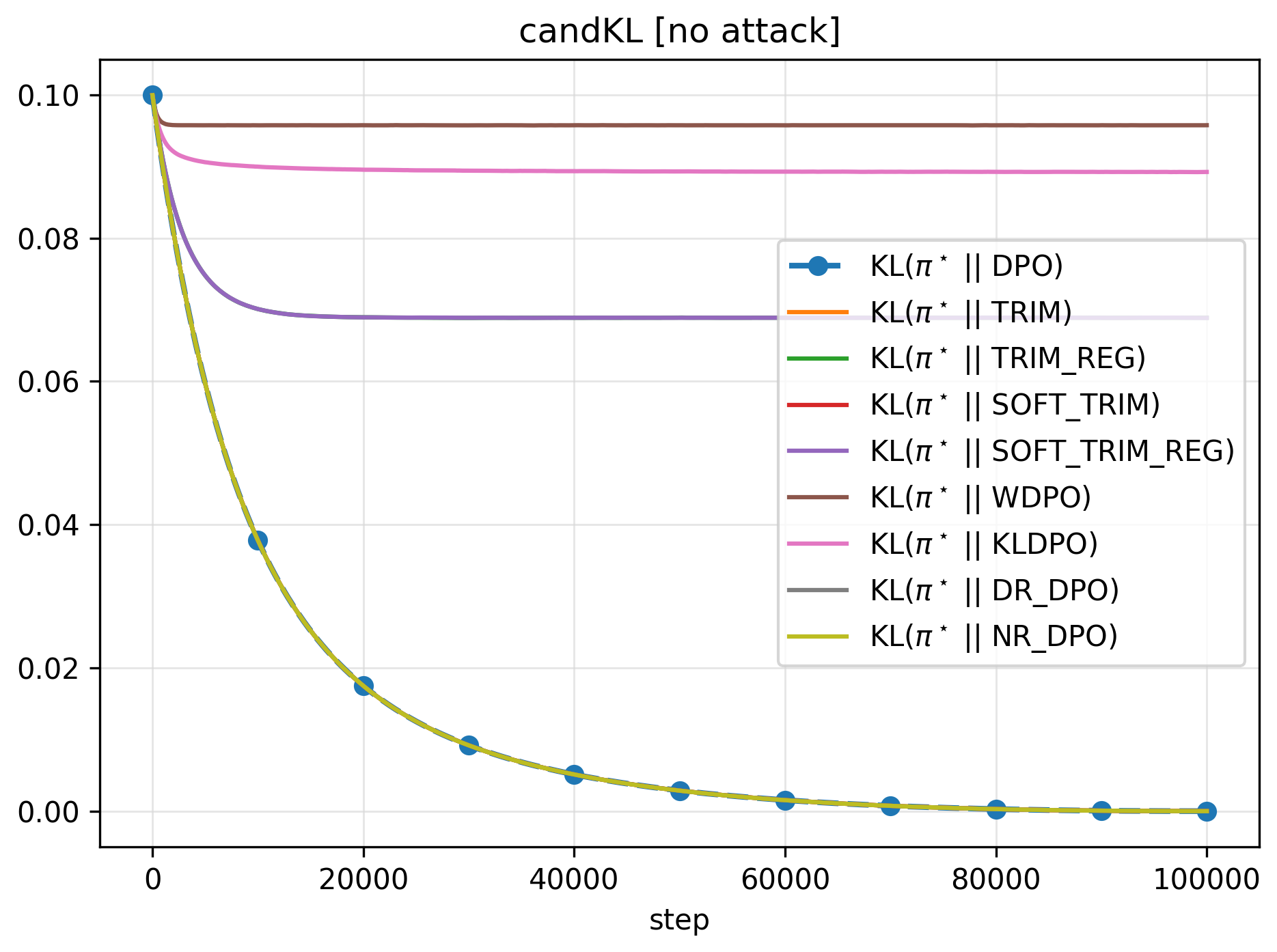}
\includegraphics[width=0.49\textwidth]{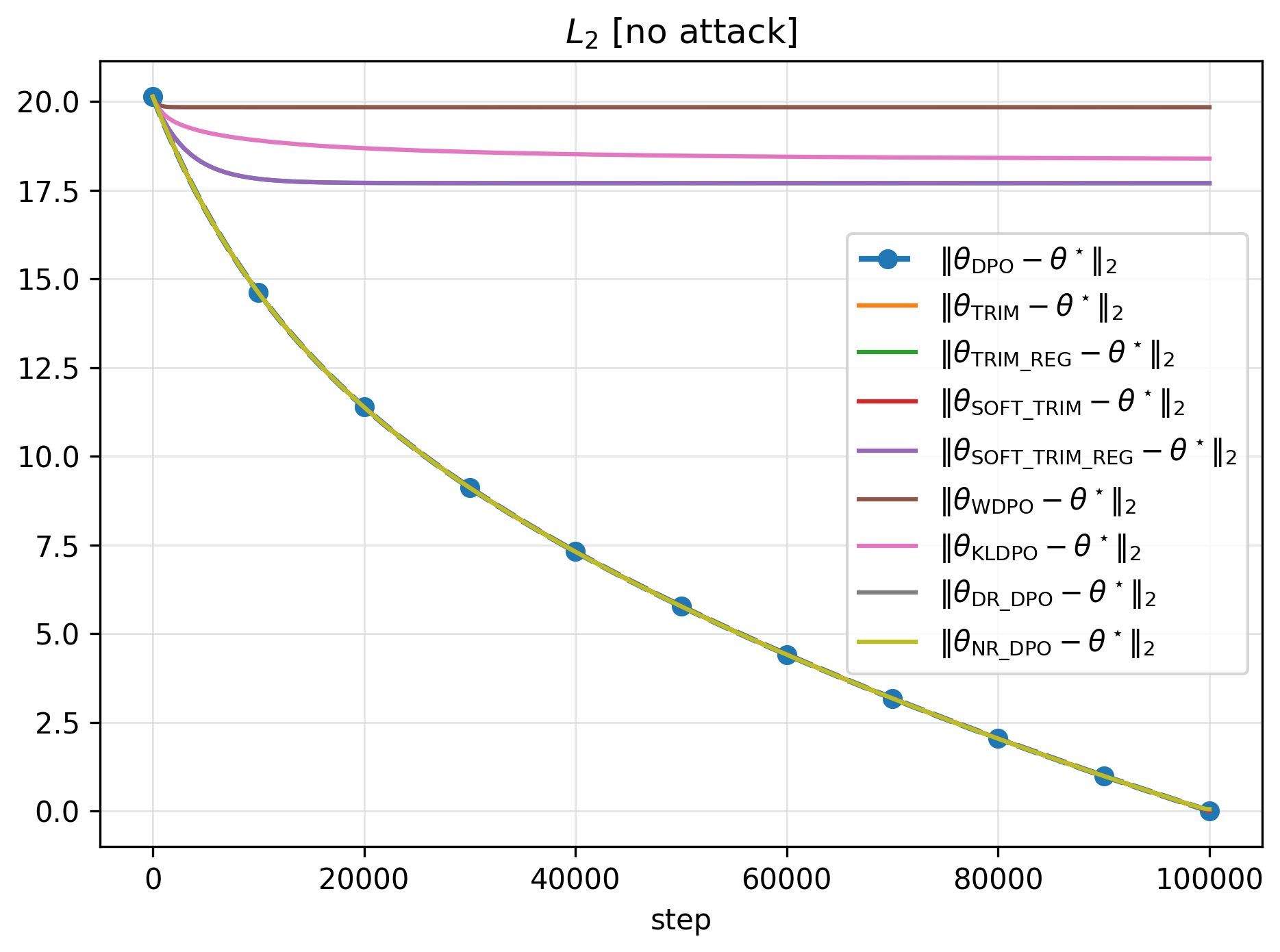}
\caption{No-attack scenario with all baselines. Left: candKL from clean DPO. Right: parameter distance from $\theta^\star$. The proposed defenses use $K=0$, and the REG variants keep $\lambda=10^{-3}$. \texttt{NR\_DPO} (with $\epsilon=0$) and \texttt{DR\_DPO} also converge to near-zero $\mathrm{candKL}$. \texttt{WDPO} and \texttt{KLDPO} converge to a moderately different solution ($\mathrm{candKL}\approx0.09$--$0.10$).}
\label{fig:app-baselines-noattack}
\end{figure}

\paragraph{Random corruption (Figure~\ref{fig:app-baselines-random}).}
Under $30\%$ uniformly random label flips, the corruption setting is closest to the noise model used by \texttt{NR\_DPO}.
\texttt{NR\_DPO} (noise rate $\epsilon=0.3$) reaches the lowest average $\mathrm{candKL}$ among the compared methods under random flips: $\mathrm{candKL}=0.010$, approximately one quarter of standard DPO's $0.041$. 
This is consistent with its random-noise correction.
\texttt{DR\_DPO} ($0.082$), \texttt{WDPO} ($0.098$), and \texttt{KLDPO} ($0.096$) have higher candKL than standard DPO.

The unregularized proposed methods give modest improvements over DPO: the ten-seed mean candKL values are $0.039$ for \texttt{TRIM} and $0.037$ for \texttt{SOFT\_TRIM}, compared with $0.041$ for DPO.

\begin{figure}[ht]
\centering
\includegraphics[width=0.49\textwidth]{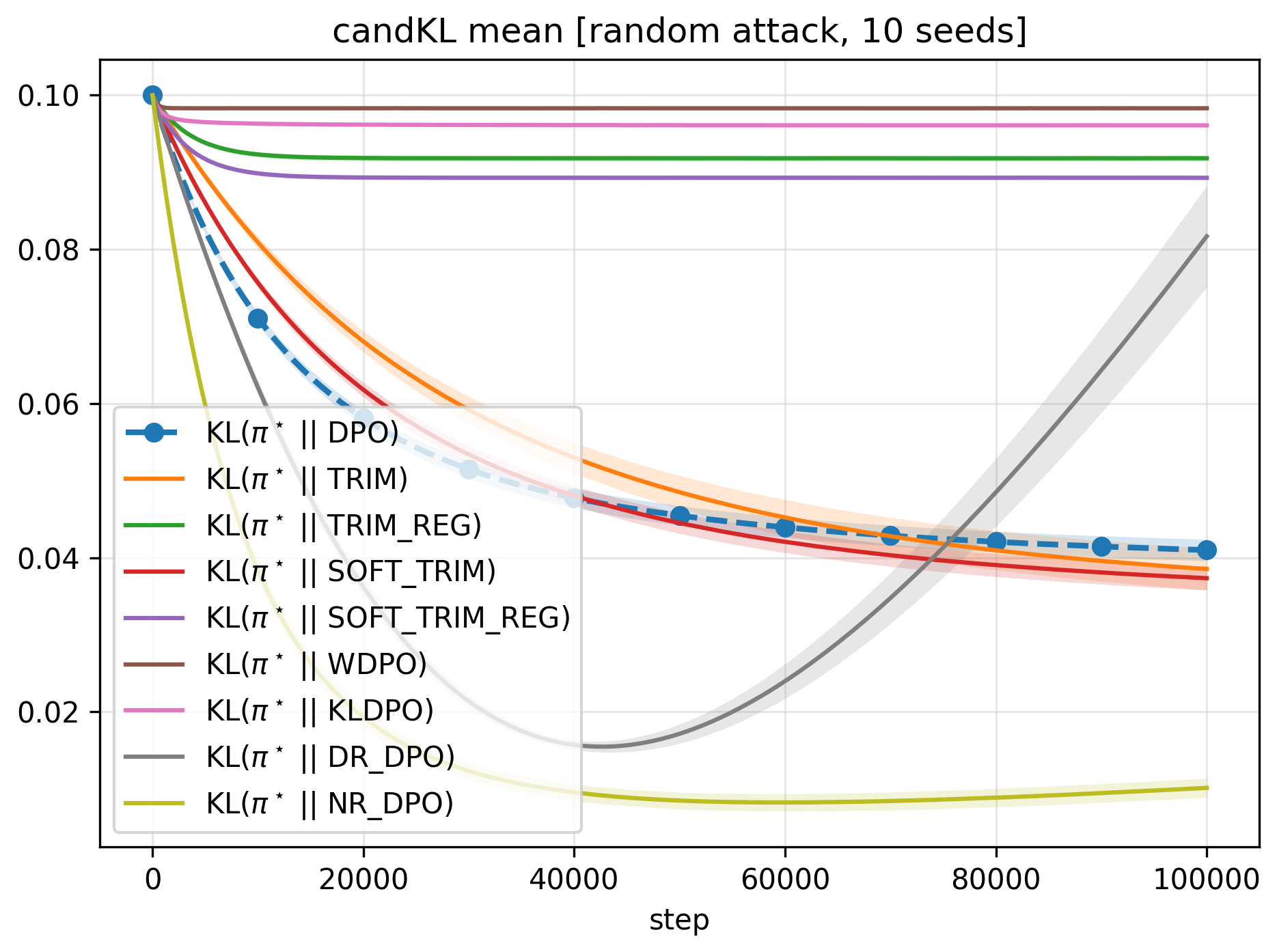}
\includegraphics[width=0.49\textwidth]{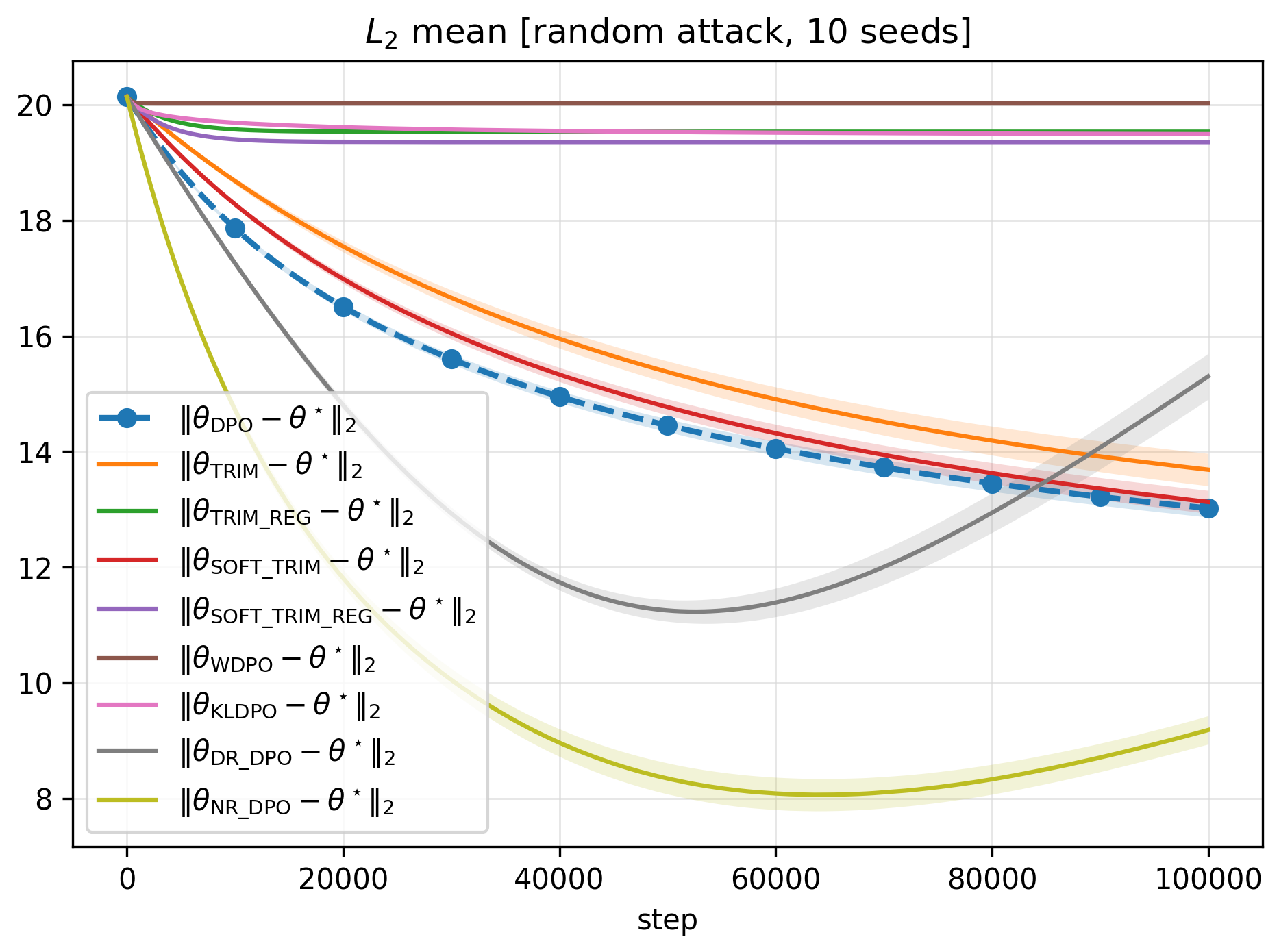}
\caption{Random-corruption scenario with all baselines, averaged over ten seeds. Left: candKL from clean DPO. Right: parameter distance from $\theta^\star$.}
\label{fig:app-baselines-random}
\end{figure}

\paragraph{Top-impact attack (Figure~\ref{fig:app-baselines-topimpact}).}
Under the top-impact attack, \texttt{DR\_DPO} ($\mathrm{candKL}=0.423$) and \texttt{NR\_DPO} ($0.304$) deviate more than attacked DPO ($0.155$).
\texttt{WDPO} and \texttt{KLDPO} have lower candKL than DPO ($0.101$ and $0.102$), but remain above \texttt{TRIM} ($0.018$). 
\texttt{SOFT\_TRIM} also has lower candKL ($0.064$) than these baselines.

Both unregularized proposed methods outperform the compared baselines in this setting. Hard trimming removes all flipped pairs and matches \texttt{TRIM} trained on clean data with the same $K$, leaving only the deviation due to discarding clean training information.

\begin{figure}[ht]
\centering
\includegraphics[width=0.49\textwidth]{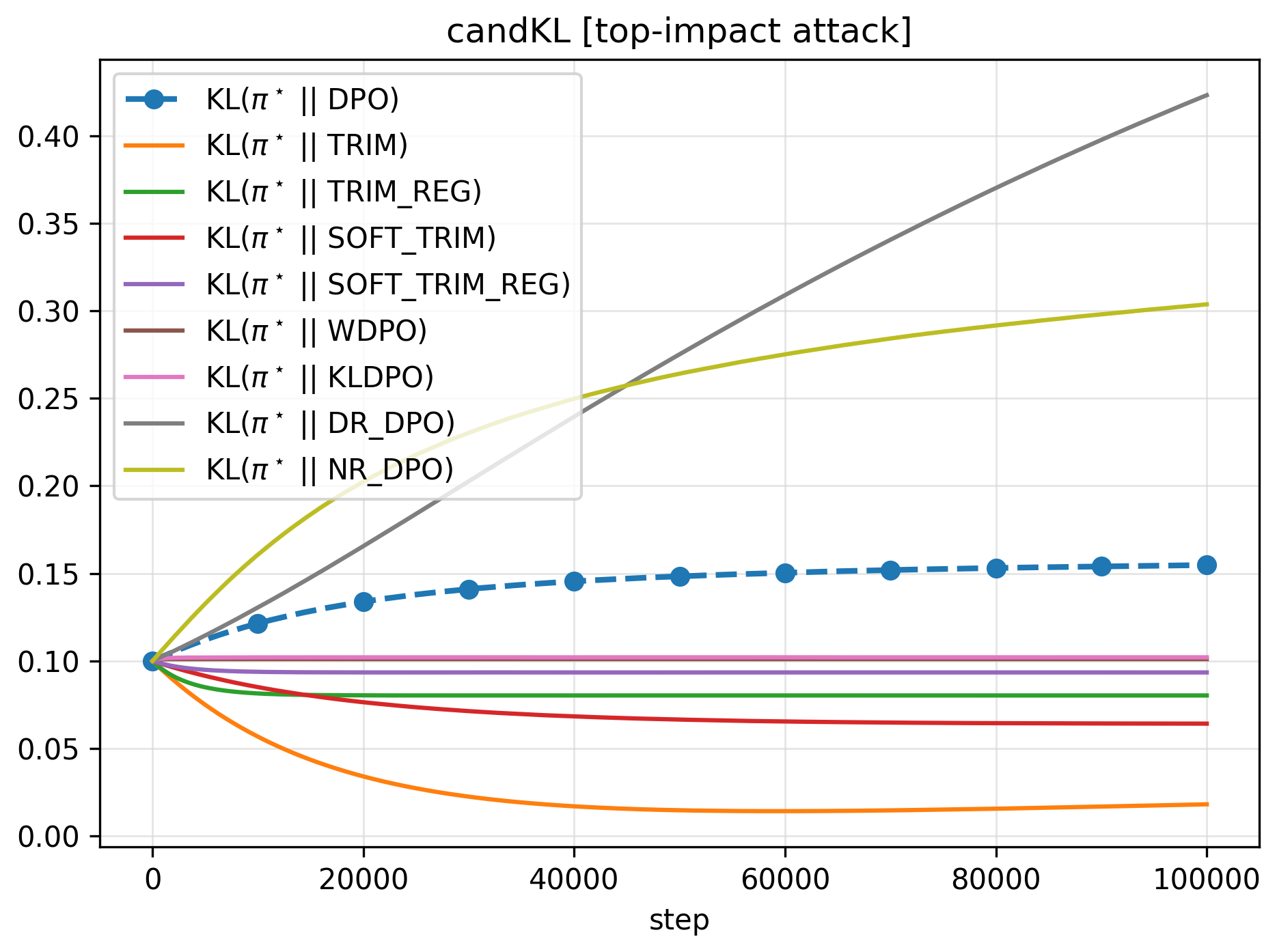}
\includegraphics[width=0.49\textwidth]{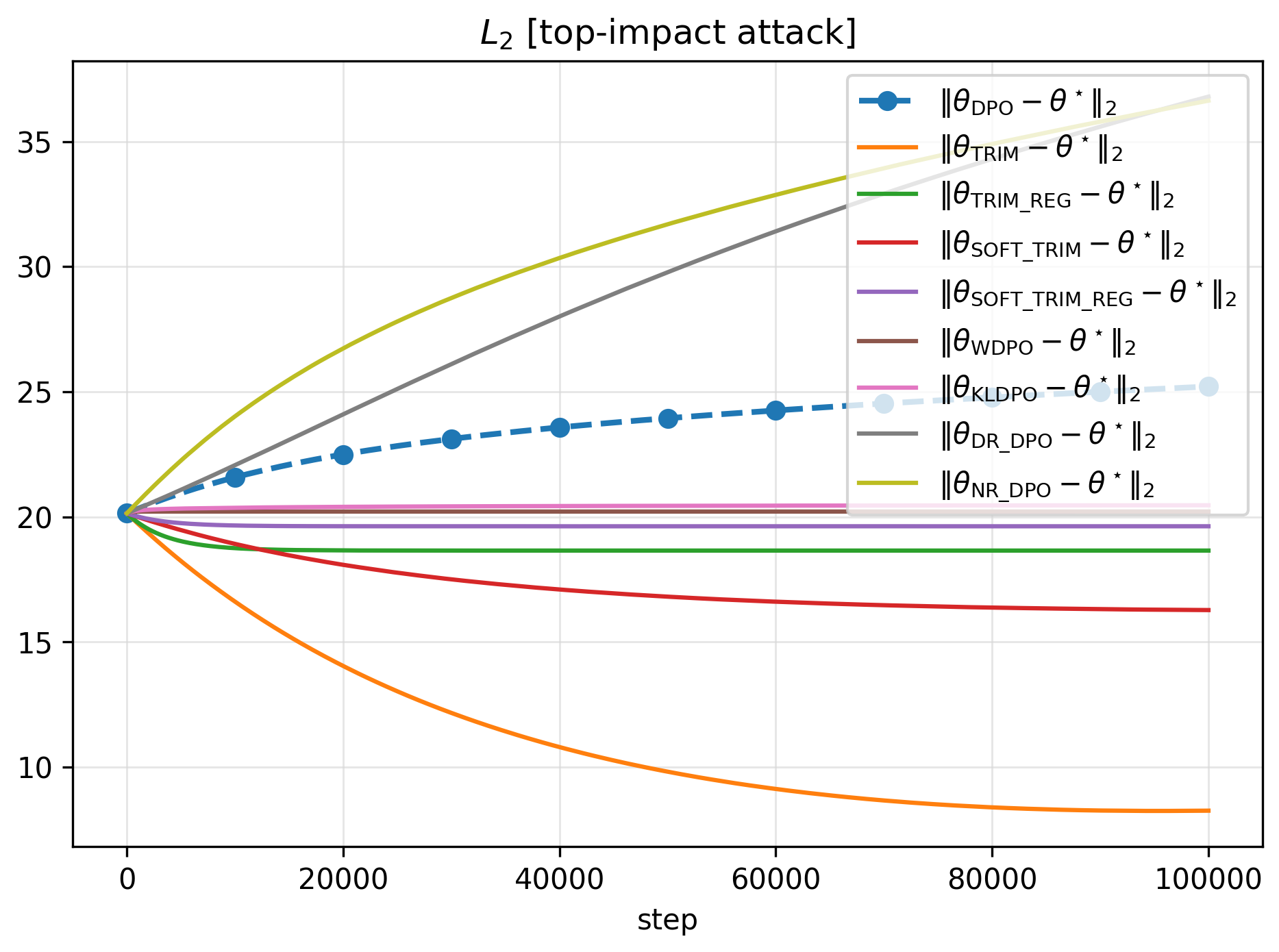}
\caption{Top-impact attack with all baselines. Left: candKL from clean DPO. Right: parameter distance from $\theta^\star$.}
\label{fig:app-baselines-topimpact}
\end{figure}

\paragraph{Targeted BMP-A attack (Figure~\ref{fig:app-baselines-targeted}).}
Under the targeted BMP-A attack, \texttt{NR\_DPO} ($\mathrm{candKL}=1.079$) and \texttt{DR\_DPO} ($1.066$) have larger deviations than standard DPO ($0.289$). Their parameter distances are $54$--$58$, compared with $30.7$ for DPO.
\texttt{WDPO} ($\mathrm{candKL}=0.103$) and \texttt{KLDPO} ($\mathrm{candKL}=0.107$) have smaller deviations than DPO in this fixed BMP-A instance.
\texttt{TRIM} ($\mathrm{candKL}=0.335$) and \texttt{SOFT\_TRIM} ($\mathrm{candKL}=0.320$) are worse than standard DPO in this instance. 
Their top-$K$ intervention region overlaps only $42\%$ of the corrupted support: hard trimming removes those pairs, while soft trimming downweights them. 
Both rules also suppress clean high-impact samples.
The regularized variants are closer to full-data clean DPO, with candKL $0.118$ for \texttt{TRIM\_REG} and $0.123$ for \texttt{SOFT\_TRIM\_REG}, both below matched full-data \texttt{DPO\_REG} ($0.130$).

Here regularization helps the proposed defenses reduce policy deviation under BMP-A, while unregularized trimming alone is insufficient. 
Both proposed REG variants improve over DPO, but WDPO and KLDPO achieve lower candKL in this instance. 
Appendix~\ref{app:matched-regularization} quantifies the additional benefit of trimming beyond the shared regularizer.

\begin{figure}[ht]
\centering
\includegraphics[width=0.49\textwidth]{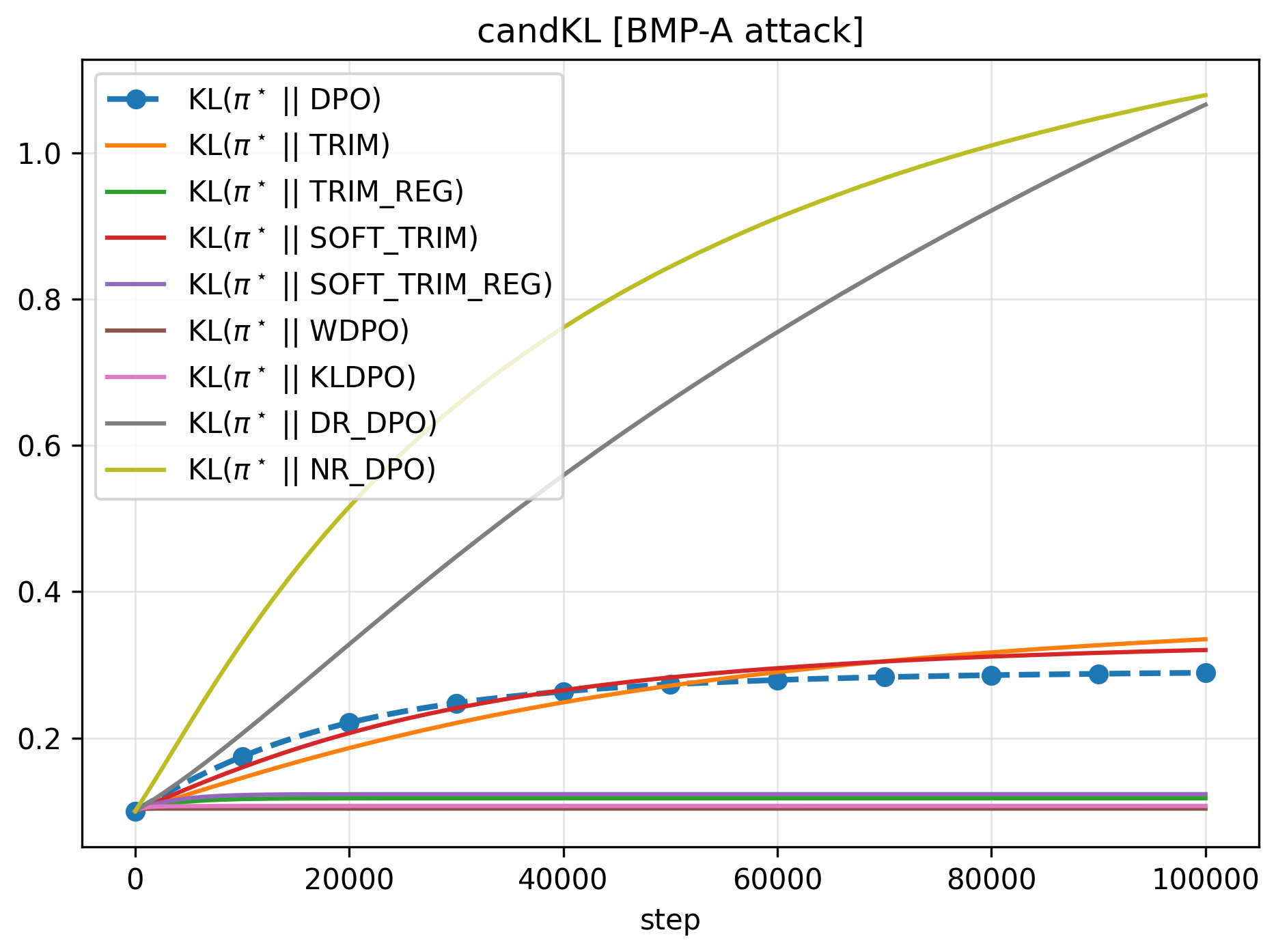}
\includegraphics[width=0.49\textwidth]{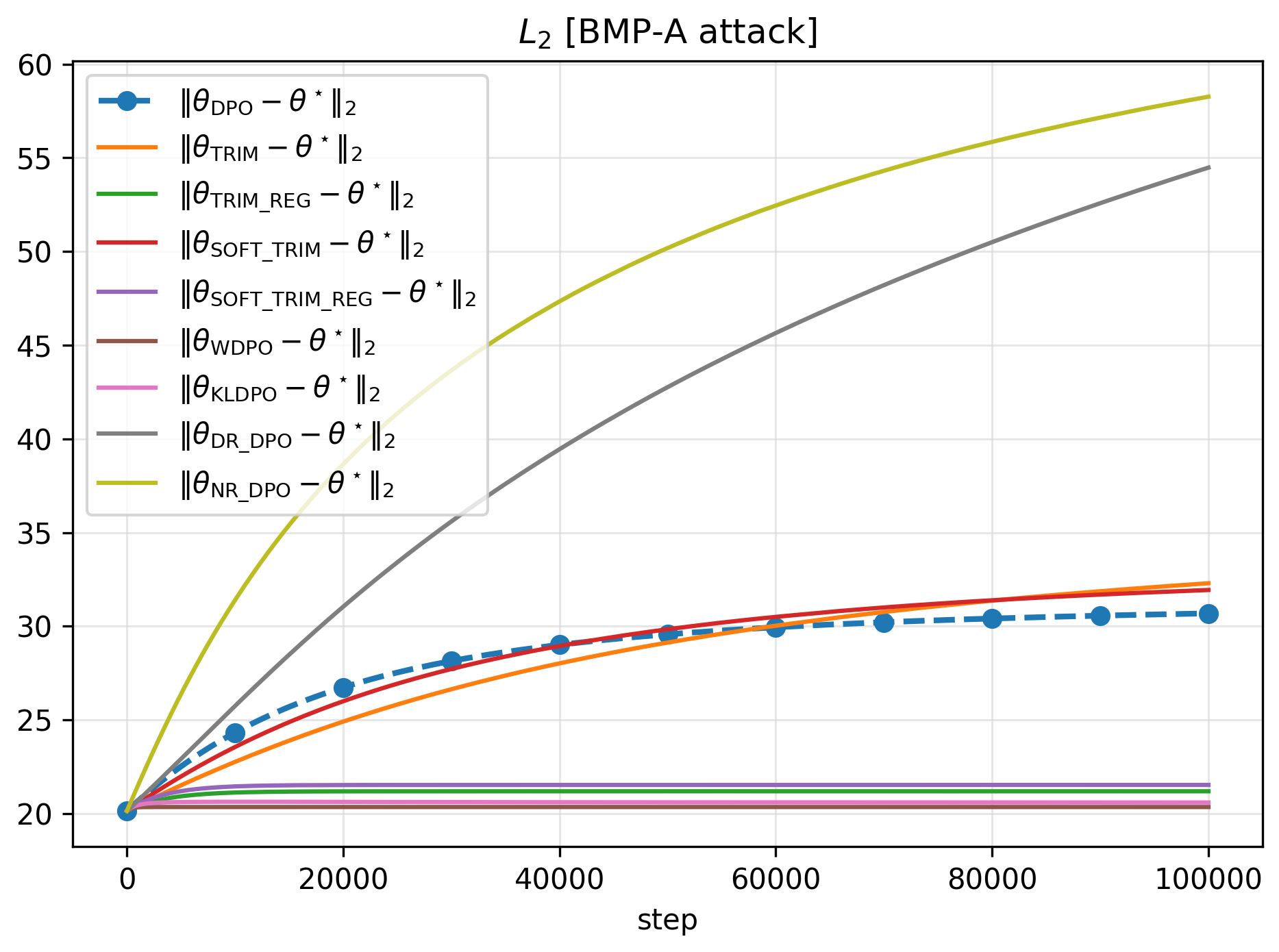}
\caption{Targeted BMP-A attack with all baselines. Left: candKL from clean DPO. Right: parameter distance from $\theta^\star$.}
\label{fig:app-baselines-targeted}
\end{figure}

Table~\ref{tab:app-baselines-summary} summarizes the final-step $\mathrm{candKL}$ across all four scenarios.

\begin{table}[ht]
\centering
\resizebox{\textwidth}{!}{%
\begin{tabular}{l c c c c}
\toprule
Method & No-attack & Random ($30\%$) & Top-impact & Targeted BMP-A \\
\midrule
\texttt{DPO} (no defense)              & $0$ (ref.)        & $0.041$           & $0.155$           & $0.289$ \\
\texttt{DPO\_REG}   & $0.069$ & $0.087$ & $0.109$ & $0.130$ \\
\midrule
\texttt{TRIM}                     & $0$               & $0.039$           & $0.018$  & $0.335$ \\
\texttt{TRIM\_REG}                & $0.069$           & $0.092$           & $0.080$           & $0.118$ \\
\texttt{SOFT\_TRIM}               & $0$               & $0.037$           & $0.064$           & $0.320$ \\
\texttt{SOFT\_TRIM\_REG}          & $0.069$           & $0.089$           & $0.093$           & $0.123$ \\
\midrule
\texttt{DR\_DPO}~\citep{wu2024towards}        & ${\approx}0$  & $0.082$           & $0.423$           & $1.066$ \\
\texttt{NR\_DPO}~\citep{chowdhury2024provably} & $0$           & $0.010$  & $0.304$           & $1.079$ \\
\texttt{WDPO}~\citep{xu2025robust}            & $0.096$       & $0.098$           & $0.101$           & $0.103$ \\
\texttt{KLDPO}~\citep{xu2025robust}           & $0.089$       & $0.096$           & $0.102$           & $0.107$ \\
\bottomrule
\end{tabular}}
\caption{Final-step $\mathrm{candKL}$ across all four baseline-comparison scenarios. The random-corruption column reports ten-seed means for all methods. Appendix~\ref{app:noatk-random} reports the corresponding mean and standard deviation for the proposed methods. In the no-attack control, the unregularized proposed defenses use $K=0$ and the regularized ones keep $\lambda=10^{-3}$. \texttt{NR\_DPO} uses $\epsilon=0$ without attacks and $\epsilon=0.3$ in attacked scenarios.}
\label{tab:app-baselines-summary}
\end{table}

For \texttt{NR\_DPO}, the strong performance under random flips does not carry over to targeted BMP-A.
The independent random-noise model required for \texttt{NR\_DPO}'s cancellation is violated by this fixed targeted support. 
\texttt{DR\_DPO} weights easy observed pairs more strongly. 
If those are attacked, the weighting can favor the poisoned signal. 
These mechanisms explain why their respective robustness assumptions need not hold here.

\texttt{WDPO} and \texttt{KLDPO} stay in a narrow range ($\mathrm{candKL}\in[0.09,0.11]$ in these runs).
This makes them relatively stable under the targeted BMP-A attack, where \texttt{WDPO} has the lowest final $\mathrm{candKL}$ in this run.
Their no-attack deviations also show that the modified training rules do not recover the full-data clean-DPO reference in these runs. 
Their design addresses distribution shift rather than identifying and removing a label-flipped subset.

These comparisons support the effectiveness of the proposed defenses in the tested settings without establishing universal empirical superiority.
The empirical ordering depends on the attack setting and may also depend on the feature map, attack budget, and regularization strength.
The feature map determines impact magnitudes and gradient directions. 
Changing it may change the attacked/trimmed-set overlap and the relative performance of the defenses.
The differences across attacks, together with the clean-data and random-flip controls, illustrate why our analysis separates clean defense bias from surviving attack effects.
Building on the constant flip-gradient shift in Lemma~\ref{lem:constant-shift}, the resulting gradient-deviation bounds yield time-uniform parameter stability and then policy-level KL guarantees under the stated log-linear, bounded-region, and curvature assumptions.
These guarantees hold uniformly over all fixed attack sets containing at most $K$ flips, including top-impact attack, random attack, targeted attack, or even attacks with knowledge of the defenses.

\section{Experiments: Matched Regularization Controls}
\label{app:matched-regularization}

We compare \texttt{TRIM\_REG} and \texttt{SOFT\_TRIM\_REG} with full-data \texttt{DPO\_REG} on the same attacked data. 
All three methods use the quadratic penalty $\lambda\|\theta-\theta_{\mathrm{ref}}\|_2^2/2$ with $\lambda=10^{-3}$. 
\texttt{DPO\_REG} does not trim or downweight pairs, allowing us to separate the effect of trimming from the shared quadratic regularizer and measure the additional effect of the impact-score rule.

\begin{table}[htbp]
\centering
\caption{Final SGD candidate-set KL from full-data clean DPO in the $160{,}000$-pair SHP setting. Percentages are computed from unrounded values relative to matched full-data \texttt{DPO\_REG}.}
\label{tab:matched-dpo-reg}
\begin{tabular}{lrrrrr}
\toprule
Attack & \texttt{DPO\_REG} & \texttt{TRIM\_REG} & Reduction & \texttt{SOFT\_TRIM\_REG} & Reduction \\
\midrule
Top-impact & 0.1089 & 0.0803 & 26.3\% & 0.0934 & 14.2\% \\
BMP-A      & 0.1296 & 0.1176 &  9.3\% & 0.1232 &  4.9\% \\
\bottomrule
\end{tabular}
\end{table}

Table~\ref{tab:matched-dpo-reg} shows lower candKL for both proposed REG variants than for \texttt{DPO\_REG} under both attacks. The reductions are larger for top-impact than for BMP-A. Figure~\ref{fig:allrdpo-targeted} shows the corresponding BMP-A trajectories. Together, these comparisons show that impact-based trimming and downweighting provide an additional benefit beyond the shared quadratic regularizer in these two attack settings.

\section{Experiments: Score Assignment and BMP-A Magnitude Geometry}
\label{app:score-ablation}

To test the effect of score assignment, Random \texttt{TRIM\_REG} removes the same number of pairs as \texttt{TRIM\_REG}, while Permuted \texttt{SOFT\_TRIM\_REG} randomly reassigns the \texttt{SOFT\_TRIM\_REG} weights, preserving their multiset and normalization. These are separate randomized controls, not aliases for the impact-based methods. 
Both controls randomize only among pairs with nonzero flip impact. 
The $35{,}256$ zero-impact pairs remain untrimmed and have weight one. 
Within each paired comparison, the attacked set, candidate set, clean-DPO reference, $K$, and $\lambda$ are fixed.

We use the SGD setup of Appendix~\ref{app:shp_simulation_details}: $100{,}000$ updates, learning rate $0.25$, batch size $25{,}600$, and candKL evaluation every $100$ updates. 
We reuse the original SGD trajectories for \texttt{DPO\_REG}, \texttt{TRIM\_REG}, and \texttt{SOFT\_TRIM\_REG}. 
Hard-trim methods sample uniformly with replacement from their retained subsets. 
Soft-trim methods sample all pairs and normalize by the global mean weight. 
Each proposed method and its randomized control use the same initialization and minibatch seed.

\begin{table}[htbp]
\centering
\caption{Nonzero-conditioned control candKL minus impact-score candKL after $100{,}000$ SGD updates. Positive values favor the impact score. Mean $\pm$ sample standard deviation is over ten paired trials.}
\label{tab:score-assignment-ablation}
\resizebox{\textwidth}{!}{%
\begin{tabular}{lrr}
\toprule
Scenario & \texttt{TRIM\_REG} vs. Random \texttt{TRIM\_REG} & \texttt{SOFT\_TRIM\_REG} vs. Permuted \texttt{SOFT\_TRIM\_REG} \\
\midrule
Clean, wrong $K$ & $-0.0086\pm0.0002$ & $-0.0042\pm0.0001$ \\
Random flips     & $-0.0038\pm0.0005$ & $-0.0019\pm0.0004$ \\
Top-impact       & $+0.0278\pm0.0003$ & $+0.0150\pm0.0001$ \\
BMP-A            & $+0.0090\pm0.0004$ & $+0.0048\pm0.0002$ \\
\bottomrule
\end{tabular}}
\end{table}

We repeat each comparison for ten randomized assignments, holding the minibatch seed fixed. 
Random-flip trials also use different attacks. 
Under top-impact and BMP-A, both impact-based methods have lower candKL than their randomized controls in all ten trials. 
The randomized controls have lower candKL in the clean-data and random-flip settings.

For BMP-A, the top-$30\%$ impact region contains $20{,}232$ of the $48{,}000$ flipped pairs ($42.15\%$). 
These pairs account for $58.88\%$ of the summed impact and $72.76\%$ of the summed squared impact over the attacked set.

Figure~\ref{fig:score-ablation} shows the candKL for the top-impact and BMP-A attacks.
The gains over matched randomized controls show that the impact-score assignment matters for these attacks, beyond the number of retained pairs or the distribution of weights.
In contrast, random flips do not preferentially affect high-impact pairs, and the clean-data control contains no attacked pairs to remove.
The mixed results in Table~\ref{tab:score-assignment-ablation} actually reflect the distinction between our uniform perturbation bounds and a claim that impact-based selection is best for every corruption pattern.

\begin{figure}[htbp]
\centering
\begin{subfigure}[t]{0.49\textwidth}
\centering
\includegraphics[width=\linewidth]{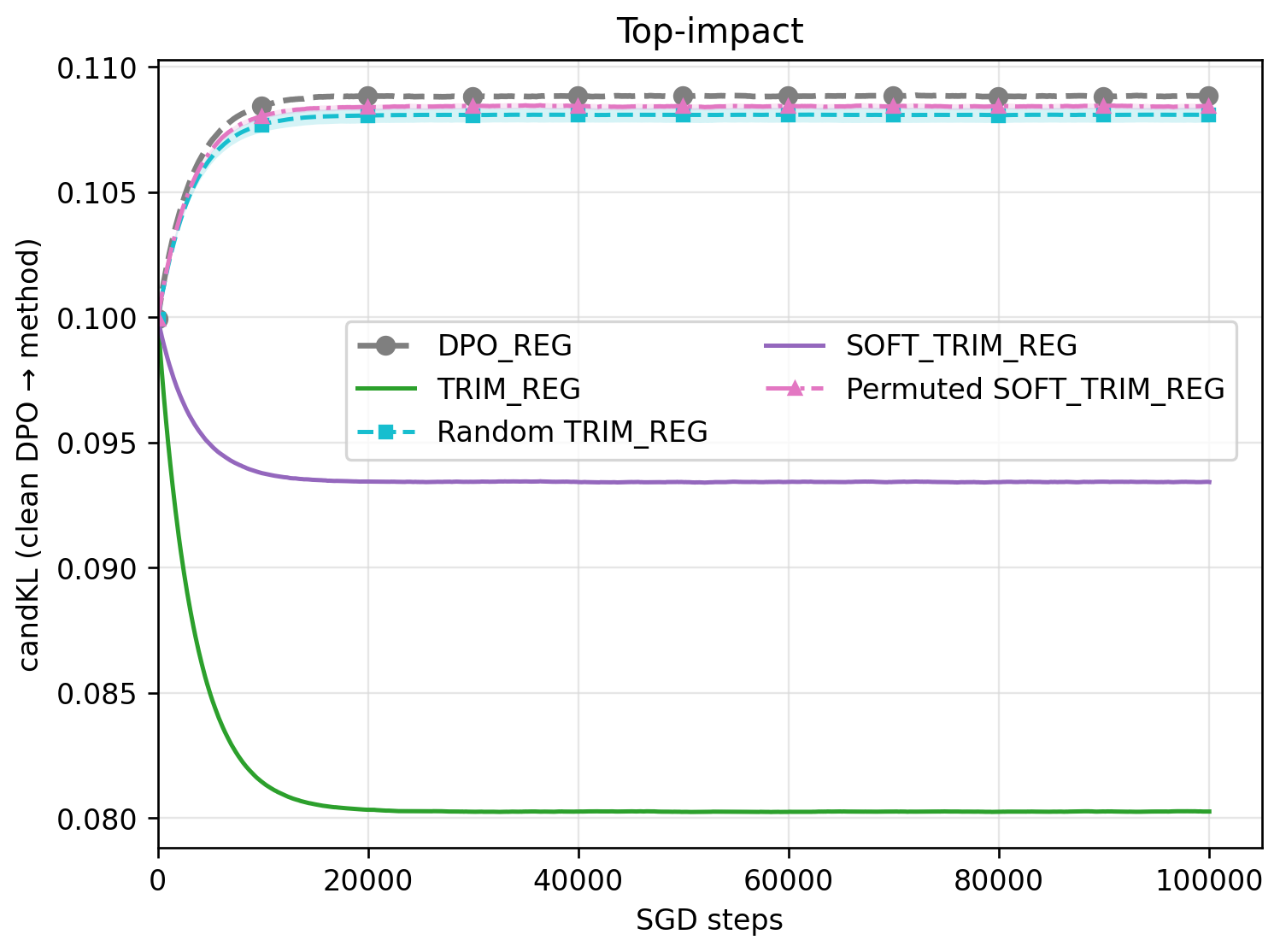}
\caption{Top-impact attack.}
\end{subfigure}\hfill
\begin{subfigure}[t]{0.49\textwidth}
\centering
\includegraphics[width=\linewidth]{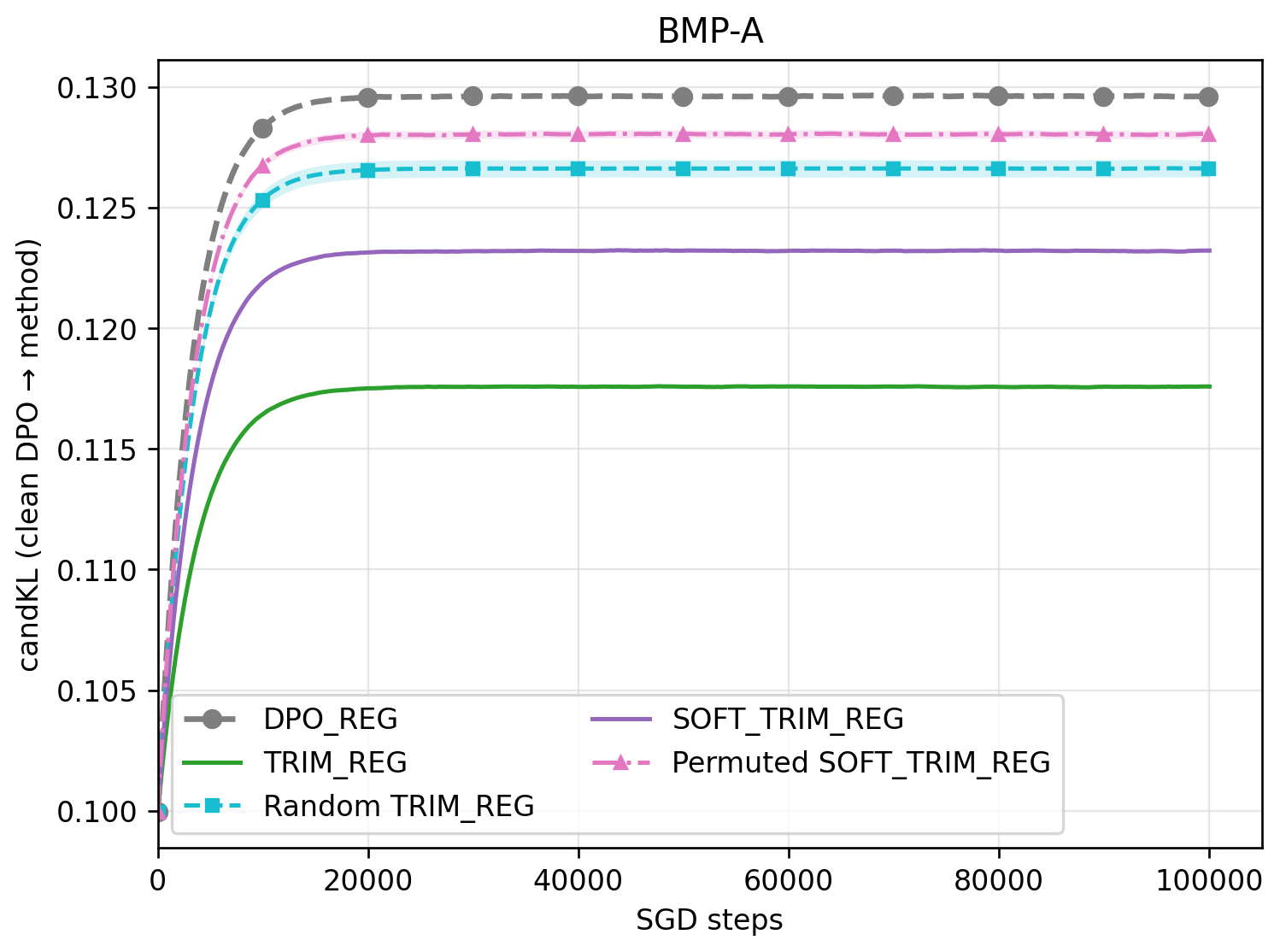}
\caption{BMP-A attack.}
\end{subfigure}
\caption{Score-assignment ablation: candKL versus SGD steps, with the same quadratic regularizer for every method. 
Random \texttt{TRIM\_REG} and Permuted \texttt{SOFT\_TRIM\_REG} are the nonzero-conditioned controls. 
Lines show means and shaded bands show $\pm$ one sample standard deviation over ten random assignments with fixed minibatch seeds. 
The impact-based methods and \texttt{DPO\_REG} use one fixed SGD run for each attack.}
\label{fig:score-ablation}
\end{figure}

\section{Experiments: Three-Dataset BMP-A Budget Sweep}
\label{app:multidataset-budget}

We evaluate SHP~\citep{ethayarajh2022understanding}, HH-RLHF~\citep{bai2022training,ganguli2022red}, and UltraFeedback-Binarized~\citep{cui2023ultrafeedback}. Each dataset uses $6{,}000$ training and $1{,}500$ held-out pairs for seeds $0,1,2$.

\paragraph{Data and features.}
SHP reuses the cached features and seed-$0$ $80/20$ split from Appendix~\ref{app:shp_simulation_details}. HH-RLHF uses \path{Anthropic/hh-rlhf}, with its \texttt{train} and \texttt{test} splits. UltraFeedback-Binarized uses \path{HuggingFaceH4/ultrafeedback_binarized}, with its \texttt{train\_prefs} and \texttt{test\_prefs} splits. We retain the last human prompt and final assistant response for HH-RLHF, and the prompt and final assistant response for UltraFeedback-Binarized. Empty or identical response pairs are excluded, as are UltraFeedback pairs with tied preference scores. We construct pools of $10{,}000$ training and $1{,}600$ held-out pairs per external dataset, excluding held-out prompts that occur in the sampled training pool, and subsample the experimental pairs independently for each seed.

External-dataset features are the $\ell_2$-normalized $384$-dimensional \texttt{[CLS]} outputs of frozen \texttt{all-MiniLM-L6-v2}. The input concatenates the last $800$ prompt characters with the first $5{,}000$ response characters, separated by \texttt{Assistant:}, and is truncated to $256$ tokens. Feature differences are not normalized, preserving the variation in flip-impact scores.

\paragraph{Optimization and attack.}
All methods start at $\theta_0=\theta_{\mathrm{ref}}=0$ and use $\beta=1$, $100{,}000$ SGD steps, batch size $1{,}024$, learning rate $0.25$, $\lambda=10^{-3}$ for REG, and $\tau=5K/12$. Seeds vary data subsampling and SGD. Within each seed, methods use the same data, attack, and minibatch seed. Following Appendix~\ref{app:bmp}, the target is the reflection $\theta^\dagger=2\theta_{\mathrm{ref}}-\theta^\star=-\theta^\star$ of the final clean-DPO parameter. We run BMP-A once to the largest budget and use nested prefixes for the $5\%,10\%,20\%,30\%$ attacks. Each attacked set remains fixed during training.

\paragraph{Metric.}
Table~\ref{tab:multidataset-budget} and Figure~\ref{fig:multidataset-budget} report held-out two-candidate $D_{\mathrm{KL}}(\pi_{\mathrm{clean}}\|\pi_{\mathrm{method}})$, with each candidate set consisting of the held-out preferred and rejected responses. This metric is not directly comparable to the main paper's $100$-candidate candKL.

\begin{table}[htbp]
\centering
\caption{Three-seed mean two-candidate policy KL under nested BMP-A budgets. Residual is BMP-A's relative target-gradient residual.}
\label{tab:multidataset-budget}
\resizebox{\textwidth}{!}{%
\begin{tabular}{llrrrrrrr}
\toprule
Dataset & Budget & \texttt{DPO} & \texttt{DPO\_REG} & \texttt{TRIM} & \texttt{TRIM\_REG} & \texttt{SOFT\_TRIM} & \texttt{SOFT\_TRIM\_REG} & BMP-A Residual \\
\midrule
SHP & 5\%  & .015 & .030 & .018 & .030 & .017 & .030 & .674 \\
    & 10\% & .035 & .034 & .042 & .034 & .039 & .034 & .493 \\
    & 20\% & .076 & .040 & .087 & .038 & .081 & .039 & .244 \\
    & 30\% & .110 & .043 & .126 & .039 & .117 & .041 & .102 \\
\midrule
HH-RLHF & 5\%  & .015 & .039 & .015 & .039 & .014 & .039 & .698 \\
        & 10\% & .034 & .048 & .038 & .048 & .035 & .048 & .520 \\
        & 20\% & .083 & .065 & .104 & .066 & .091 & .065 & .271 \\
        & 30\% & .120 & .072 & .150 & .071 & .134 & .072 & .174 \\
\midrule
UltraFeedback-Binarized & 5\%  & .010 & .023 & .013 & .024 & .011 & .024 & .675 \\
                        & 10\% & .024 & .027 & .030 & .027 & .027 & .027 & .478 \\
                        & 20\% & .049 & .030 & .061 & .029 & .053 & .030 & .273 \\
                        & 30\% & .071 & .032 & .077 & .030 & .073 & .031 & .151 \\
\bottomrule
\end{tabular}}
\end{table}

At $20\%$ and $30\%$, both proposed REG variants have lower KL than DPO on all three datasets. 
Their differences from \texttt{DPO\_REG} are smaller, and both have slightly higher KL than \texttt{DPO\_REG} on HH-RLHF at $20\%$ before rounding. 
At $5\%$, all REG methods have higher KL than DPO. 
The smaller-budget attacks also have larger relative target-gradient residuals, ranging from $0.674$ to $0.698$ at $5\%$ and from $0.478$ to $0.520$ at $10\%$.
These residuals indicate that BMP-A less closely matches the target gradient at smaller budgets, and attacked DPO remains closer to the clean policy than at the larger budgets.

Overall, the regularized defenses are effective against the larger-budget BMP-A attacks across all three datasets. 
At smaller budgets, defense-induced bias can outweigh the reduction in the surviving-attack effect. 
This is the tradeoff captured by our bounds, which do not assert that a defense improves on DPO for every attacked set.

\begin{figure}[htbp]
\centering
\begin{subfigure}[t]{0.32\textwidth}
\centering
\includegraphics[width=\linewidth]{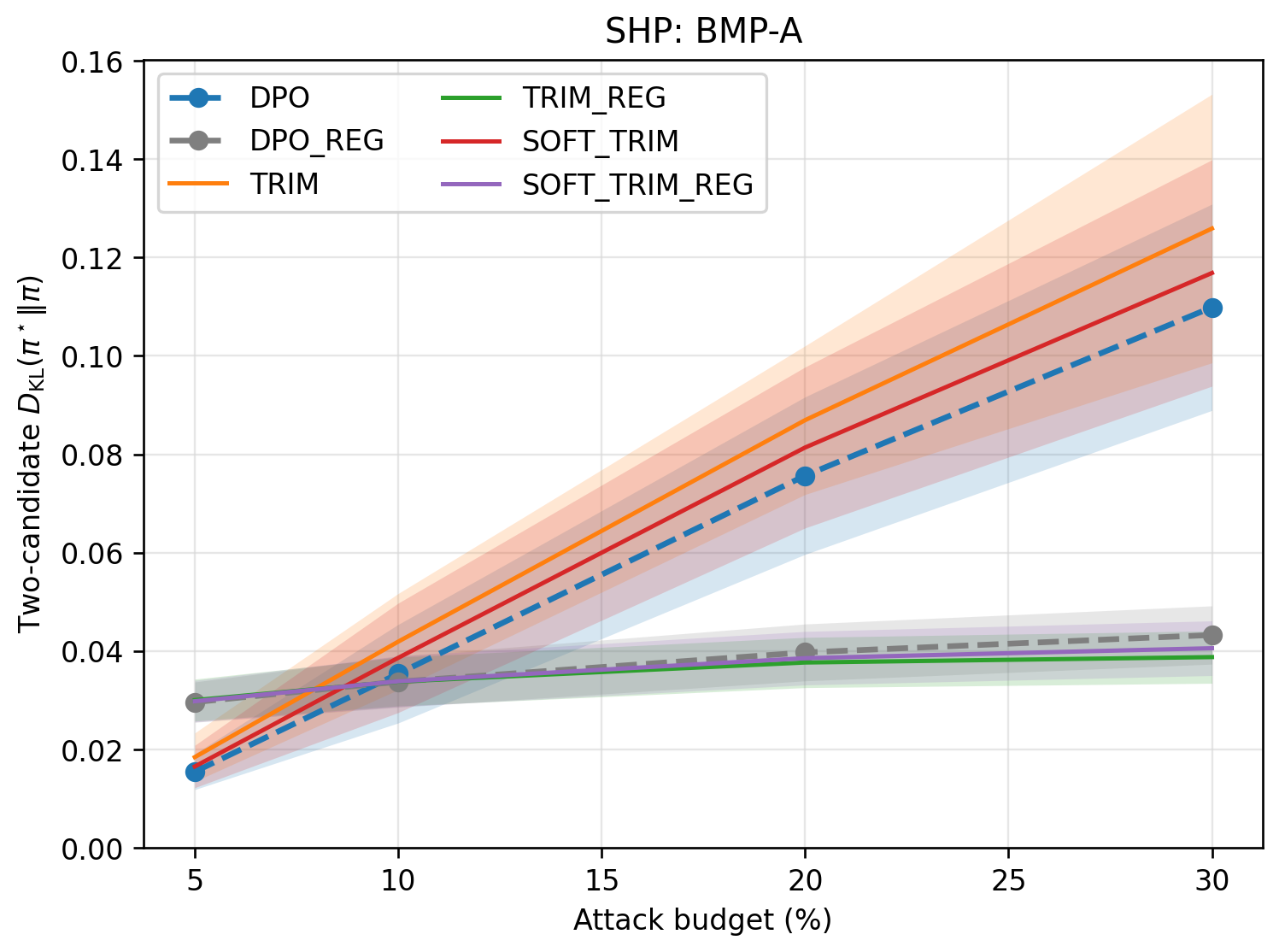}
\caption{SHP.}
\end{subfigure}
\hfill
\begin{subfigure}[t]{0.32\textwidth}
\centering
\includegraphics[width=\linewidth]{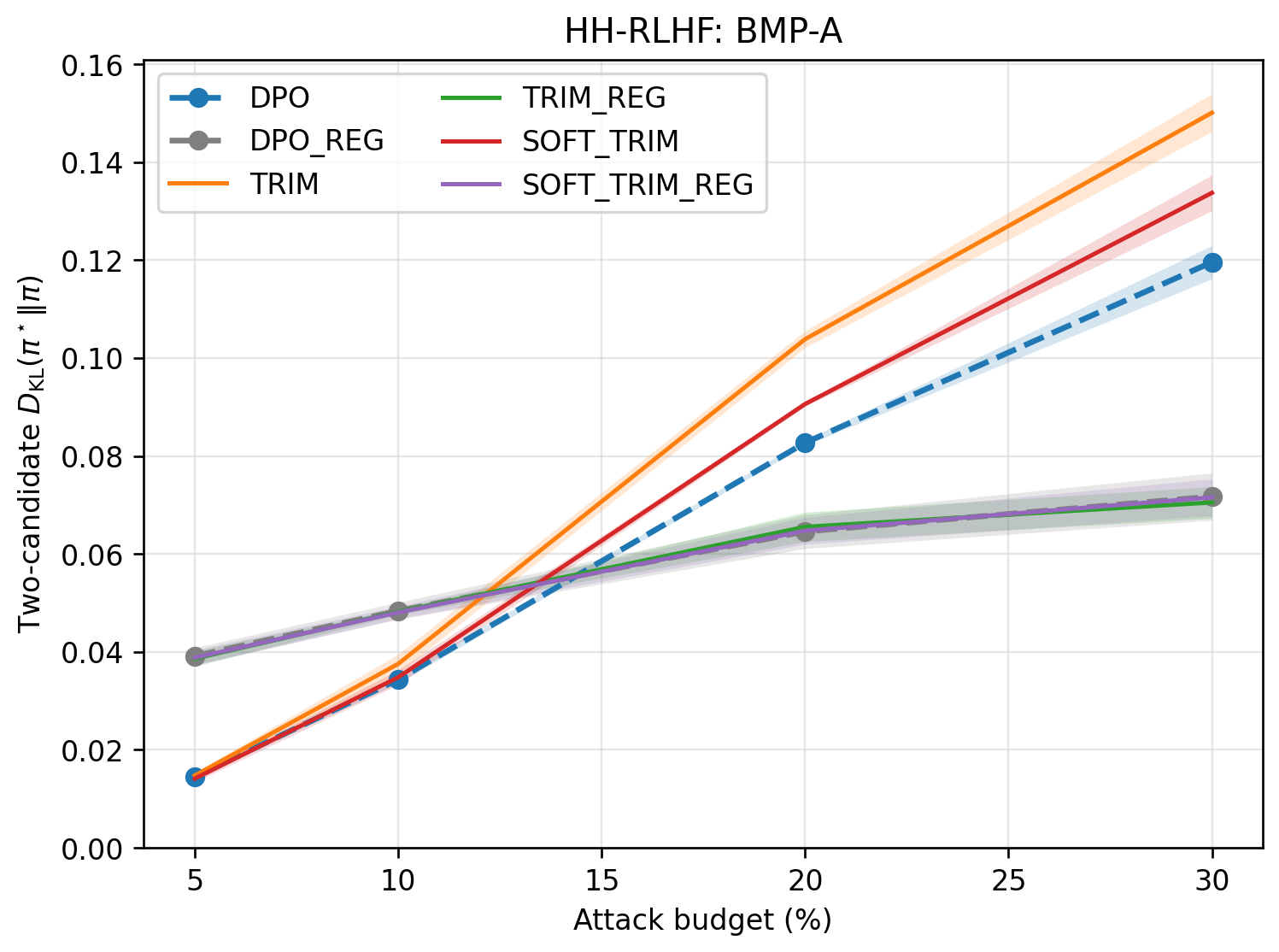}
\caption{HH-RLHF.}
\end{subfigure}
\hfill
\begin{subfigure}[t]{0.32\textwidth}
\centering
\includegraphics[width=\linewidth]{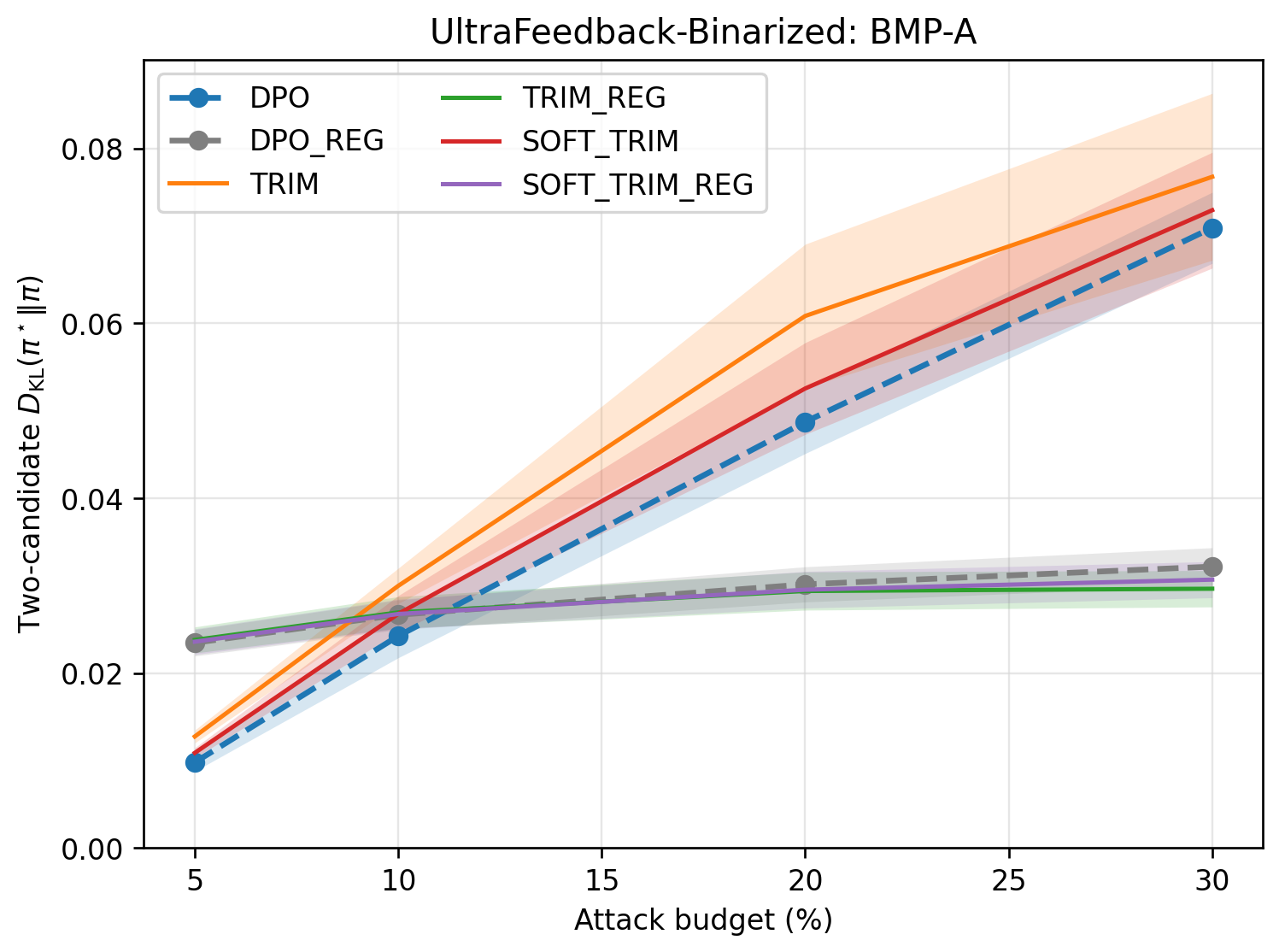}
\caption{UltraFeedback-Binarized.}
\end{subfigure}
\caption{Three-dataset log-linear BMP-A budget sweep. Lines show means and shaded bands show $\pm$ one sample standard deviation over three seeds.}
\label{fig:multidataset-budget}
\end{figure}

\section{Experiments: Budget Selection and Misspecified $K$}
\label{app:k-selection}

Our bounds assume that $K$ is an upper bound on the number of flips. Here we examine sensitivity to the assumed budget and select $K$ using held-out policy KL to the clean-DPO reference available in the experiment.

We use the datasets, preprocessing, initialization, optimizer, and reflected-target construction of Appendix~\ref{app:multidataset-budget}. 
For each dataset and seed, one fixed $20\%$ true BMP-A attacked set on $6{,}000$ training pairs is reused while the defender's assumed budget varies over
$K/N\in\{0\%,1\%,2\%,4\%,8\%,16\%,32\%\}$.
For each method and seed, we select the budget with the lowest two-candidate KL on the first $500$ held-out pairs, breaking ties toward smaller $K$. We evaluate the selected method on the remaining $1{,}000$ pairs. There are three seeds, $100{,}000$ SGD steps, $\lambda=10^{-3}$, and $\tau=5K/12$ for $K>0$. When $K=0$, all soft weights are set to one, so the unregularized and REG methods reduce to DPO and \texttt{DPO\_REG}, respectively.

\begin{table}[htbp]
\centering
\caption{Mean test two-candidate KL for a fixed true $20\%$ attack while assumed $K$ changes. Each entry is unregularized / regularized methods.}
\label{tab:k-misspecification}
\resizebox{\textwidth}{!}{%
\begin{tabular}{llrr}
\toprule
Dataset & Assumed $K/N$ & \texttt{TRIM} / \texttt{TRIM\_REG} & \texttt{SOFT\_TRIM} / \texttt{SOFT\_TRIM\_REG} \\
\midrule
SHP & 0\%  & .0743 / .0390 & .0743 / .0390 \\
    & 1\%  & .0767 / .0390 & .0758 / .0390 \\
    & 2\%  & .0770 / .0388 & .0760 / .0389 \\
    & 4\%  & .0788 / .0386 & .0765 / .0388 \\
    & 8\%  & .0806 / .0381 & .0780 / .0385 \\
    & 16\% & .0875 / .0375 & .0801 / .0381 \\
    & 32\% & .0886 / .0357 & .0820 / .0372 \\
\midrule
HH-RLHF & 0\%  & .0818 / .0636 & .0818 / .0636 \\
        & 1\%  & .0824 / .0640 & .0824 / .0640 \\
        & 2\%  & .0820 / .0636 & .0821 / .0639 \\
        & 4\%  & .0832 / .0632 & .0819 / .0635 \\
        & 8\%  & .0872 / .0634 & .0837 / .0634 \\
        & 16\% & .0983 / .0643 & .0877 / .0637 \\
        & 32\% & .1099 / .0628 & .0941 / .0639 \\
\midrule
UltraFeedback-Binarized & 0\%  & .0484 / .0303 & .0484 / .0303 \\
                        & 1\%  & .0500 / .0303 & .0493 / .0302 \\
                        & 2\%  & .0498 / .0301 & .0494 / .0302 \\
                        & 4\%  & .0508 / .0299 & .0497 / .0301 \\
                        & 8\%  & .0521 / .0297 & .0502 / .0299 \\
                        & 16\% & .0557 / .0294 & .0512 / .0297 \\
                        & 32\% & .0595 / .0288 & .0531 / .0294 \\
\bottomrule
\end{tabular}}
\end{table}

\begin{table}[htbp]
\centering
\caption{Validation-selected assumed $K/N$ for seeds $0/1/2$. The true attack fraction is $20\%$ in every run.}
\label{tab:k-selected}
\begin{tabular}{lrrrr}
\toprule
Dataset & \texttt{TRIM} & \texttt{TRIM\_REG} & \texttt{SOFT\_TRIM} & \texttt{SOFT\_TRIM\_REG} \\
\midrule
SHP & 0/0/0 & 32/32/32 & 0/0/0 & 32/32/32 \\
HH-RLHF & 0/0/0 & 32/32/4 & 0/0/0 & 32/16/4 \\
UltraFeedback-Binarized & 2/0/0 & 32/32/32 & 8/0/0 & 32/32/32 \\
\bottomrule
\end{tabular}
\end{table}

Figure~\ref{fig:k-misspecification} shows the held-out test curves as the assumed budget varies. Both proposed REG variants have lower mean test KL than DPO across the entire grid on all three datasets, while unregularized trimming generally increases the deviation. Validation selects the test-grid minimizer in $31$ of $36$ cases. Only $15$ of the $36$ selected budgets satisfy $\widehat K/N\ge20\%$.

Thus, the REG variants retain their advantage over attacked DPO in this experiment even when the assumed budget differs from the true budget. A clean-policy reference also helps select $K$ for policy proximity, but the selected budget is not an estimate of the number of attacked pairs. The uniform guarantees apply when the chosen $K$ is an upper bound on that number.

\begin{figure}[htbp]
\centering
\begin{subfigure}[t]{0.32\textwidth}
\centering
\includegraphics[width=\linewidth]{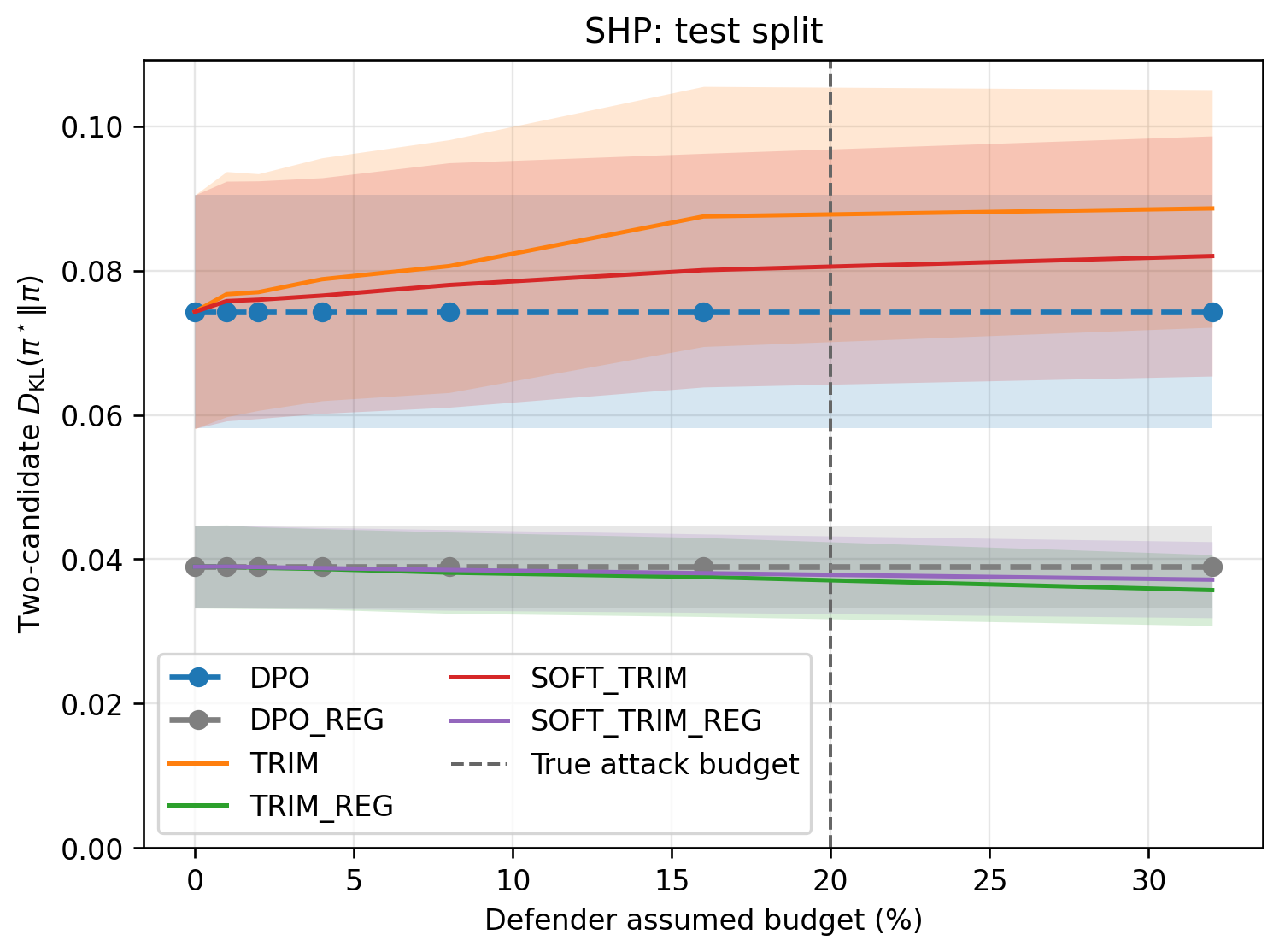}
\caption{SHP.}
\end{subfigure}
\hfill
\begin{subfigure}[t]{0.32\textwidth}
\centering
\includegraphics[width=\linewidth]{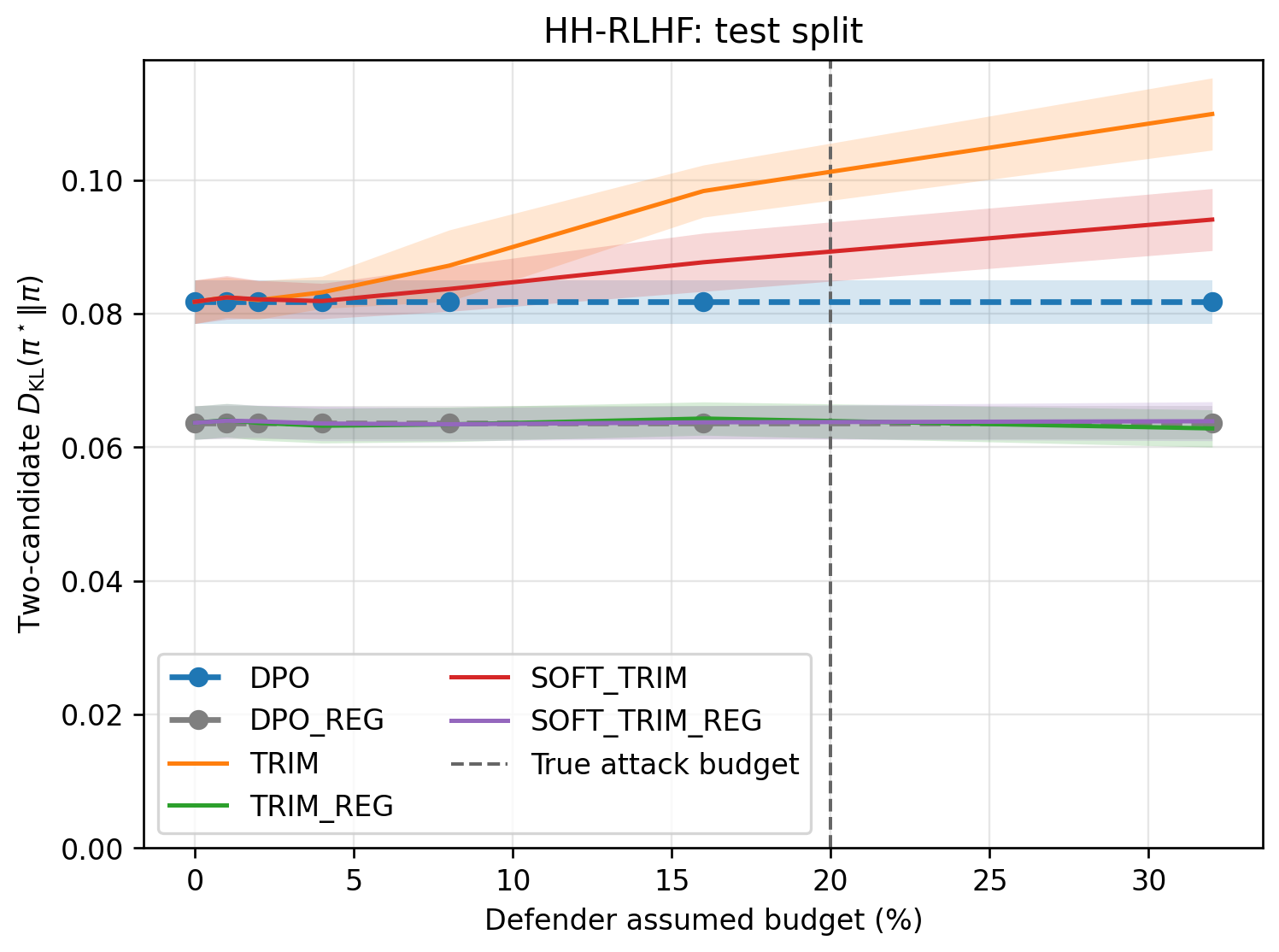}
\caption{HH-RLHF.}
\end{subfigure}
\hfill
\begin{subfigure}[t]{0.32\textwidth}
\centering
\includegraphics[width=\linewidth]{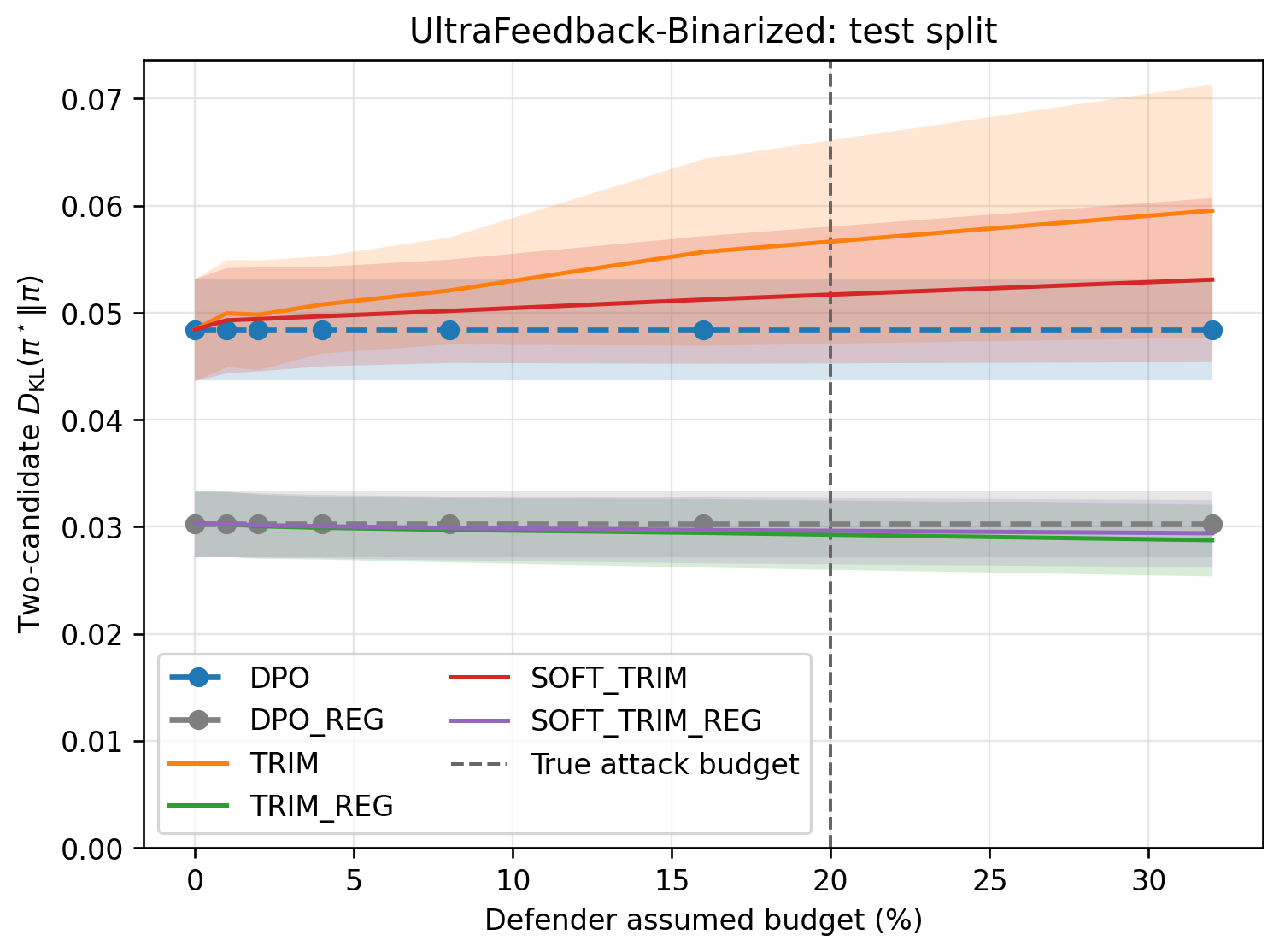}
\caption{UltraFeedback-Binarized.}
\end{subfigure}
\caption{Held-out test KL for a fixed-$20\%$ BMP-A support with a varied defender budget. Lines show means and shaded bands show $\pm$ one sample standard deviation over three seeds. The vertical dashed line marks the true attack fraction.}
\label{fig:k-misspecification}
\end{figure}

\section{Extension to Nonlinear DPO Policies}
\label{app:nonlinear-extension}

We discuss the difficulties in extending the analysis beyond log-linear policies and report preliminary experiments with a small nonlinear policy over frozen representations.

\subsection{Flip-gradient identity for differentiable DPO policies}
\label{app:nonlinear-flip-identity}

We restate the flip-gradient identity of~\citet{yang2026efficient} in our notation and include its proof for completeness.

\begin{lemma}[Flip-gradient identity for differentiable DPO policies]\label{lem:general-flip-identity}
For any differentiable policy $\pi_\theta$ and a flipped pair $i\in A$, let $\ell_i$ be the DPO loss on the clean pair $(x^i,y_w^i,y_l^i)$ and let $\tilde\ell_i$ be the observed loss after swapping the two responses. Wherever the log-policy gradients exist,
\begin{equation}\label{eq:general-flip-identity}
\nabla\tilde\ell_i(\theta)-\nabla\ell_i(\theta)
=\beta\!\left[
\nabla\log\pi_\theta(y_w^i\mid x^i)
-\nabla\log\pi_\theta(y_l^i\mid x^i)
\right].
\end{equation}
At a fixed $\theta$, reversing the pair orientation negates the vector on the right-hand side and therefore preserves its norm. Unlike Lemma~\ref{lem:constant-shift}, however, this vector generally depends on $\theta$.
\end{lemma}

\begin{proof}
Let $z_i(\theta)$ be the DPO logit of the clean pair, so that
$\ell_i(\theta)=-\log\sigma(z_i(\theta))$ and
$\tilde\ell_i(\theta)=-\log\sigma(-z_i(\theta))$.
Differentiation gives
$\nabla\tilde\ell_i(\theta)-\nabla\ell_i(\theta)=\nabla z_i(\theta)$.
The reference-policy terms in $z_i(\theta)$ are constant in $\theta$, leaving exactly~\eqref{eq:general-flip-identity}.
\end{proof}

\subsection{Why the log-linear robustness analysis does not directly extend}
\label{app:nonlinear-difficulty}

In the log-linear class, each flip contributes the fixed gradient shift $\beta v_i$, so the impact ranking and attack-induced gradient shift do not change during training. For a nonlinear policy, the flip-gradient vector in Lemma~\ref{lem:general-flip-identity} depends on $\theta$. Extending our analysis would require control of these vectors and their variation over the training region, rather than only at initialization. Recomputing scores can also change the retained set or weights during training.

The nonlinear objective is generally nonconvex, so the covariance argument used in Theorems~\ref{thm:trim-K-iter-stability} and~\ref{thm:softtrim-iter-stability} does not directly provide the required curvature and contraction. A nonlinear extension would also need a suitable parameter-to-policy bound and, for refreshed scores, an analysis of the changing objective.

\subsection{Nonlinear preliminary experiments setup}
\label{app:nonlinear-experiments}

We use SHP with seeds $0,1,2$, subsampling $6{,}000$ training pairs and $2{,}000$ held-out pairs from the split in Appendix~\ref{app:shp_simulation_details}. A seed-specific Gaussian projection followed by $\ell_2$ normalization maps frozen $384$-dimensional MiniLM representations to $\bar\psi\in\mathbb R^{32}$. The policy uses the nonlinear energy head
\begin{equation}\label{eq:nonlinear-energy-head}
f_\theta(\bar\psi)=a^\top\tanh(W\bar\psi+b),
\qquad
\pi_\theta(j\mid x;\mathcal Y)\propto\exp\!\big(f_\theta(\bar\psi(x,\mathcal Y_j(x)))\big),
\end{equation}
with hidden width $8$. Seeds vary data subsampling, the projection, initialization, and SGD. Within each seed, all methods share the same randomly initialized frozen reference, starting parameters, and minibatch seed. We run $10{,}000$ SGD steps with batch size $512$, learning rate $1$, $\beta=1$, $\lambda=10^{-3}$ for REG, and $\tau=5K/12$.

The attack target is trained from the reference for $10{,}000$ steps with all preference labels reversed. Neural BMP-A uses the exact flip-gradient vectors at that target and constructs nested attack sets for $5\%,10\%,20\%,30\%$ budgets. Each set is fixed during the corresponding training run. The defender does not observe the target or attacked set and computes the static impact scores once at the common reference initialization. At $30\%$ ($K=1{,}800$), the mean BMP-A relative target residual is $0.097$, and the static top-$K$ set contains $40.9\%$ of the attacked pairs. We also train on clean data using the same assumed budgets to measure the clean defense bias.

The $2{,}000$ held-out pairs are used for pairwise diagnostics. The reported candKL instead uses $500$ fixed held-out SHP prompts from the candidate cache of Appendix~\ref{app:shp_simulation_details}, with $100$ candidates per prompt, shared across methods. The clean-DPO nonlinear policy trained for the same number of steps is the comparator.

\subsection{Static impact-score results}
\label{app:nonlinear-static}

\begin{table}[htbp]
\centering
\caption{Mean $\pm$ sample standard deviation candidate-set
$D_{\mathrm{KL}}(\pi_{\mathrm{clean}}\|\pi_{\mathrm{method}})$ over three nonlinear-head seeds under the fixed $30\%$ neural BMP-A setting.}
\label{tab:nonlinear-policy}
\resizebox{\textwidth}{!}{%
\begin{tabular}{rrrrrr}
\toprule
\texttt{DPO} & \texttt{DPO\_REG} & \texttt{TRIM} & \texttt{TRIM\_REG} & \texttt{SOFT\_TRIM} & \texttt{SOFT\_TRIM\_REG} \\
\midrule
$.213\pm.025$ & $.141\pm.019$ & $.297\pm.029$ & $.115\pm.013$ & $.241\pm.009$ & $.128\pm.015$ \\
\bottomrule
\end{tabular}}
\end{table}

At the $30\%$ attack budget, \texttt{TRIM\_REG} and \texttt{SOFT\_TRIM\_REG} have $18.7\%$ and $9.6\%$ lower mean candKL than \texttt{DPO\_REG}, respectively (Table~\ref{tab:nonlinear-policy}). Without regularization, \texttt{TRIM} and \texttt{SOFT\_TRIM} have higher candKL than DPO.

Figure~\ref{fig:nonlinear-policy} shows the budget sweep and clean-data controls. At $5\%$, all REG methods, including full-data \texttt{DPO\_REG}, have higher candKL than DPO, while unregularized \texttt{TRIM} and \texttt{SOFT\_TRIM} slightly improve on DPO. On clean data, candKL increases with the assumed budget for unregularized \texttt{TRIM} and \texttt{SOFT\_TRIM}. At $30\%$, their candKL values are $0.0494$ and $0.0132$, respectively.

These preliminary results show a benefit from combining impact-based trimming with regularization under the larger nonlinear BMP-A attack. At smaller budgets, the regularization bias can outweigh this benefit, as also seen for full-data \texttt{DPO\_REG}. Extending the robustness guarantees to this policy class requires the additional analysis discussed in Appendix~\ref{app:nonlinear-difficulty}.

\begin{figure}[htbp]
\centering
\begin{subfigure}[t]{0.49\textwidth}
\centering
\includegraphics[width=\linewidth]{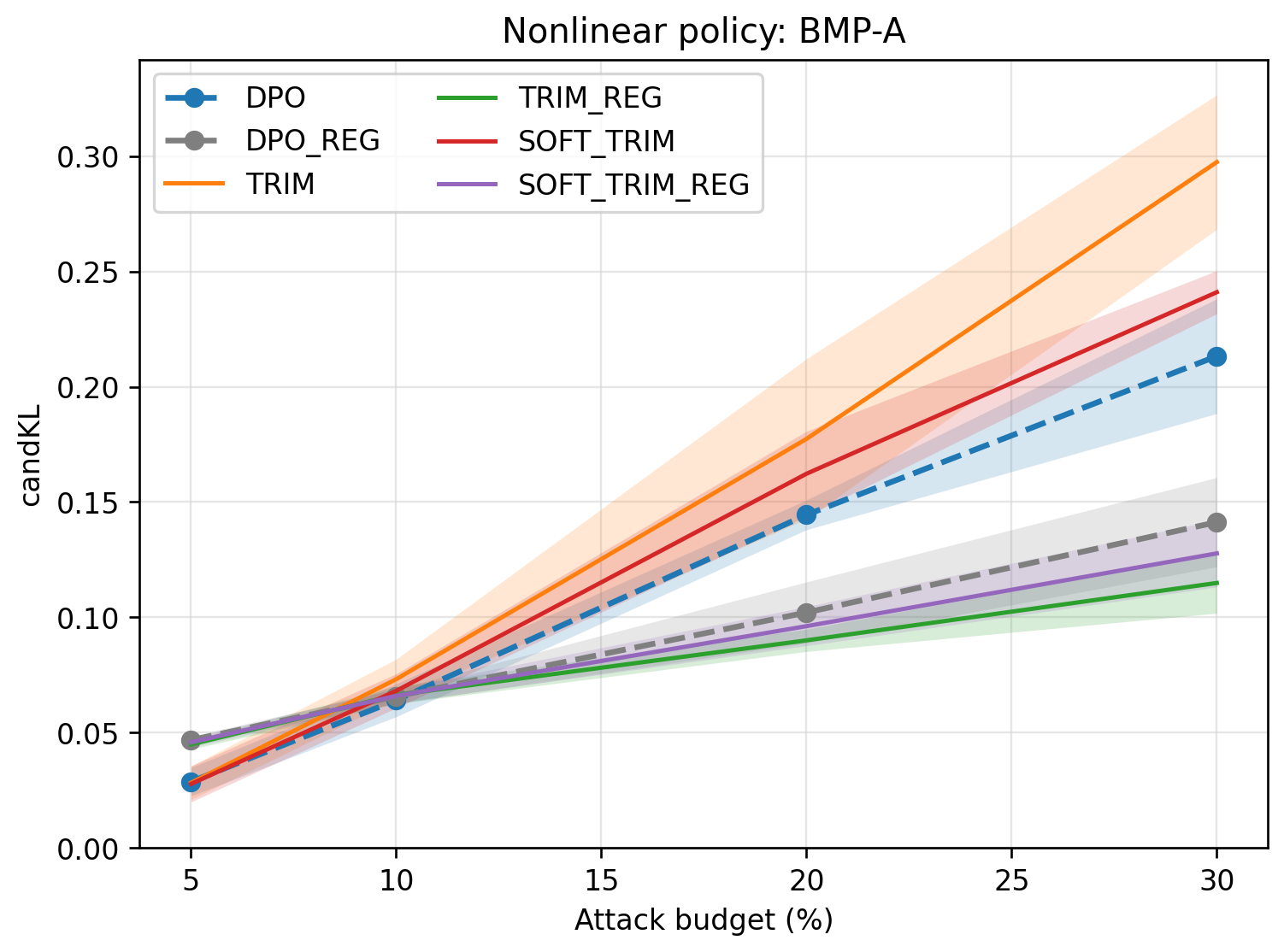}
\caption{Neural BMP-A attacks.}
\end{subfigure}\hfill
\begin{subfigure}[t]{0.49\textwidth}
\centering
\includegraphics[width=\linewidth]{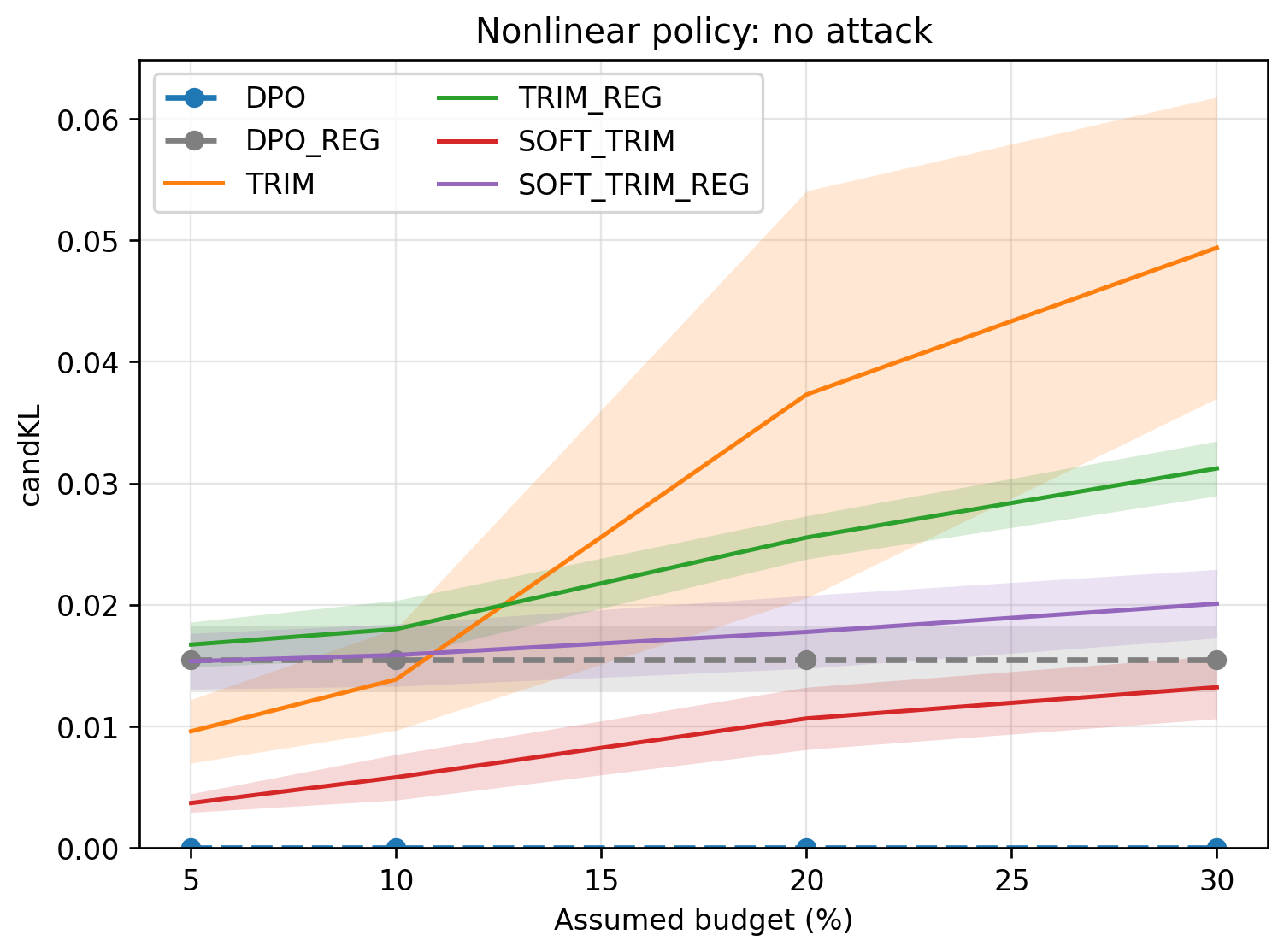}
\caption{Clean data, assumed attack budgets.}
\end{subfigure}
\caption{Small nonlinear-head sweep over frozen SHP representations. Left: nested $5\%,10\%,20\%,30\%$ neural BMP-A attacks. Right: clean data with the same assumed budgets. Lines show means and shaded bands show $\pm$ one sample standard deviation over three seeds.}
\label{fig:nonlinear-policy}
\end{figure}

\subsection{Periodic impact-score refresh}
\label{app:nonlinear-refresh}

We keep the $30\%$ attacked set fixed and recompute the parameter-dependent impact scores every $500$ optimization steps. This gives $20$ score evaluations at steps $0,500,\ldots,9{,}500$.

\begin{table}[htbp]
\centering
\caption{Static versus periodically refreshed candidate-set KL, averaged over three nonlinear-head seeds.}
\label{tab:nonlinear-refresh}
\begin{tabular}{lrr}
\toprule
Method & Static & Refreshed \\
\midrule
\texttt{TRIM} & .2974 & .2941 \\
\texttt{TRIM\_REG} & .1149 & .1056 \\
\texttt{SOFT\_TRIM} & .2410 & .2266 \\
\texttt{SOFT\_TRIM\_REG} & .1277 & .1216 \\
\bottomrule
\end{tabular}
\end{table}

Table~\ref{tab:nonlinear-refresh} shows lower mean candKL for all four methods after score refresh. At the last refresh, the top-$K$ set contains $43.6\%$ of the attacked pairs for \texttt{TRIM} and $45.8\%$ for \texttt{SOFT\_TRIM}, compared with $40.9\%$ at initialization. The corresponding fractions for \texttt{TRIM\_REG} and \texttt{SOFT\_TRIM\_REG} are $41.5\%$ and $41.9\%$. \texttt{TRIM} removes these pairs, while \texttt{SOFT\_TRIM} downweights them.

Refreshing the impact scores therefore gives a further improvement in this nonlinear experiment, with \texttt{TRIM\_REG} attaining the lowest mean candKL among the compared methods. The unregularized variants still have higher candKL than DPO, so score refresh alone does not resolve the tradeoff between removing attacked pairs and changing the clean objective.

\end{document}